\pdfoutput=1
\documentclass[mnsc]{informs3}
\pdfoutput=1

\OneAndAHalfSpacedXI 

\usepackage{graphicx}
\usepackage{epstopdf}
\usepackage{subfigure,caption}
\usepackage{microtype}
\usepackage{graphicx}
\usepackage{subfigure}
\usepackage{booktabs} 

\usepackage{amsfonts}
\usepackage{xcolor}
\let\proof\relax
\let\endproof\relax
\usepackage{amsthm}
\usepackage{amsmath}
\usepackage{amssymb}
\usepackage{soul}
\usepackage{comment}
\usepackage{enumitem}
\usepackage{bbm}
\usepackage{graphicx}
\usepackage{float}
\usepackage{subfigure}
\usepackage{mathtools}
\usepackage[ruled,linesnumbered,noend,noline]{algorithm2e}

\newtheorem{definition}{Definition}

\newtheorem{theorem}{Theorem}

\newtheorem{lemma}{Lemma}
\newtheorem{proposition}{Proposition}
\newtheorem{remark}{Remark}

\newenvironment{proofsketch}{\proof}{\endproof}
\theoremstyle{remark}

\newcommand{\ra}{\rightarrow}
\newcommand{\mb}[1]{\mathbb{#1}}
\newcommand{\mbf}[1]{\mathbf{#1}}
\newcommand{\mc}[1]{\mathcal{#1}}

\newcommand{\mbp}{\mathbb{P}}

\newcommand{\rr}{\mathbb{R}}
\newcommand{\ee}{\mathbb{E}}
\newcommand{\pp}{\mathbb{P}}

\newcommand{\nn}{\mathbb{N}}
\renewcommand{\epsilon}{\varepsilon}

\newcommand{\abs}[1]{\left| #1 \right|}
\newcommand{\norm}[1]{\left\| #1 \right\|}

\newcommand{\ceil}[1]{\left\lceil #1 \right\rceil}
\newcommand{\floor}[1]{\left\lfloor #1 \right\rfloor}

\newcommand{\indicator}{\mathbbm{1}}

\newcommand\defeq{\stackrel{\mathclap{\normalfont\tiny\mbox{def}}}{=}}

\newcommand{\wmedit}[1]{{#1}}

\usepackage{hyperref}

\usepackage{multirow}
\usepackage{colortbl}
\usepackage{booktabs,multirow,array,caption}
\usepackage{pbox}
\usepackage{pifont}
\usepackage{tikz}
\usepackage{xcolor}
\usepackage{graphicx}

\newcommand*\circled[1]{\tikz[baseline=(char.base)]{\node[shape=circle,draw,inner sep=1pt] (char) {#1};}}
\newcommand{\cmark}{\ding{51}}%
\newcommand{\xmark}{\ding{55}}%
\DeclareSymbolFont{extraup}{U}{zavm}{m}{n}
\DeclareMathSymbol{\varheartsuit}{\mathalpha}{extraup}{86}
\DeclareMathSymbol{\vardiamondsuit}{\mathalpha}{extraup}{87} 
\newcommand{\vertrule}[1][.7ex]{\rule{.6pt}{#1}}
\newcommand{\si}{s_{\vertrule{}}}
\renewcommand{\so}{s_{\circ}}
\renewcommand{\ni}{N_{\vertrule{}}}
\newcommand{\no}{N_{\circ}}
\newcommand{\nos}{N_{\circ}^{\star}}
\newcommand{\re}{{\text{ref}}}
	
\definecolor{dark-gray}{gray}{0.45}
\newcommand{\algedit}[1]{\textcolor{gray}{#1}}

\usepackage{array}
\newcolumntype{M}[1]{>{\centering\arraybackslash}m{#1}}

\usepackage{natbib}
 \bibpunct[, ]{(}{)}{,}{a}{}{,}%
 \def\bibfont{\small}%
\usepackage{bm}
\usepackage{comment}

\ECRepeatTheorems

\EquationsNumberedThrough    

\allowdisplaybreaks
\begin{document}


\RUNAUTHOR{}

\RUNTITLE{Model-Free Non-Stationary RL: Near-Optimal Regret and Applications}

\TITLE{\Large  Model-Free Non-Stationary RL: Near-Optimal Regret and Applications in Multi-Agent RL and Inventory Control}

\ARTICLEAUTHORS{
\AUTHOR{Weichao Mao}
\AFF{Department of Electrical and Computer Engineering \& Coordinated Science Laboratory, University of Illinois Urbana-Champaign, Urbana, IL 61801, \EMAIL{weichao2@illinois.edu}}

\AUTHOR{Kaiqing Zhang}
\AFF{Laboratory for Information \& Decision Systems, Massachusetts Institute of Technology, Cambridge, MA 02139, \EMAIL{kaiqing@mit.edu}}

\AUTHOR{Ruihao Zhu}
\AFF{Cornell SC Johnson College of Business, Ithaca, NY 14853, \EMAIL{ruihao.zhu@cornell.edu}}

\AUTHOR{David Simchi-Levi}
\AFF{Institute for Data, Systems, and Society, Department of Civil and Environmental Engineering, Operations Research Center, Massachusetts Institute of Technology, Cambridge, MA 02139, \EMAIL{dslevi@mit.edu}}
\AUTHOR{Tamer Ba\c{s}ar}
\AFF{Department of Electrical and Computer Engineering \& Coordinated Science Laboratory, University of Illinois Urbana-Champaign, Urbana, IL 61801, \EMAIL{basar1@illinois.edu}} 

} 

\ABSTRACT{%
We consider model-free reinforcement learning (RL) in non-stationary Markov decision processes. Both the reward functions and the state transition functions are allowed to vary arbitrarily over time as long as their cumulative variations do not exceed certain variation budgets. We propose Restarted Q-Learning with Upper Confidence Bounds  (RestartQ-UCB), the first model-free algorithm for non-stationary RL, and show that it outperforms existing solutions in terms of dynamic regret. Specifically, RestartQ-UCB with Freedman-type bonus terms achieves a dynamic regret bound of $\widetilde{O}(S^{\frac{1}{3}} A^{\frac{1}{3}} \Delta^{\frac{1}{3}} H T^{\frac{2}{3}})$, where $S$ and $A$ are the numbers of states and actions, respectively, $\Delta>0$ is the variation budget, $H$ is the number of time steps per episode, and $T$ is the total number of time steps. We further present a parameter-free algorithm named Double-Restart Q-UCB that does not require prior knowledge of the variation budget. We show that our algorithms are \emph{nearly optimal} by establishing an information-theoretical lower bound of $\Omega(S^{\frac{1}{3}} A^{\frac{1}{3}} \Delta^{\frac{1}{3}} H^{\frac{2}{3}} T^{\frac{2}{3}})$, the first lower bound in non-stationary RL. Numerical experiments validate the advantages of RestartQ-UCB in terms of both cumulative rewards and computational efficiency. We demonstrate the power of our results in examples of multi-agent RL and inventory control across related products. 
}%



\maketitle

%

\newpage
\section{Introduction}\label{sec:intro}
Reinforcement learning (RL) focuses on the class of problems where an agent maximizes its cumulative reward through sequential interactions with an initially unknown but fixed environment, usually modeled by a Markov Decision Process (MDP). 
In classical RL problems, the state transition functions and the reward functions are assumed to be time-invariant, i.e., stationary. However, stationary models cannot capture the time-varying environments in a wide range of sequential decision-making problems, such as online advertisement auctions~\citep{cai2017real,lu2019reinforcement}, dynamic pricing~\citep{chawla2016simple}, traffic management~\citep{chen2020toward}, healthcare operations~\citep{shortreed2011informing}, multi-agent RL~\citep{littman1994markov}, and inventory control~\citep{agrawal2019learning,cheung2020reinforcement}. Among the many intriguing applications, in the following, we specifically elaborate on two research areas, namely multi-agent RL and inventory control, that can significantly benefit from progresses on non-stationary RL. In Appendix~\ref{app:application}, we further discuss the potential application scenarios of non-stationary RL in other important areas, such as sequential transfer RL and multi-task RL. 

\begin{itemize}[leftmargin=*]
	\item \textbf{Multi-agent RL:} In multi-agent RL, a set of agents either collaborate or compete by taking actions in a shared environment. This commonly occurs in many operational scenarios when multiple decision-makers interact with each other, such as ads auctions \citep{BalseiroG19} and dynamic pricing \citep{BirgeCKW21}. In such scenarios, each agent faces a non-stationary environment, especially when the agents learn and update their policies simultaneously, as the actions of the other agents can alter the environment. We discuss this connection with more details in Section~\ref{sec:application} through a concrete example, where we show that our non-stationary RL solution can be readily applied to a multi-agent RL problem against a slowly-changing opponent. 
	\item \textbf{Inventory control across related but different products}: In conventional inventory control \citep{HuhR09,ZhangCS18,agrawal2019learning}, the retailer typically focuses on managing the stock level of a single product. Nevertheless, the sequential launch of new related products (e.g., the line of iPhone) provides us with the opportunity to leverage experience from past products to inform inventory management for future products. In Section~\ref{sec:application_inventory}, we discuss how one can apply our non-stationary RL solutions to guide the inventory management not only for a single product but also across a sequence of related, but \emph{different} products. 
\end{itemize}

RL in a non-stationary MDP is highly non-trivial due to the following challenges. First, similar to stationary RL, the agent faces the \emph{exploration vs. exploitation} dilemma: it needs to explore the uncertain environment efficiently while maximizing its rewards along the way. In~\citet{jaksch2010near}, the authors proposed to leverage the ``optimism in the face of uncertain'' principle to guide exploration.
Another challenge, which is unique to non-stationary RL, is the trade-off between \emph{remembering and forgetting}.  On the one hand, since the underlying MDP varies over time, data samples collected in prior interactions can become obsolete. In fact, it has been shown that a standard stationary RL algorithm might incur a linear regret if the non-stationarity is not handled properly~\citep{ortner2020variational}. On the other hand, the agent needs to extract a sufficient amount of information from historical data to inform future decision-making.

\begin{table*}[t]\centering 
	\def\arraystretch{1.6}
	\begin{tabular}{|c|c|c|c|c|}
		\hline
		Setting           & Algorithm & Regret & \pbox{3cm}{Model-\\Free} & Comment \\ \hline
		\multirow{5}{*}{\pbox{3cm}{Undis-\\counted}}  
		&     \citet{jaksch2010near}      &    $\widetilde{O}(S^{\textcolor{white}{\frac{1}{1}}}A^{\frac{1}{2}}L^{\frac{1}{3}}D^{\textcolor{white}{\frac{1}{1}}}T^{\frac{2}{3}})$  & \xmark &    only abrupt changes     \\
		&     \citet{gajane2018sliding}      &    $\widetilde{O}(S^{\frac{2}{3}}A^{\frac{1}{3}}L^{\frac{1}{3}}D^{\frac{2}{3}}T^{\frac{2}{3}})$    & \xmark &   only abrupt changes    \\
		&    \citet{ortner2020variational}       &   $\widetilde{O}(S^{\textcolor{white}{\frac{2}{3}}}A^{\frac{1}{2}}\Delta^{\frac{1}{3}}D^{\textcolor{white}{\frac{1}{1}}}T^{\frac{2}{3}})$     &  \xmark &   requires local  variations   \\
		&    \citet{cheung2020reinforcement}       &   $\widetilde{O}(S^{\frac{2}{3}}A^{\frac{1}{2}}\Delta^{\frac{1}{4}}D^{\textcolor{white}{\frac{1}{1}}}T^{\frac{3}{4}})$     & \xmark &   does not require $\Delta$    \\ 
		& \cellcolor{gray!40} Lower bound & $\Omega ( S^{\frac{1}{3}}A^{\frac{1}{3}} \Delta^{\frac{1}{3}} D^{\frac{2}{3}} T^{\frac{2}{3}} )$ &  &  \\
		\hline
		\multirow{4}{*}{Episodic}  
		&      \citet{domingues2020kernel}     &   $\widetilde{O}(S^{\textcolor{white}{\frac{1}{1}}}A^{\frac{1}{2}}\Delta^{\frac{1}{3}}H^{\frac{4}{3}}T^{\frac{2}{3}})$     &  \xmark    &   also metric spaces\\
		&   \cellcolor{gray!40} RestartQ-UCB &   $\widetilde{O}( S^{\frac{1}{3}}A^{\frac{1}{3}} \Delta^{\frac{1}{3}} H^{\textcolor{white}{\frac{1}{1}}}T^{\frac{2}{3}} )$     &  \cmark    &   \\
		&   \cellcolor{gray!40} Double-Restart Q-UCB &   $\widetilde{O}( S^{\frac{1}{3}}A^{\frac{1}{3}} \Delta^{\frac{1}{3}} H T^{\frac{2}{3}} \! +\! H^{\frac{3}{4}}T^{\frac{3}{4}} )$     &  \cmark    &  does not require $\Delta$ \\
		&   \cellcolor{gray!40} Lower bound   &   $\Omega( S^{\frac{1}{3}}A^{\frac{1}{3}} \Delta^{\frac{1}{3}} H^{\frac{2}{3}} T^{\frac{2}{3}} )$     &      &   \\ \hline\multirow{2}{*}{\pbox{3cm}{\wmedit{Linear}\\\wmedit{MDPs}}}  
		&      \citet{zhou2020nonstationary}     &   $\widetilde{O}(d^{\frac{4}{3}}\Delta^{\frac{1}{3}}H^{\frac{4}{3}}T^{\frac{2}{3}})$     &  \cmark    &   \\
		&   \citet{touati2020efficient}  &   $\widetilde{O}(d^{\frac{5}{4}}\Delta^{\frac{1}{4}}H^{\frac{5}{4}}T^{\frac{3}{4}})$    &   \cmark   &   \\ \hline
	\end{tabular}
	\caption{Dynamic regret comparisons for RL in non-stationary MDPs. $S$ and $A$ are the numbers of states and actions, $L$ is the number of abrupt changes, $D$ is the maximum diameter, \wmedit{$d$ is the dimension of the feature space for linear MDPs,} $H$ is the number of steps per episode, and $T$ is the total number of steps. \wmedit{All upper bounds listed in the table are high-probability results that hold with probability at least $1-\delta$ for some $\delta\in(0,1)$, and $\widetilde O(\cdot)$ suppresses logarithmic dependence on $S,A,T$ and $\frac{1}{\delta}$.} Gray cells denote results from this paper.  }\label{table:1}
\end{table*}

To resolve the aforementioned challenges, \citet{ortner2020variational} and \citet{cheung2020reinforcement} have proposed algorithms to guide learning in non-stationary MDPs. Although both model-based and model-free algorithms have been proposed for stationary RL, existing solutions for non-stationary RL are often built upon model-based methods. Nevertheless, it has been observed that model-based solutions often suffer from the following shortcomings:

\textbullet\ \textbf{Time- and space-inefficiency:} Model-based methods are in general more time- and space-consuming, and are less compatible with the design of modern deep RL architectures~\citep{jin2018q,zhang2020almost}. 

\textbullet\ \textbf{Inefficient exploration:} In~\citet{cheung2020reinforcement,CheungSLZ20}, an example was given to show that under non-stationarity, the estimated model can incorrectly indicate that transitioning between states is very unlikely. This suggests that model-based methods, which try to estimate the latent model, might suffer ``The Perils of Drift'' \citep{cheung2020reinforcement}. 

\textbullet\ \textbf{Limited applicability:} In an important application of nonstationary RL --- \emph{decentralized} multi-agent RL, the agents cannot observe the actions taken by the other agents. This information structure precludes model-based methods, as the explicit estimation of the state transition functions is hardly possible without observing all the agents' actions. 

These observations have thus motivated us to turn our attention to model-free methods, which, instead of maintaining estimates of the unknown underlying model, directly learn the Q-values.

\textbf{Main Contributions.} In this paper, we focus on the problem of designing model-free algorithms with near-optimal performances for non-stationary RL. Our contributions are as follows:
\begin{enumerate}[leftmargin=*,itemsep=-2pt]
	\item We introduce an algorithm named Restarted Q-Learning with Upper Confidence Bounds (RestartQ-UCB), which is the first model-free algorithm in the general setting of non-stationary RL. Our algorithm adopts a simple but effective restarting strategy~\citep{jaksch2010near,besbes2014stochastic} that resets the memory of the agent according to a calculated schedule. The restarting strategy ensures that our algorithm only refers to the most up-to-date experience for decision-making. RestartQ-UCB also utilizes an extra optimism term (in addition to the standard Hoeffding/Freedman-based bonus) for exploration to counteract the non-stationarity of the MDP. This additional bonus term, depending on the local variation budgets (i.e., the environmental variation in each restarting interval), guarantees that our optimistic $Q$-value is still an upper bound of the optimal $Q^\star$-value even when the environment changes. Our analysis shows that RestartQ-UCB achieves the lowest dynamic regret bound when compared to existing works in the literature;
	\item We conduct simulations showing that RestartQ-UCB achieves highly competitive cumulative rewards against a state-of-the-art solution \citep{zhou2020nonstationary}, while only taking $0.18\%$ of its computation time;
	\item We establish the first lower bounds  in non-stationary RL, which suggest  that our algorithm is optimal in all parameter dependences except for an $H^{\frac{1}{3}}$ factor, where $H$ is the episode length;
	\item To further showcase the flexibility and potential of non-stationary RL, we illustrate how it can be utilized to address the non-stationarity issue inherent in multi-agent RL. Specifically, we show that RestartQ-UCB can be readily applied to a multi-agent RL example against a slowly-changing opponent~\citep{radanovic2019learning,lee2020linear}. The setting we consider is a more practical and general decentralized learning setting, which entails model-free solutions;
	\item A preliminary version of this paper \citep{anonymous} has appeared in the Proceedings of the 38th International Conference on Machine Learning (ICML 2021). In the current version, 1) We have proved a new result of the Freedman-based RestartQ-UCB that no longer requires knowledge of the local variation budget (Theorem \ref{thm:freedman}); 2) We further show that our algorithm can easily remove the dependence on prior knowledge of the variation budget, a critical assumption commonly made in the literature~\citep{ortner2020variational,zhou2020nonstationary}. To do that, we propose a parameter-free algorithm that leverages a ``double restart'' strategy to adaptively learn the variation budget; 3) In addition, we discuss the application of our non-stationary RL algorithm in inventory control. Specifically, we demonstrate how to implement our RestartQ-UCB algorithm for the problem of inventory control across related, but \emph{different} products with time-varying demands; 4) We also provide detailed proofs that were missing in the conference version. 
\end{enumerate}  

\textbf{Related Works.} Dynamic regret of non-stationary RL has been mostly studied using model-based solutions. \citet{jaksch2010near} consider the setting where the MDP is allowed to change abruptly for $L$ times. 
A sliding window approach is proposed in~\citet{gajane2018sliding} under the same setting. 
\citet{ortner2020variational} generalize the previous setting by allowing the MDP to vary either abruptly or gradually at every step, subject to a total variation budget. 
\citet{cheung2020reinforcement} consider the same setting and 
introduce a Bandit-over-RL technique that adaptively tunes the algorithm without knowing the variation budget. Directly applying their method to our episodic setting will lead to a dynamic regret of $\widetilde{O}(S^{\frac{2}{3}} A^{\frac{1}{2}} \Delta^{\frac{1}{4}} H T^{\frac{3}{4}}).$ Although it may be possible to further obtain an improved dependence on $T$, this is sub-optimal in terms of $S$ and $A$. We remark that a recent (but later than ours) version of this paper develops a lower bound tailored to the infinite horizon undiscounted non-stationary RL, this is not directly applicable to our episodic non-stationary RL setting.

In a setting most similar to ours, \citet{domingues2020kernel} investigate non-stationary RL in the episodic setting, and propose a kernel-based approach when the state-action set forms a metric space. Their results can be reduced to an $\widetilde{O}(SA^{\frac{1}{2}}\Delta^{\frac{1}{3}}H^{\frac{4}{3}}T^{\frac{2}{3}})$ regret in the tabular case. \citet{fei2020dynamic} assume stationary transitions and adversarial 
full-information rewards, and their setting is not directly comparable with ours. 
Two concurrent works~\citet{zhou2020nonstationary} and \citet{touati2020efficient} consider non-stationary RL in linear MDPs, but their regret bounds, $\widetilde{O}(S^{\frac{4}{3}}A^{\frac{4}{3}}\Delta^{\frac{1}{3}}H^{\frac{4}{3}}T^{\frac{2}{3}})$ and $\widetilde{O}(S^{\frac{5}{4}}A^{\frac{5}{4}}\Delta^{\frac{1}{4}}H^{\frac{5}{4}}T^{\frac{3}{4}})$ when reduced to the tabular RL setting, respectively, are less competitive than ours. \wmedit{After an earlier version of this paper was made publicly available, \cite{wei2021non} have proposed a black-box reduction procedure that turns an RL algorithm in a (nearly-)stationary environment to a non-stationary RL algorithm. In the episodic setting, \citet{wei2021non} have achieved a strong dynamic regret bound of $\widetilde{O}(S^{\frac{1}{3}} A^{\frac{1}{3}} \Delta^{\frac{1}{3}} H^{\frac{5}{3}} T^{\frac{2}{3}})$ (with or without knowledge of the degree of non-stationarity). However, their regret bound has a worse dependence on $H$ when compared to ours, and it has been pointed out in \citet{wei2021non} that such a sub-optimality cannot be improved upon by using a Freedman-style confidence bound as we do. Their compelling theoretical guarantees also come at the cost of a rather sophisticated and memory-inefficient algorithmic design, which needs to maintain many instances of the stationary subroutine, and constantly switch among them.} Interested readers are referred to~\citet{padakandla2020survey} for a comprehensive survey on RL in non-stationary environments. Table~\ref{table:1} compares our regret bounds with existing results that tackle similar settings as ours. 
It can be seen that our result is the first one that achieves the optimal dependence on $S$ and $A$, and also establishes the tightest dependence on $H/D$ and $T$ among existing solutions in the literature, without relying on their assumptions.

Another related line of research studies online/adversarial MDPs~\citep{yu2009online,neu2010online,arora2012deterministic,yadkori2013online,dick2014online,wang2018reinforcement,lykouris2019corruption,jin2019learning}, but they mostly only allow variations in reward functions, and use the static regret as a performance metric. In addition, RL with low switching cost~\citep{bai2019provably} also shares a similar spirit as our restarting strategy since it also periodically forgets previous experiences. However, such algorithms do not address the non-stationarity of the environment, and their dynamic regret in terms of the variation budget is unclear. 

Non-stationarity has also been considered in bandit problems \citep{besbes2019optimal}. Within different non-stationary multi-armed bandit (MAB) settings, various methods have been proposed, including decaying memory and sliding windows~\citep{garivier2011upper,keskin2017chasing}, as well as restart-based strategies~\citep{auer2002nonstochastic,besbes2014stochastic,allesiardo2017non}. These methods largely inspired later research on non-stationary RL. A more recent line of work developed methods that do not require prior knowledge of the variation budget~\citep{karnin2016multi,cheung2019hedging} or the number of abrupt changes~\citep{auer2019adaptively}. Other related settings considered in the literature include Markovian bandits~\citep{tekin2010online,ma2018improvements,ZhouXCG20},  non-stationary contextual bandits~\citep{luo2018efficient,chen2019new}, linear bandits~\citep{cheung2019learning,zhao2020simple}, continuous-armed bandits~\citep{mao2020poly}, and learning with seasonal patterns \citep{ChenWW20}.

\textbf{Outline.} The rest of the paper is organized as follows: In Sections~\ref{sec:preliminary}, we introduce the mathematical model of our problem and necessary preliminaries. In Section~\ref{sec:algorithm}, we present our RestartQ-UCB algorithm. A dynamic regret analysis of RestartQ-UCB is provided in Section~\ref{sec:analysis}. In Section~\ref{sec:borl}, we further propose a parameter-free algorithm that does not require prior knowledge of the variation budget. In Section~\ref{sec:lower_bound}, we establish information-theoretical lower bounds. Simulation results are presented in Section~\ref{sec:simulations}. In Sections~\ref{sec:application} and~\ref{sec:application_inventory}, we discuss the applications of our method to two important scenarios: multi-agent RL and inventory control, respectively. Finally, we conclude the main part of the paper in Section~\ref{sec:conclusions}. Some supplementary material and proofs of all the results are included in eleven  appendices at the end of the paper.

\section{Preliminaries}\label{sec:preliminary}
\textbf{Model:} We consider an episodic RL setting where an agent interacts with a non-stationary MDP for $M$ episodes,  with each episode containing $H$ steps. We use a pair of integers $(m,h)$ as a \emph{time index} to denote the $h$-th step of the $m$-th episode.  The environment can be denoted by a tuple $(\mc{S},\mc{A}, H, P, r)$, where $\mc{S}$ is the finite set of states with $\abs{\mc{S}}=S$, $\mc{A}$ is the finite set of actions with $\abs{\mc{A}}=A$, $H$ is the number of steps in one episode, $P=\{P_h^m\}_{m\in[M],h\in[H]}$ is the set of transition kernels, and $r=\{r_h^m\}_{m\in[M],h\in[H]}$ is the set of mean reward functions. Specifically, when the agent takes action $a_h^m \in \mc{A}$ in state $s_h^m \in \mc{S}$ at the time  $(m,h)$, it will receive a random reward {$R_h^m(s_h^m,a_h^m)\in [0,1]$} with expected value $r_h^m(s_h^m,a_h^m)$, and the environment transitions to a next state $s_{h+1}^m$ following the distribution $P_h^m\left(\cdot \mid s_h^m,a_h^m\right)$. It is worth emphasizing that the transition kernel and the mean reward function depend both on $m$ and $h$, and hence the environment is non-stationary over time. The episode ends when $s_{H+1}^m$ is reached. We further denote $T=MH$ as the total number of steps. 

A deterministic policy $\pi:[M]\times [H]\times \mc{S}\ra \mc{A}$ is a mapping from the time index and state space to the action space, and we let $\pi_h^m(s)$ denote the action chosen in state $s$ at time $(m,h)$. Define $V_h^{m,\pi}: \mc{S}\ra \rr$ to be the value function under policy $\pi$ at time $(m,h)$, i.e., 
$$
V_{h}^{m,\pi}(s)\defeq\mathbb{E}\left[\sum_{h^{\prime}=h}^{H} r^m_{h^{\prime}}\left(s_{h^{\prime}}, \pi_{h^{\prime}}^m\left(s_{h^{\prime}}\right)\right) \mid s_{h}=s\right],
$$
where $ s_{h^{\prime}+1} \sim P_{h^{\prime}}^{m}\left(\cdot \mid s_{h^{\prime}}, a_{h^{\prime}}\right)$. Accordingly, the state-action value function $Q_h^{m,\pi}:\mc{S}\times\mc{A}\ra \rr$ is defined as:
\[
\begin{aligned}
Q_{h}^{m,\pi}(s, a)\defeq r^m_{h}(s, a) + \mathbb{E}\left[\sum_{h^{\prime}=h+1}^{H}\! r_{h^{\prime}}^m\left(s_{h^{\prime}}, \pi_{h^{\prime}}^m\left(s_{h^{\prime}}\right)\right) \mid s_{h}=s, a_{h}=a\right]
\end{aligned}
\]
For simplicity of notation, we let $P^m_hV_{h+1}(s,a)\defeq \ee_{s'\sim P^m_h(\cdot\mid s,a)}\left[V_{h+1}(s')\right]$. Then, the Bellman equation gives $V^{m,\pi}_h(s) = Q_h^{m,\pi}(s,\pi_h^m(s))$ and $Q_h^{m,\pi}(s,a) = (r^m_h +P_h^mV^{m,\pi}_{h+1})(s,a)$, and we also have $V_{H+1}^{m,\pi}(s) = 0,\forall s\in\mc{S}$ by definition. Since the state space, the action space, and the length of each episode are all finite, there always exists an optimal policy  {$\pi^{\star}$ that gives the optimal value $V_h^{m,\star}(s) \defeq V_h^{m,\pi^{\star}}(s) = \sup_\pi V_h^{m,\pi}(s),\forall s\in\mc{S}, m\in[M], h\in[H]$.} From the Bellman optimality equation, we have $V^{m,\star}_h(s) = \max_{a\in\mc{A}} Q_h^{m,\star}(s,a)$, where $Q_h^{m,\star}(s,a) \defeq (r^m_h +P_h^mV^{m,\star}_{h+1})(s,a)$, and $V_{H+1}^{m,\star}(s) = 0,\forall s\in\mc{S}$.

\textbf{Dynamic Regret:} The agent aims to maximize the cumulative expected reward over the entire $M$ episodes, by adopting some policy $\pi$. We measure the optimality of the policy $\pi$ in terms of its \emph{dynamic regret} \citep{cheung2020reinforcement,domingues2020kernel}, which compares the agent's policy with the optimal policy of each individual episode in hindsight: 
\[
\mc{R}(\pi, M) \defeq \sum_{m=1}^{M}\left(V_{1}^{m,\star}\left(s_{1}^{m}\right)-V_{1}^{m,\pi}\left(s_{1}^{m}\right)\right),
\]
where the initial state $s_1^m$ of each episode is chosen by an oblivious adversary~\citep{zhang2020almost}. Dynamic regret is a stronger measure than the standard (static) regret, which only considers the single policy that is optimal over all episodes combined. 

\textbf{Variation:} We measure the non-stationarity of the MDP in terms of its \emph{variation budget} in the mean reward function and transition kernels:
\[
\begin{aligned}
\Delta_r \defeq \sum_{m=1}^{M-1}\sum_{h=1}^{H} \sup_{s,a}|r_h^m(s,a) - r_h^{m+1}(s,a)|,\quad 
\Delta_p \defeq \sum_{m=1}^{M-1}\sum_{h=1}^{H} \sup_{s,a}\norm{P_h^m(\cdot \mid s,a) - P_h^{m+1}(\cdot\mid s,a)}_1, 
\end{aligned}
\]
where $\norm{\cdot}_1$ is the $L^1$-norm. Note that our definition of variation budgets only imposes restrictions on the summation of non-stationarity across two different episodes, and does not put any restriction on the difference between two consecutive steps in the same episode; that is, $P_h^m(\cdot\mid s,a)$ and $P_{h+1}^m(\cdot\mid s,a)$ are allowed to be arbitrarily different. We further let $\Delta = \Delta_r + \Delta_p$, and assume $\Delta>0$.

\section{Algorithm: RestartQ-UCB}\label{sec:algorithm}
\begin{algorithm*}[t]
	\For{epoch $d \gets 1$ to $D$}
	{
		\textbf{Initialize:} $V_h(s)\gets H-h+1, Q_h(s,a)\gets H-h+1, N_h(s,a)\gets 0, \check{N}_h(s,a)\gets 0,$ $\check{r}_h(s,a)\gets 0, \check{v}_{h}(s,a)\gets 0,\algedit{\check{\mu}_h(s,a)\gets 0,  \check{\sigma}_h(s,a)\gets 0, \mu^{\re}_h(s,a)\gets 0, \sigma^{\re}_h(s,a)\gets 0, V_h^{\re}(s)\gets H},$ for all $(s,a,h)\in \mc{S}\times \mc{A}\times [H]$\;
		\For{episode $k \gets (d-1)K+1$ to $\min\{dK,M\}$}
		{
			observe $s_1$\;
			\For{step $h \gets 1$ to $H$}
			{
				Take action $a_h  \gets \arg \max_a Q_h(s_h ,a)$, receive $R_h(s_h ,a_h )$, and observe $s_{h+1}$\label{line:6}\;
				$\check{r}_h(s_h,a_h)\gets \check{r}_h(s_h,a_h)+R_h(s_h,a_h), \check{v}_{h}(s_h,a_h) \gets \check{v}_{h}(s_h,a_h) + V_{h+1}(s_{h+1})$\;
				\algedit{$\check{\mu}(s_h,a_h )\gets \check{\mu}(s_h,a_h )+ V_{h+1}(s_{h+1}) - V_{h+1}^{\re}(s_{h+1})$\;}
				\algedit{$\check{\sigma}(s_h,a_h )\gets \check{\sigma}(s_h,a_h )+ \left( V_{h+1}(s_{h+1}) - V_{h+1}^{\re}(s_{h+1})\right)^2$\;}
				\algedit{$\mu^{\re}(s_h,a_h )\gets \mu^{\re}(s_h,a_h )+V_{h+1}^{\re}(s_{h+1}), \sigma^{\re}(s_h,a_h )\gets \sigma^{\re}(s_h,a_h )+ ( V_{h+1}^{\re}(s_{h+1}) )^2$\;}
				$n\defeq N_h(s_h ,a_h )\gets N_h(s_h ,a_h ) + 1, \check{n}\defeq \check{N}_h(s_h ,a_h ) \gets \check{N}_h(s_h ,a_h ) + 1$\;
				\If{$N_h(s_h,a_h)\in \mc{L}$}
				{
					// \texttt{Reaching the end of the stage}
					
					$b_h \gets \sqrt{\frac{H^2}{\check{n}}\iota} +\sqrt{\frac{1}{\check{n}}\iota},\ b_\Delta \gets \Delta_r^{(d)} + H\Delta_p^{(d)}$\;
					\algedit{
					$\underline{b}_h \!\gets\! 2\sqrt{\frac{\sigma^{\re} / n - (\mu^{\re}/n)^2}{n}\iota} + 2\sqrt{\frac{\check{\sigma}/\check{n} - (\check{\mu} / \check{n})^2}{\check{n}}\iota} + 5( \frac{H\iota}{n}\! + \!\frac{H\iota}{\check{n}} \!+\! \frac{H\iota^{3/4}}{n^{3/4}}\! + \!\frac{H\iota^{3/4}}{\check{n}^{3/4}})\!+\!\sqrt{\frac{1}{\check{n}}\iota}$\;
					}
					$Q_h(s_h ,a_h )\gets \min\left\{ \frac{\check{r}}{\check{n}} \!+\! \frac{\check{v}}{\check{n}} \!+\! b_h \!+\! 2b_\Delta, \algedit{\frac{\check{r}}{\check{n}} \!+\! \frac{\mu^{\re}}{n}\!+\!\frac{\check{\mu}}{\check{n}} \!+\!2\underline{b}_h \!+\!4b_\Delta, } Q_h(s_h,a_h ) \right\}$;\hfill   $(*)$\label{line:16}
					
					$V_h(s_h )\gets \max_a Q_h(s_h ,a )$\;
					$\check{N}_h(s_h ,a_h)\gets 0, \check{r}_h(s_h,a_h)\gets 0, \check{v}_h(s_h,a_h)\gets 0, \algedit{\check{\mu}_h(s_h,a_h)\gets 0,\check{\sigma}_h(s_h,a_h)\gets 0}$\;
				}
			
				\algedit{\textbf{if} $\sum_a N_h(s_h,a) = N_0$ \textbf{then}\qquad  \texttt{// Learn the reference value}}
				{
					
					\algedit{\ \ \ \ \ $V_h^{\re}(s_h)\gets V_h(s_h)$\;}
				}
			}
		}
	} 
	\caption{RestartQ-UCB (Hoeffding/\algedit{Freedman})}\label{alg:RQUCB}
\end{algorithm*}

\wmedit{We present our algorithm Restarted Q-Learning with Hoeffding/Freedman Upper Confidence Bounds  {(RestartQ-UCB Hoeffding/Freedman)} in Algorithm~\ref{alg:RQUCB}. For illustrative purposes, we start with a simpler RestartQ-UCB algorithm with Hoeffding-style bonus terms, which only executes the pseudocode colored in black in Algorithm~\ref{alg:RQUCB}. Further incorporating the gray parts in Algorithm~\ref{alg:RQUCB} leads to the RestartQ-UCB  algorithm with Freedman-style bonus terms and reference-advantage decomposition~\citep{zhang2020almost}, which achieves a sharper dynamic regret bound at the cost of a slightly more involved analysis. }

\wmedit{Common to both the Hoeffding and the Freedman bonus terms,} RestartQ-UCB breaks the $M$ episodes into $D$ \emph{epochs}, with each epoch containing $K=\lceil \frac{M}{D}\rceil$ episodes (except for the last epoch which possibly has less than $K$ episodes). \wmedit{With a large value of $D$, Algorithm~\ref{alg:RQUCB} restarts more frequently to adjust to the potential variations of the environment, at the cost of spending more time searching for new optimal policies. On the contrary, a small value of $D$ would lead to running stable policies for long periods of time with less frequent restarts, but the resulting algorithm might not be able to adjust to the environmental variations rapidly enough. To strike a balance, we set the number of epochs to be $D = S^{-\frac{1}{3}} A^{-\frac{1}{3}} \Delta^{\frac{2}{3}} H^{-\frac{2}{3}} T^{\frac{1}{3}}$ so as to achieve the optimal dynamic regret bound, and such a choice will be justified later in our analysis.} RestartQ-UCB periodically restarts a Q-learning algorithm with UCB exploration at the beginning of each epoch, thereby addressing the non-stationarity of the environment. For each $d\in[D]$, define $\Delta^{(d)}_r$ to be the \emph{local variation budget} of the mean reward function within epoch $d$. By definition, we have $\sum_{d=1}^D \Delta_r^{(d)} \leq \Delta_r$.  Define the local variation budget of transitions $\Delta_p^{(d)}$ analogously.

Since our algorithm essentially invokes the same procedure for every epoch, in the following, we focus our analysis on what happens inside one epoch only (and without loss of generality, we focus on epoch $1$, which contains episodes $1,2,\dots,K$). At the end of our analysis, we will merge the results across all epochs. 

For each triple $(s,a,h)\in \mc{S}\times \mc{A} \times [H]$, we divide the visitations (within epoch $1$) to the triple into multiple \emph{stages}, where the length of the stages increases exponentially at a rate of $(1+\frac{1}{H})$. Specifically, let $e_1 = H$, and $e_{i+1} = \lfloor (1+\frac{1}{H}) e_i\rfloor, i\geq 1$ denote the lengths of the stages. Further, let the partial sums $\mc{L}\defeq \{ \sum_{i=1}^{j} e_i \mid j = 1,2,3,\dots \}$ denote the set of the ending times of the stages. We remark that the stages are defined for each individual triple $(s,a,h)$, and for different triples the starting and ending times of their stages do not necessarily align in time. Such a definition of stages is mostly motivated by the design of the learning rate $\alpha_t = \frac{H+1}{H+t}$ in \citet{jin2018q}. It ensures that only the last $O(1/H)$ fraction of samples is given non-negligible weights when used to estimate the optimistic $Q_h(s,a)$ values, while the first $1-O(1/H)$ fraction is forgotten \citep{zhang2020almost}.  We set $\iota \defeq \log\left(\frac{2}{\delta}\right)$, where $\delta$ is an input parameter that can be set by us. 

Recall that the time index $(k,h)$ represents the $h$-th step of the $k$-th episode. At each step $(k,h)$, we take the optimal action with respect to the optimistic $Q_h(s,a)$ value (Line~\ref{line:6} in Algorithm~\ref{alg:RQUCB}), which is designed as an optimistic estimate of the optimal $Q_h^{k,\star}(s,a)$ value of the corresponding episode. For each triple $(s,a,h)$, we update the optimistic $Q_h(s,a)$ value at the end of each stage, using samples only from this latest stage that is about to end (Line~\ref{line:16} in Algorithm~\ref{alg:RQUCB}).  \wmedit{The optimism in $Q_h(s,a)$ comes from two bonus terms $b_h\algedit{/\underline{b}_h}$ and $b_\Delta$, where $b_h\algedit{/\underline{b}_h}$ is a standard Hoeffding/Freedman-based optimism that is commonly used in upper confidence bounds~\citep{jin2018q,zhang2020almost}, and $b_\Delta$ is the extra optimism that we need to take into account because of the non-stationarity of the environment.} The definition of $b_\Delta$ requires knowledge of the local variation budget in each epoch, which is a rather strong assumption in practice. However, we can further show (later in Theorems~\ref{thm:regret_no_budget} and~\ref{thm:freedman}) that if we simply replace Equation ($*$) in Algorithm~\ref{alg:RQUCB} with the following update rule:  
\begin{align} 
Q_h(s_h ,a_h )\gets \min\left\{ \frac{\check{r}}{\check{n}} + \frac{\check{v}}{\check{n}} + b_h, \algedit{\frac{\check{r}}{\check{n}} + \frac{\mu^{\re}}{n}+\frac{\check{\mu}}{\check{n}} +2\underline{b}_h, } Q_h(s_h,a_h ) \right\}\label{eqn:new_update}  
\end{align}
then our algorithm can achieve the same regret without  assumptions on the local variation budget.

\wmedit{Compared with the Hoeffding-based algorithm, there are two major improvements in the Freedman-based one. The first improvement is the replacement of the Hoeffding-based bonus term $b_h^k$ with a tighter term $\underline{b}_h^k$. The latter term takes into account the second moment information of the random variables, which allows sharper tail bounds that rely on second moments to come into use (in our case, the Freedman's inequality). The second improvement is a variance reduction technique, or more specifically, the reference-advantage decomposition as coined in~\citet{zhang2020almost}. The intuition is to first learn a reference value function $V^\re$ that serves as a roughly accurate estimate of the optimal value function $V^\star$ in each epoch. The goal of learning the optimal value function $V^\star = V^\re + (V^*-V_\re)$ can hence be decomposed into estimating the two terms $V^\re$ and $V^*-V_\re$. The reference value $V^\re$ is a fixed term, and can be accurately estimated using a large number of samples (in Algorithm~\ref{alg:RQUCB}, we estimate $V^\re$ only when we have $N_0 = c SAH^6\iota$ samples for a large constant $c$). The advantage term $V^* - V^\re$ can also be estimated more accurately due to the reduced variance. }

\section{Analysis}\label{sec:analysis}
In this section, we present our main result---a dynamic regret analysis of the RestartQ-UCB algorithm. Our first result on RestartQ-UCB with Hoeffding-style bonus terms is summarized in the following theorem. Complete proofs of its supporting lemmas are given in Appendix~\ref{app:lemmas}. 

\begin{theorem}~\label{thm:regret}
	(Hoeffding) For $T= \Omega(SA\Delta H^2)$, and for any $\delta \in (0,1)$, with probability at least $1-\delta$, the dynamic regret of RestartQ-UCB with Hoeffding bonuses is bounded by $\widetilde{O}( S^{\frac{1}{3}}A^{\frac{1}{3}} \Delta^{\frac{1}{3}} H^{\frac{5}{3}} T^{\frac{2}{3}} )$, where $\widetilde{O}(\cdot)$ hides poly-logarithmic factors of $S,A,T$ and $1/\delta$. 
\end{theorem}

Our proof relies on the following technical lemma, stating that for any triple $(s,a,h)$, the difference of their optimal $Q$-values at two different episodes $1 \leq k_1 < k_2 \leq K$ is bounded by the variation of this epoch.

\begin{lemma}\label{lemma:optimalQ}
	For any triple $(s,a,h)$ and any $1 \leq k_1 < k_2\leq K$, it holds that $|Q_h^{k_1,\star}(s,a) - Q_h^{k_2,\star}(s,a)|\leq \Delta_r^{(1)} +H \Delta_p^{(1)}$.
\end{lemma}

Let $Q_{h}^{k}(s, a)$ denote the value of $Q_{h}(s, a)$ at the \emph{beginning} of the $k$-th episode in RestartQ-UCB Hoeffding. The following lemma states that the optimistic $Q$-value $Q_h^k(s,a)$ is an upper bound of the optimal $Q$-value $Q_h^{k,\star}(s,a)$ with high probability. Note that we only need to show that the event holds with probability $1-\texttt{poly}(S,A,K,H)\delta$, because we can replace $\delta$ with $\delta / \texttt{poly}(S,A,K,H)$ in the end to get the desired high probability bound without affecting the polynomial part of the regret bound. 

\begin{lemma}\label{lemma:optimistic_bound}
	(Hoeffding) For $\delta \in (0,1)$, with probability at least $1-2KH\delta$, it holds that $Q_{h}^{k,\star}(s, a) \leq Q_{h}^{k+1}(s, a) \leq Q_{h}^{k}(s, a),\forall (s,a,h,k)\in \mc{S}\times\mc{A}\times[H]\times[K]$. 
\end{lemma} 

Building upon Lemmas~\ref{lemma:optimalQ} and~\ref{lemma:optimistic_bound}, a complete proof of Theorem~\ref{thm:regret} is given in Appendix~\ref{app:proof_regret}. We remark that Algorithm~\ref{alg:RQUCB} relies on the assumption that the local variations $b_\Delta$ are known a priori, which is a strong but commonly made assumption in the literature on non-stationary RL~\citep{ortner2020variational,zhou2020nonstationary}. To the best of our knowledge, existing restart-based solutions either crucially rely on this local variation assumption~\citep{ortner2020variational}, or suffer a severe regret degeneration after removing this assumption~\citep{zhou2020nonstationary}. Interestingly, in the following theorem, we show that this assumption can be safely removed in our approach without affecting the regret bound. The only modification to the algorithm is to replace the $Q$-value update rule in Equation ($*$) of Algorithm~\ref{alg:RQUCB} with the new update rule in Equation~\eqref{eqn:new_update}. 

\begin{theorem}~\label{thm:regret_no_budget}
	(Hoeffding, no local budgets) For $T= \Omega(SA\Delta H^2)$, and for any $\delta \in (0,1)$, with probability at least $1-\delta$, the dynamic regret of RestartQ-UCB with Hoeffding bonuses and no knowledge of local budgets is bounded by $\widetilde{O}( S^{\frac{1}{3}}A^{\frac{1}{3}} \Delta^{\frac{1}{3}} H^{\frac{5}{3}} T^{\frac{2}{3}} )$, where $\widetilde{O}(\cdot)$ hides poly-logarithmic factors of $S,A,T$ and $1/\delta$. 
\end{theorem}
	
To understand why this simple modification works, notice that in ($*$) we add exactly the same value $2b_\Delta$ to the upper confidence bounds of all $(s,a)$ pairs in the same epoch. Subtracting the same value from all optimistic $Q$-values simultaneously should not change the choice of actions in future steps. \wmedit{The only difference is that the new ``optimistic'' $Q_h^k(s_h,a_h)$ values would no longer be strict upper bounds of the optimal $Q_h^{k,\star}(s_h,a_h)$ anymore, but instead ``upper bounds'' subject to some error term induced by $b_\Delta$. Specifically, since $Q_h(s_h,a_h)$ is updated using $V_{h+1}(s_{h+1})$, which in turn contains some error in terms of $b_\Delta$, the error will propagate across the steps. By properly tracking such error terms, we can see that there are in total $H-h+1$ copies of the $2b_\Delta$ error accumulated from step $H$ back to step $h$.  This leads to the following variant of Lemma~\ref{lemma:optimistic_bound} that quantifies the error terms in the new ``optimistic'' bounds. }

\begin{lemma}\label{lemma:new_optimistic_bound}
	(Hoeffding, no local budgets) Suppose that we have no prior knowledge of the local variations and replace the update rule ($*$) in RestartQ-UCB Hoeffding with Equation~\eqref{eqn:new_update}. For $\delta \in (0,1)$, with probability at least $1-2KH\delta$, it holds that $Q_{h}^{k,\star}(s, a) - 2(H-h+1)b_\Delta \leq Q_{h}^{k+1}(s, a) \leq Q_{h}^{k}(s, a),\forall (s,a,h,k)\in \mc{S}\times\mc{A}\times[H]\times[K]$. 
\end{lemma} 

\begin{remark}
	The easy removal of the local budget assumption is non-trivial in the design of the algorithm, and to the best of our knowledge is absent in the non-stationary RL literature with restarts. In fact, it has been shown in a concurrent work~\citep{zhou2020nonstationary} that removing this assumption could lead to a much worse regret bound (cf. Corollary 2 and Corollary 3 therein). 
\end{remark}
Replacing the Hoeffding-based upper confidence bound with a Freedman-style one will lead to a tighter regret bound, summarized in Theorem~\ref{thm:freedman} below. \wmedit{To remove the local budget assumption, we also need to replace the update rule $(*)$ in Algorithm~\ref{alg:RQUCB} with Equation~\eqref{eqn:new_update}.} The proof of the theorem follows a similar procedure as in the proof of Theorem~\ref{thm:regret_no_budget}, and is given in Appendix~\ref{app:proof_freedman}. It relies on a  reference-advantage decomposition technique for variance reduction as  in~\citet{zhang2020almost}. 
\begin{theorem}~\label{thm:freedman}
	(Freedman, \wmedit{no local budgets}) For $T$ greater than some polynomial of $S, A,\Delta$ and $H$, and for any $\delta \in (0,1)$, with probability at least $1-\delta$, the dynamic regret of RestartQ-UCB with Freedman bonuses (Algorithm~\ref{alg:RQUCB} including the gray parts) is upper bounded by $\widetilde{O}( S^{\frac{1}{3}}A^{\frac{1}{3}} \Delta^{\frac{1}{3}} H T^{\frac{2}{3}} )$, where $\widetilde{O}(\cdot)$ hides poly-logarithmic factors of $S,A,T$ and $1/\delta$.
\end{theorem}
\begin{remark}[From High Probability Regret Bound to Expected Regret Bound]
	We note that $\delta$ is an input parameter, and our high probability regret bounds can immediately imply expected regret bounds. In all the above theorems presented in this section, the dynamic regret depends on $1/\delta$ through logarithmic terms. Since the regret can at most be $O(T)$, by setting $\delta=1/T$, one can retain the same regret bound in an expectation sense. For instance, in Theorem \ref{thm:freedman}, by setting $\delta = 1/T,$ we have that with probability at least $1-\delta,$ the regret is $\widetilde{O}( S^{\frac{1}{3}}A^{\frac{1}{3}} \Delta^{\frac{1}{3}} H T^{\frac{2}{3}} ),$ while with probability at most $\delta,$ the regret is $O(T).$ Hence, the expected regret of the algorithm is $(1-\delta)\widetilde{O}( S^{\frac{1}{3}}A^{\frac{1}{3}} \Delta^{\frac{1}{3}} H T^{\frac{2}{3}} )+\delta O(T)=\widetilde{O}( S^{\frac{1}{3}}A^{\frac{1}{3}} \Delta^{\frac{1}{3}} H T^{\frac{2}{3}} )$
\end{remark}

\section{Unknown Variation Budgets}\label{sec:borl}
 
In Theorem~\ref{thm:freedman}, we have removed the assumption on the knowledge of ``local''  variation budgets $\Delta^{(d)}_r$ and $\Delta_p^{(d)}$ for $d \in [D]$, but the design of the algorithm still relies on knowledge of the ``total'' variation budget $\Delta$. Specifically, to achieve the dynamic regret bound presented in Theorem~\ref{thm:freedman}, we need to set the number of epochs to  $D^\star=S^{-\frac{1}{3}}A^{-\frac{1}{3}}\Delta^{\frac{2}{3}}T^{\frac{1}{3}}$, which clearly requires to know $\Delta$ in advance. To further overcome such a limitation, in this section, we propose a parameter-free algorithm that adaptively learns the variation budget $\Delta$ when it is unknown a priori, while still achieving sublinear dynamic regret in $T$.

\begin{algorithm*}[t]
	\textbf{Input:} Parameters $W, \mc{J}$, $\alpha$, and $\gamma$ as given in Equation~\eqref{eqn:para} and 
	\eqref{eqn:gamma}. 
	
	\textbf{Initialize:} Weights of the bandit arms $s_1(j) =  \exp\left( \frac{\alpha \gamma}{3}\sqrt{\frac{\ceil{M/W}}{J+1}}\right)$ for $j = 0, 1,\dots, \ceil{\ln W}$.
	
	\For{phase $i \gets 1$ to $\ceil{\frac{M}{W}}$}
	{
		$p_i(j) \gets (1-\gamma) \frac{s_i(j)}{\sum_{j'=0}^J s_i(j')} + \frac{\gamma}{J+1}, \forall j =0,1,\dots, J$\;
		Draw an arm $A_i$ from $\{0,\dots, J\}$ randomly according to the probabilities $p_i(0),\dots, p_i(J)$\;
		Set the estimated number of epochs $D_i \gets \floor{\frac{TW^{\frac{A_i}{J}}}{SAH^2 W}}$\;
		Run a new instance of Algorithm~\ref{alg:RQUCB} (including gray parts) for $W$ episodes with parameter value $D\gets D_i$\;
		Observe the cumulative reward $R_i$ from the last $W$ episodes\;
		\For{arm $j\gets 0,1,\dots, J$}
		{
			$\hat{R}_i(j) \gets R_i\mb{I}\{j=A_i\} /\left( WH p_i(j)\right) $\;
			$s_{i+1}(j)\gets s_{i}(j)\exp\left( \frac{\gamma}{3(J+1)} \left(\hat{R}_i(j)+ \frac{\alpha}{p_i(j)\sqrt{(J+1)\ceil{M/W}}}\right) \right)$\;
		}
	} 
	\caption{Double-Restart Q-UCB}\label{alg:borl}
\end{algorithm*}

Our new algorithm, Double-Restart Q-UCB, for the unknown variation budget setting is presented in Algorithm~\ref{alg:borl}. Inspired by the Bandit-over-Bandit algorithm \citep{cheung2019hedging,cheung2020reinforcement} that adaptively tunes the algorithm parameters in a linear bandit problem, we also use a multi-armed bandit algorithm as a master procedure to learn the optimal value $D^\star$ of $D$. Given a set $\mc{J}$ of candidate values for $D$, the idea of our algorithm is to first divide the time horizon $T$ into multiple \emph{phases}, and then in each phase we experiment with one candidate value from the set $\mc{J}$. If we choose values from $\mc{J}$ properly using a bandit algorithm, the cumulative reward we obtain through this experimentation procedure should be close to the performance of using the best fixed candidate from $\mc{J}$ in hindsight. Since the underlying environment need not drift according to any statistical pattern, we use an adversarial bandit algorithm Exp3.P~\citep{auer2002nonstochastic} to defend against the possibly adversarial changes of the best $D$ value in each phase. 

\begin{figure}[!htbp]
	\centering\includegraphics[width=.9\textwidth]{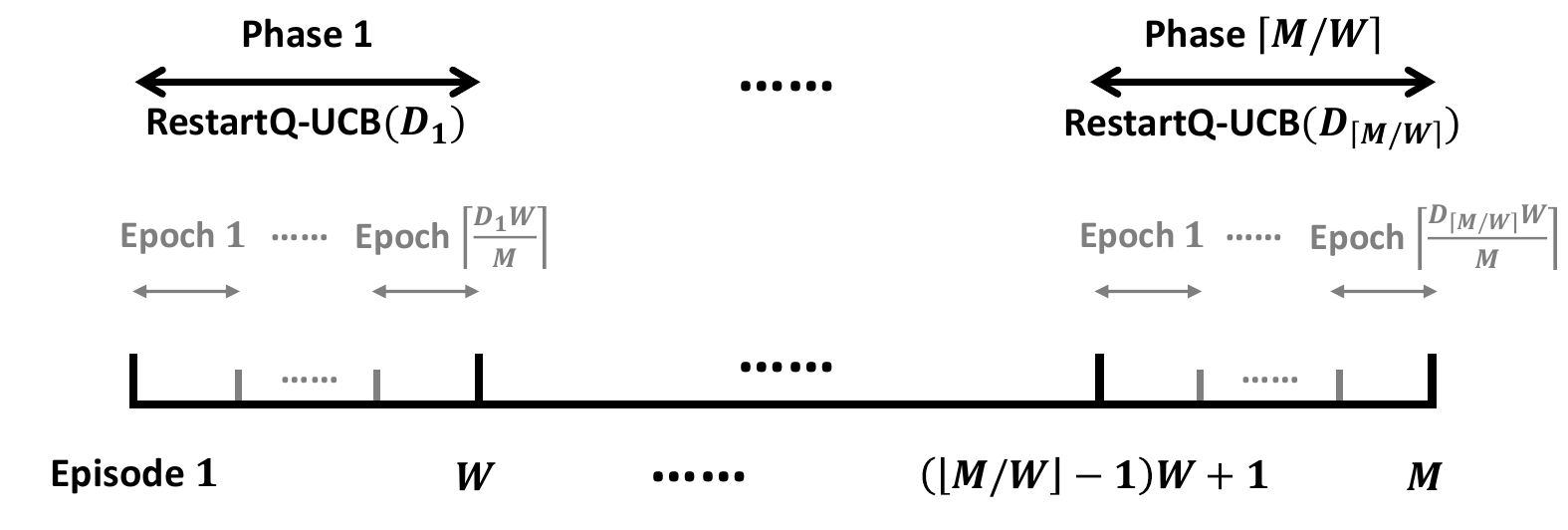}
	\caption{Structure of the Double-Restart Q-UCB algorithm.}\label{fig:structure}
\end{figure}

We sketch the high-level structure of the Double-Restart Q-UCB algorithm in Figure~\ref{fig:structure} to help clarify any possible confusion regarding our definitions of ``phases'', ``epochs'', and ``episodes''. Concretely, we divide the overall $M$ episodes into $\ceil{\frac{M}{W}}$ phases, each phase containing $W\in\nn_+$ episodes (except that the last phase could have less than $W$ episodes). At the beginning of each phase $i$, we start a new instance of Algorithm~\ref{alg:RQUCB} (including gray parts) with a candidate value of $D_i\in\mc{J}$ to be experimented in this phase. Since Algorithm~\ref{alg:RQUCB} itself is a restart-based process, it further sub-divides the $W$ episodes in phase $i$ into $\ceil{\frac{D_i W}{M}}$ epochs. To understand this value, suppose $D_i$ is an appropriate value for $D$, such that dividing the overall horizon into $D_i$ epochs leads to near-optimal dynamic regret. Then, since the overall horizon contains $M$ episodes while each phase only contains $W$ episodes, we should only divide each phase into $\ceil{\frac{D_i W}{M}}$ epochs to reflect the corresponding consequence of choosing $D_i$ as the overall number of epochs. Since we restart Algorithm~\ref{alg:RQUCB} in each phase and Algorithm~\ref{alg:RQUCB} in turn restarts an optimistic Q-learning sub-routine in each epoch, our overall algorithm exhibits a double-loop restarting behavior, and hence the name Double-Restart Q-UCB.

In the following, we instantiate the choices of the set $\mc{J}$, the phase length $W$, as well as the parameter values used in the Exp3.P bandit algorithm. First, we define
\begin{equation}\label{eqn:para}
W = \sqrt{HT}, J = \ceil{\ln W}, \text{ and } \mc{J} = \left\{\floor{\frac{T}{SAH^2 W}}, \floor{\frac{TW^{\frac{1}{J}}}{SAH^2 W}},\floor{\frac{TW^{\frac{2}{J}}}{SAH^2 W}},\dots, \floor{\frac{TW}{SAH^2 W}} \right\}, 
\end{equation}
where $\mc{J}$ is the set of candidate values for $D$ and we can see that $|\mc{J}| = \ceil{\ln W} + 1 = J + 1$. Each candidate value in $\mc{J}$ is also called an ``arm'' in the language of bandits, and we use ``arm $j$'' to refer to the candidate value $\floor{\frac{TW^{\frac{j}{J}}}{SAH^2 W}}$ for $j = 0,1,\dots, J$. We initialize the weights of the bandit arms by $s_1(j) = \exp\left( \frac{\alpha \gamma}{3}\sqrt{\frac{\ceil{M/W}}{J+1}}\right)$ for $j = 0, 1,\dots, J$, where  as specified in \citet{auer2002nonstochastic},
\begin{equation}\label{eqn:gamma}
\alpha = 2\sqrt{\ln \left(\ceil{M/W}(J+1)/\delta\right)},\text{ and } \gamma = \min\left\{ \frac{3}{5}, 2\sqrt{\frac{3}{5}\frac{(J+1)\ln(J+1)}{\ceil{M/W}}} \right\},
\end{equation}
for some failure probability $\delta>0$. 
At the beginning of each phase $i \in \{1,2,\dots, \ceil{\frac{M}{W}}\}$, we randomly draw an arm $j$ with probability $p_i(j)$ that is calculated from the weights
\[
p_i(j) = (1-\gamma) \frac{s_i(j)}{\sum_{j'=0}^J s_i(j')} + \frac{\gamma}{J+1}, \forall j =0,1,\dots, J.
\]

We set our estimated parameter $D_i$ to be the value associated with the selected arm $j$ in the set $\mc{J}$. We then run Algorithm~\ref{alg:RQUCB} for $W$ episodes by setting the number of epochs to be $D = D_i$. To put it in another way, we execute a new instance of Algorithm~\ref{alg:RQUCB} for $\ceil{\frac{D_i W}{M}}$ epochs, where each epoch contains $K_i = \floor{\frac{M}{D_i}}$ episodes. We collect the cumulative reward $R_i$ from the aforementioned $W$ episodes. The normalized value $R_i/(WH) \in [0,1]$ hence corresponds to the reward of playing the selected arm in time step $i$ of the bandit problem. Finally, we update the weights of the bandit arms based on the observed reward, using the following update rule specified in the Exp3.P algorithm:
\[
s_{i+1}(j)\gets s_{i}(j)\exp\left( \frac{\gamma}{3(J+1)} \left(\hat{R}_i(j)+ \frac{\alpha}{p_i(j)\sqrt{(J+1)\ceil{M/W}}}\right) \right), 
\]
where $\hat{R}_i(j) = R_i\mb{I}\{j=A_i\} /\left( WH p_i(j)\right),\forall j = 0,1,\dots,J$, and $A_i$ denotes the arm selected at phase $i$. 

The following result states that our Double-Restart Q-UCB algorithm achieves a sublinear dynamic regret in $T$, without requiring knowledge of the (total) variation budget $\Delta$. 

\begin{theorem}~\label{thm:borl}
	(\wmedit{Freedman}, no total budgets) For $T$ greater than some polynomial of $S, A,\Delta$ and $H$, and for any $\delta \in (0,1)$, with probability at least $1-\delta$, the dynamic regret of Double-Restart Q-UCB with Freedman bonuses and no prior knowledge of the total variation budget $\Delta$ is bounded by \wmedit{$\widetilde{O}( S^{\frac{1}{3}}A^{\frac{1}{3}} \Delta^{\frac{1}{3}} H T^{\frac{2}{3}} + H^{\frac{3}{4}}T^{\frac{3}{4}} )$}, where $\widetilde{O}(\cdot)$ hides poly-logarithmic factors. 
\end{theorem}

The regret bound in Theorem~\ref{thm:borl} consists of two terms: The first term is the dynamic regret of using the optimal candidate value $D^\dagger\in\mc{J}$ of the number of epochs. This term is in the same order as the known-variation case (Theorem~\ref{thm:freedman}), because we have discretized the candidate value set $\mc{J}$ at a proper granularity such that the optimal candidate value $D^\dagger\in\mc{J}$ approximates the actual optimal value $D^\star$. The second regret term in Theorem \ref{thm:borl} is caused by the regret of learning the optimal candidate value inside $\mc{J}$ using the Exp3.P algorithm. Due to the additional step of estimating the unknown variation budget, the overall dynamic regret bound becomes slightly worse in terms of its dependence on $T$ (from $\widetilde{O}(T^{\frac{2}{3}})$ in Theorem \ref{thm:freedman} to $\widetilde{O}(T^{\frac{3}{4}})$). Such a degradation seems unavoidable under the current framework as it has also appeared in a similar bandit scenario~\citep{cheung2019learning}. 

\begin{remark}[Comparison with \citet{CheungSLZ20}]
	We follow the Bandit-over-RL technique to utilize a separate bandit algorithm to select the key parameters for our algorithm. But we have to emphasize that the resulting algorithm is simpler and more practical for implementation. This is because our Double-Restart Q-UCB algorithm is essentially running a stationary Q-UCB algorithm in between restarts. In contrast, the algorithm in \citet{CheungSLZ20}) relies on a carefully tuned sliding-window update schedule. More importantly, we point out that such a design (together with our new analysis) can lead to an improved dynamic regret bound in terms of $S$ and $A$ (even with the Hoeffding-style bonus terms similar to \citet{CheungSLZ20}). This exhibits the advantage of our design compared to that of \citet{CheungSLZ20}, which combines restart and sliding-window.
\end{remark}

\section{Lower Bounds}\label{sec:lower_bound}
In this section, we provide information-theoretical lower bounds of the dynamic regret to characterize the fundamental limits of any algorithm in non-stationary RL.

\begin{theorem}~\label{thm:lb}
	For any algorithm, there exists an episodic non-stationary MDP such that the dynamic regret of the algorithm is at least $\Omega( S^{\frac{1}{3}}A^{\frac{1}{3}} \Delta^{\frac{1}{3}} H^{\frac{2}{3}} T^{\frac{2}{3}} )$.
\end{theorem}
\begin{proofsketch}
	The proof of our lower bound relies on the construction of a ``hard instance'' of non-stationary MDPs. The instance we construct is essentially an MDP with piecewise constant dynamics on each \emph{segment} of the horizon, and its dynamics experience an abrupt change at the beginning of each new segment. Specifically, we divide the horizon $T$ into $L$ segments\footnote{The definition of segments is irrelevant to, and should not be confused with, the notion of epochs we previously defined. }, where each segment has $T_0 \defeq \floor{\frac{T}{L}}$ steps and contains $M_0\defeq \floor{\frac{M}{L}}$ episodes. Within each segment, the system dynamics of the MDP do not vary, and we construct the dynamics for each segment in a way such that the instance is a hard instance of stationary MDPs on its own. The MDP within each segment is essentially similar to the hard instances constructed in \citet{osband2016lower,jin2018q}. Between two consecutive segments, the dynamics of the MDP change abruptly, and we let the dynamics vary in a way such that no information learned from previous interactions with the MDP can be used in the new segment. In this sense, the agent needs to learn a new hard MDP in each segment. Finally, optimizing the value of $L$ and the variation magnitude between consecutive segments (subject to the constraints of the total variation budget) leads to our lower bound. 
\end{proofsketch}

\begin{remark}
	We emphasize that in our construction of the worst-case non-stationary MDP, we only let the state transition kernel vary over time but keep the reward functions fixed. By doing so, we are able to provide a lower bound of order $\Omega( S^{\frac{1}{3}}A^{\frac{1}{3}} \Delta^{\frac{1}{3}} H^{\frac{2}{3}} T^{\frac{2}{3}} ).$ Recall that the upper bound stated in Theorem \ref{thm:freedman} is  $O( S^{\frac{1}{3}}A^{\frac{1}{3}} \Delta^{\frac{1}{3}} HT^{\frac{2}{3}} ),$ and hence our upper and lower bounds match in terms of $\Delta~(=\Delta_r+\Delta_p).$
\end{remark}
\begin{remark}[Tightness of Our Results]
	For our setting, we conjecture that the lower bound can be improved. Our current construction of the lower bound relies on a chain of $H$ copies of ``JAO MDPs'' \cite{jaksch2010near}. The non-stationarity is achieved by changing the transitions abruptly after a fixed time period, and such a change applies \emph{simultaneously} across all $H$ copies of JAO MDPs. One possible direction is to construct the lower bound instances such that the state transition kernel is allowed to vary within the same episode, which we have not taken advantage of. Including this extra ingredient into the construction could potentially lead to a sharper lower bound, and we leave this as future work.
\end{remark}

A useful side result of our proof is the following lower bound for non-stationary RL in the un-discounted setting, which is the same setting as studied in~\citet{gajane2018sliding},~\citet{ortner2020variational} and \citet{cheung2020reinforcement}. 

\begin{proposition}\label{corollary}
	Consider a reinforcement learning problem in un-discounted non-stationary MDPs with horizon length $T$, total variation budget $\Delta$, and maximum MDP diameter $D$~\citep{cheung2020reinforcement}. For any learning algorithm, there exists a non-stationary MDP such that the dynamic regret of the algorithm is at least $\Omega( S^{\frac{1}{3}}A^{\frac{1}{3}} \Delta^{\frac{1}{3}} D^{\frac{2}{3}} T^{\frac{2}{3}} )$.
\end{proposition}
\section{Simulations}\label{sec:simulations}
In this section, we empirically evaluate RestartQ-UCB on reinforcement learning tasks with various types of non-stationarity. We compare RestartQ-UCB with three baseline algorithms: LSVI-UCB-Restart~\citep{zhou2020nonstationary}, Q-Learning UCB, and Epsilon-Greedy~\citep{watkins1989learning}. LSVI-UCB-Restart is a state-of-the-art non-stationary RL algorithm that combines optimistic least-squares value iteration with periodic restarts. 
Q-Learning UCB is simply our RestartQ-UCB algorithm with no restart. It is a Q-learning based algorithm that uses upper confidence bounds to guide the exploration. Epsilon-Greedy is a restart-based algorithm that uses an epsilon-greedy strategy for action selection.

\begin{figure*}[t]
	\vspace{-0.5cm}\centering
	\subfigure[Abrupt variations]{\label{subfig:1}\includegraphics[width=0.35\textwidth]{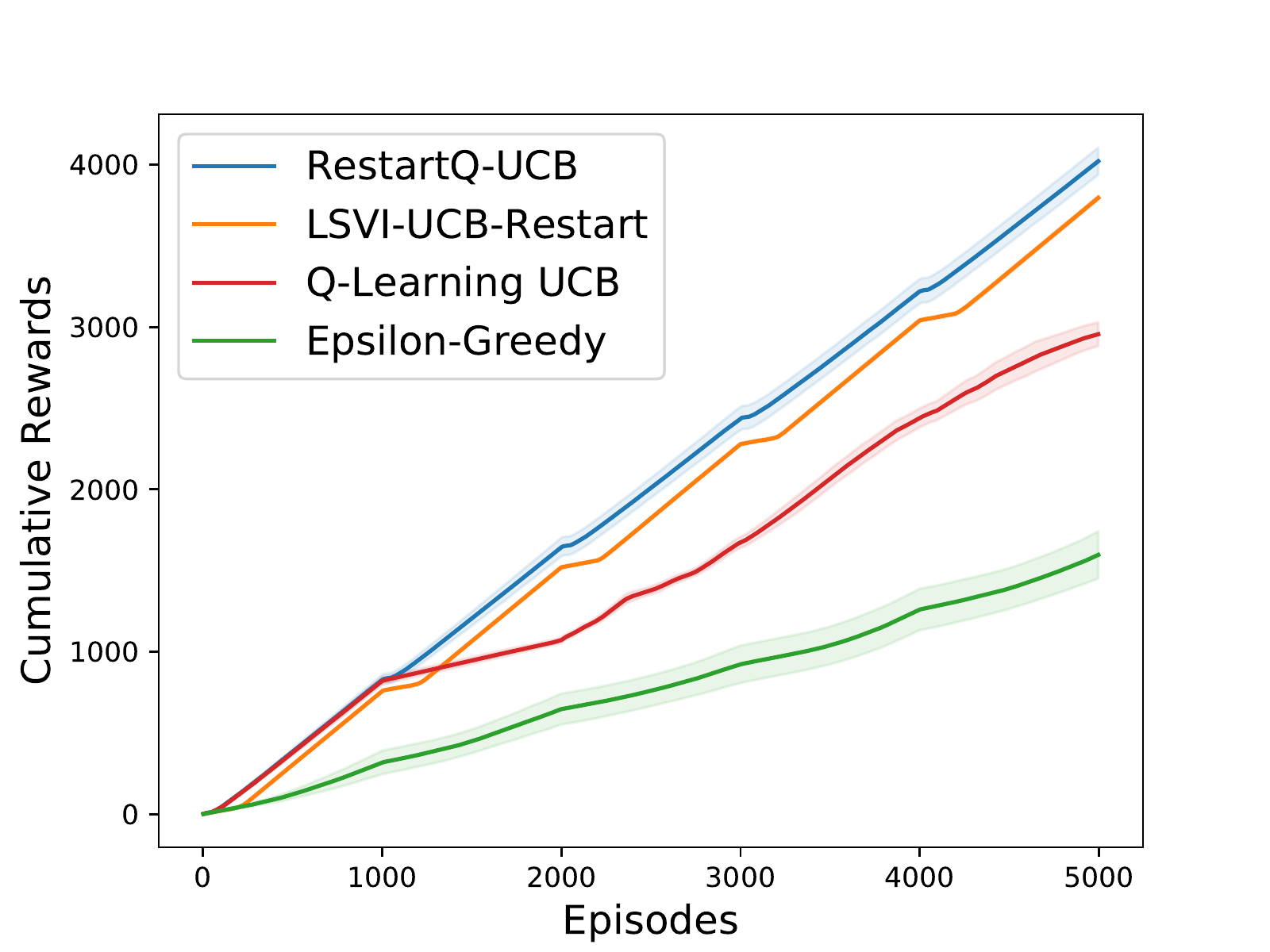}}
	\hspace{-.4cm}\subfigure[Gradual variations]{\label{subfig:2}\includegraphics[width=0.35\textwidth]{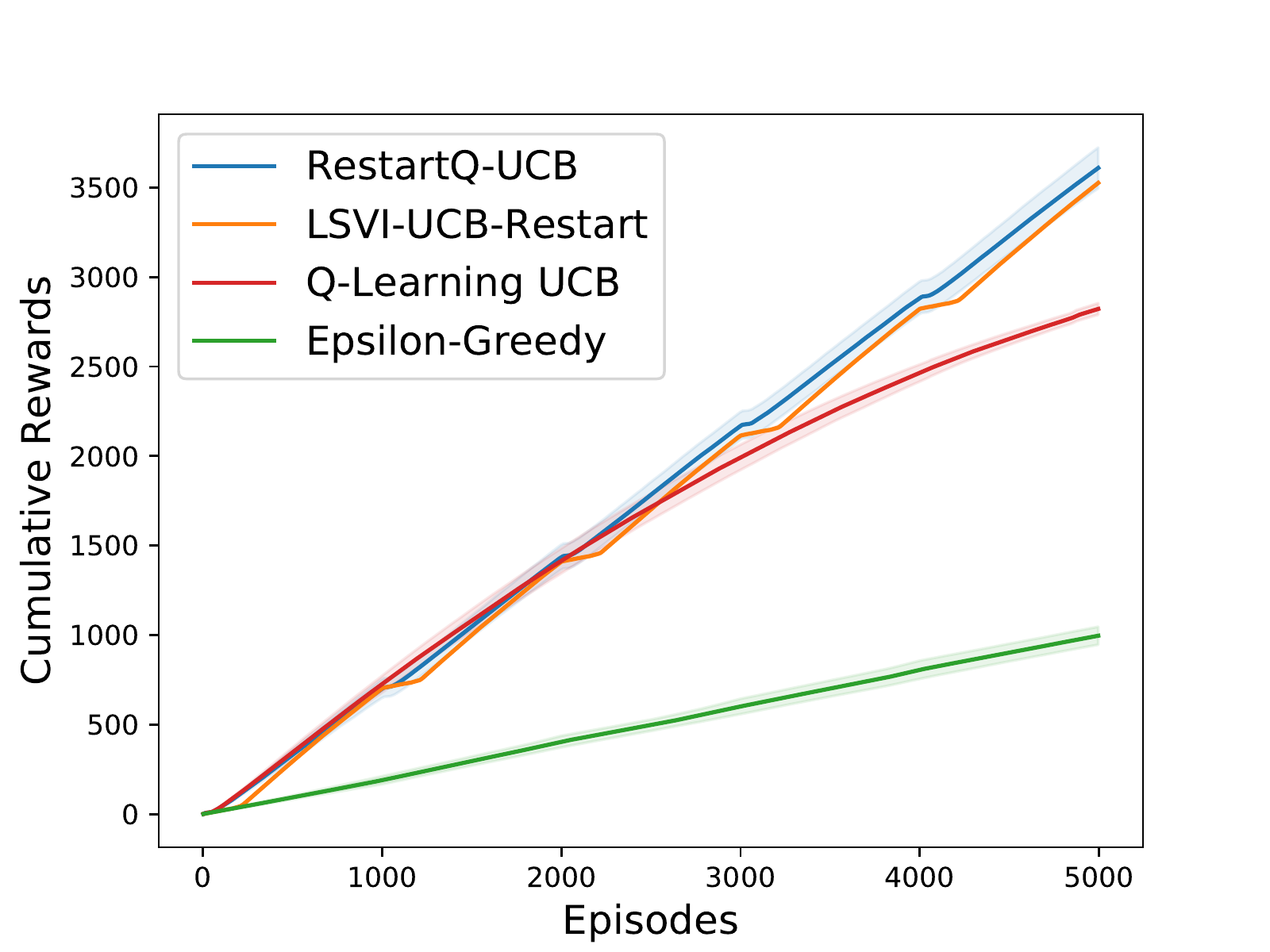}}
	\hspace{-.4cm}\subfigure[Time usage]{\label{subfig:3}\raisebox{10mm}{\includegraphics[width=0.32\textwidth]{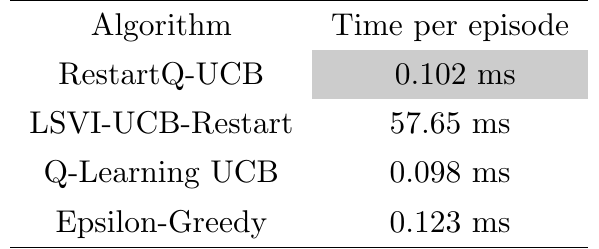}}}
	\caption{Cumulative rewards of the four algorithms under (a) abrupt variations, and (b) gradual variations, respectively, as well as their (c) time usage. Shaded areas denote the standard deviations of rewards. Note that RestartQ-UCB significantly outperforms Q-Learning UCB and Epsilon-Greedy, and matches LSVI-UCB-Restart while being \emph{much more} time-efficient. }
	\label{fig:simulations}
\end{figure*}

We evaluate the cumulative rewards of the four algorithms on a variant of a reinforcement learning task named Bidirectional Diabolical Combination Lock~\citep{agarwal2020pc,misra2020kinematic}. This task is designed to be particularly difficult for \emph{exploration}. We introduce two types of non-stationarity to the task, namely \emph{abrupt} variations and \emph{gradual} variations. A detailed discussion on the task settings as well as the configuration of the hyper-parameters is deferred to Appendix~\ref{app:simulations}. The cumulative rewards of the four algorithms in the abruptly-changing and gradually-changing environments are shown in Figures~\ref{subfig:1} and~\ref{subfig:2}, respectively. All results are averaged over $30$ runs. 

As we can see, RestartQ-UCB outperforms Q-Learning UCB and Epsilon-Greedy under both types of environment variations. For the abruptly-changing environment as an example, RestartQ-UCB achieves $1.36$ and $2.52$ times of the cumulative rewards of Q-Learning UCB and Epsilon-Greedy, respectively. This demonstrates the importance of both addressing the environment variations (using restarts) and actively exploring the environment (using UCB-based bonus terms) in non-stationary RL. LSVI-UCB-Restart nearly matches the performance of RestartQ-UCB, which is unsurprising because both of them use the restarting strategy and optimistic exploration. Nevertheless, LSVI-UCB-Restart requires a higher time and space complexity. It needs to store all the history information in one epoch and solve a regularized least-squares minimization problem at every time step. This is indeed evidenced by our simulation results (shown in Figure~\ref{subfig:3}) that RestartQ-UCB only takes $0.18\%$ of the computation time of LSVI-UCB-Restart. 

\begin{remark}
	The heavy computation in LSVI-UCB-Restart mostly comes from the usage of a high-dimensional feature. In our simulations, we followed Example 2.1 in \citet{jin2019learning} to convert a linear MDP algorithm to a tabular one, which results in a feature dimension of $d= S\times A$. This is essentially the most efficient feature encoding when no special structure is imposed on the tabular MDP. We believe that designing low-dimensional features for specific MDP instances can possibly reduce the computations for LSVI-UCB-Restart by a large amount, and is an interesting future direction for learning in linear MDPs per se. 
\end{remark}

\section{Application to Multi-Agent RL}\label{sec:application}
\wmedit{In this section, we discuss an application of our non-stationary RL method to multi-agent RL in episodic stochastic teams, which by nature leads to a non-stationary RL problem from each agent's perspective. Such a decentralized multi-agent scenario also helps us demonstrate the significance and flexibility of the model-free approach. }

\subsection{Problem Setup}
In general, an $N$-player episodic stochastic \wmedit{team is defined by a tuple $(\mc{N}, H, \mc{S},\{\mc{A}^{(i)}\}_{i=1}^N, r, P)$}, where (1) $\mc{N} = \{1,2,\dots,N\}$ is the set of agents; (2) $H\in\mb{N}_+$ is the number of time steps in each episode; (3) $\mc{S}$ is the finite state space; (4) $\mc{A}^{(i)}$ is the finite action space for agent $i\in\mc{N}$; (5) $r_{h}:\mc{S}\times \mc{A} \ra [0,1]$ is the \wmedit{team reward function at step $h\in[H]$ common to all the agents $i\in\mc{N}$}, where $\mc{A} = \times_{i=1}^N \mc{A}^{(i)}$; and (6) $P_{h}: \mc{S}\times \mc{A} \ra \Delta(\mc{S})$ is the transition kernel at step $h\in[H]$, where the next state depends on the current state and the joint actions of all the agents. The game lasts for $M$ episodes, and we let $T = MH$ be the total number of time steps. At each time step $(m,h)$, the agents observe the state $s_h^m \in \mc{S}$, and take actions $a^{{(i)},m}_h \in\mc{A}^{(i)}, i\in\mc{N}$ simultaneously.\footnote{Note that we use superscripts in parentheses to index the agents, while a superscript with no parenthesis denotes the index of an episode.} We let $a_h^m = (a_h^{(1),m},\dots, a_h^{(N),m}) $. \wmedit{All the agents receive a common reward with an expected value of $r_h(s_h^m,a_h^m)$}, and the environment transitions to the next state $s_{h+1}^m\sim P_h(\cdot | s_h^m,a_h^m)$. For each agent $i$, a policy is a mapping from the time index and state space to (possibly a distribution over) the action space. We denote the set of policies for agent $i$ by $\Pi^{(i)}=\{ \pi^{(i)} : [M]\times[H]\times \mc{S}\ra \Delta(\mc{A}^{(i)})\}$. The set of joint policies are denoted by $\Pi = \times_{i=1}^N \Pi^{(i)}$. \wmedit{Each agent seeks to find a policy such that the collective choice maximizes the cumulative team reward. }

For notational convenience, and without much loss of conceptual generality, we consider two-player teams, i.e., $N = 2$. We consider the problem where we can control the policy of agent 1, while agent 2 is an opponent that is adapting its own policy in an unknown way. Achieving sublinear regret in the face of an arbitrarily changing opponent is known
to be computationally hard~\citep{radanovic2019learning}. Therefore, existing works~\citep{radanovic2019learning,lee2020linear} often focus on a setting where the opponent is only ``slowly changing'' its policy over time. One such example is when the opponent is using a relatively stable learning algorithm. \wmedit{We also focus on the \emph{decentralized} setting, a more practical multi-agent RL paradigm, where an agent \emph{cannot} observe the actions and rewards of the other agent. }

A joint policy induces a probability measure on the sequence of states and joint actions. For a joint policy $\pi = (\pi^{(1)},\pi^{(2)}) \in \Pi$, and for each time step $(m,h)\in[M]\times [H]$, state $s \in \mc{S}$, we define the \wmedit{state value function for the agents as follows}:
\[
\begin{aligned}
V_{h}^{m,\pi}(s)&\defeq\mathbb{E}\left[\sum_{h^{\prime}=h}^{H} r\left(s_{h^{\prime}}, \pi_{h^{\prime}}^{(1),m}\left(s_{h^{\prime}}\right),  \pi_{h^{\prime}}^{(2),m}\left(s_{h^{\prime}}\right) \right) \mid s_{h}=s\right].
\end{aligned}
\]
For a joint policy $(\pi^{(1)},\pi^{(2)})$, we again evaluate the optimality of agent 1's policy $\pi^{(1)}$ in terms of its \emph{dynamic regret}, which compares the agent's policy with the optimal policy of each individual episode in hindsight: 
\[
\mc{R}^{\pi^{(2)}}(\pi^{(1)}, M) \defeq \sum_{m=1}^{M}\left( \sup_{\pi^{{(1)}\star}} V_1^{m,(\pi^{{(1)}\star},\pi^{(2)})} (s_{1}^{m})-V_{1}^{m,(\pi^{(1)},\pi^{(2)})}(s_{1}^{m})\right).
\]
The initial state of each episode $s_1^m$ is again chosen by an oblivious adversary. 



We model the slowly-changing behavior of agent 2 by requiring it to have a \wmedit{low \emph{switching cost}~\citep{bai2019provably,gao2021provably} defined as follows.}

\begin{definition}
	The switching cost between any pair of policies $(\pi,\pi')$ is the number of $(h, s)$ pairs on which $\pi$ and $\pi'$ act differently:
	\[
	n_{\text{switch}}(\pi,\pi') \defeq \abs{\left\{ (h,s)\in[H]\times \mc{S}: \pi_h(s)\neq \pi_h'(s) \right\}}. 
	\]
	For a policy trajectory $(\pi^1,\dots,\pi^M)$ across $M$ episodes, its switching cost is defined as
	$
	N_{\text{switch}} \defeq \sum_{m=1}^M n_{\text{switch}}(\pi^m, \pi^{m+1}). 
	$
\end{definition}

\wmedit{We characterize the behavior of agent 2 by assuming that the switching cost of its policy trajectory is upper bounded by $O(T^\beta)$ for some $0 < \beta < 1$. Many existing RL algorithms~\citep{bai2019provably,zhang2020almost} satisfy this upper bound. }

\subsection{Learning Team-Optimality}

Our non-stationary RL algorithm can be readily applied to learning team-optimal policies in ``smooth games'', which is the setting considered in~\citet{radanovic2019learning}. This corresponds to the setting where a team of agents learn to collaborate. \wmedit{We define team-optimality as the joint policy of the agents that induces the highest possible accumulated reward. }

\begin{definition}
	In a two-player team, a joint policy $\pi^\star = (\pi^{{(1)}\star},\pi^{{(2)}\star}) \in \Pi$ is called \emph{team-optimal} if 
	\[
	V_h^{(\pi^{(1)\star},\pi^{(2)\star})}(s) = \sup_{\pi^{(1)},\pi^{(2)}} V_h^{(\pi^{(1)},\pi^{(2)})}(s), \forall s \in \mc{S},h\in[H],
	\] 
	where $V_h^{(\pi^{(1)},\pi^{(2)})}(s) \defeq \mathbb{E}[\sum_{h'=h}^{H} r_{h'}(s_{h'}, \pi_{h'}^{(1)}\left(s_{h'}\right),  \pi_{h'}^{(2)}(s_{h'}) ) \mid s_{h}=s]$ is the value function. 
\end{definition}

Since we cannot control the behavior of agent 2, its behavior might be sub-optimal and drive us away from team-optimality. To avoid such scenarios, we impose a structural assumption that allows us to quantify the distance from optimality. In particular, we assume that the team is $(\lambda,\mu)$-smooth, following the definition in \citet{radanovic2019learning}. 

\begin{definition}
 	(Adapted from Definition 1 in~\citet{radanovic2019learning}) A two-player stochastic team is $(\lambda,\mu)$-smooth if there exists a pair of policies $(\pi^{(1)\star},\pi^{(2)\star})$ such that for every policy pair $(\pi^{(1)},\pi^{(2)})$ and every $h\in[H], s\in\mc{S}$:
 	\[
 	\begin{aligned}
 	V^{(\pi^{(1)\star},\pi^{(2)\star})}_h(s)&\geq V^{(\pi^{(1)},\pi^{(2)})}_h(s),\\
 	V^{(\pi^{(1)\star},\pi^{(2)})}_h(s)&\geq \lambda \cdot V^{(\pi^{(1)\star},\pi^{(2)\star})}_h(s) - \mu\cdot V^{(\pi^{(1)},\pi^{(2)})}(s).
 	\end{aligned}
 	\]
\end{definition}
The $(\lambda,\mu)$-smoothness ensures that agent 2's sub-optimal behavior only has a bounded negative impact on the joint value. This notion of smoothness is motivated by the definition of smooth games in~\citet{roughgarden2009intrinsic} and \citet{syrgkanis2015fast}, as stated in \citet{radanovic2019learning}. \wmedit{Applying our RestartQ-UCB algorithm for agent 1 would lead to the following theorem, which implies that the time-average return of the agents converges to a $\frac{\lambda}{1+\mu}$ factor of the team-optimal value as $T$ grows. }

\begin{theorem}\label{thm:team}
	Let $\pi^{(2)}$ denote the policy of agent 2, and suppose that the switching cost of agent 2 satisfies $N_{\text{switch}} = O(T^{\beta})$ for $0 < \beta < 1$. Assume that the team problem is $(\lambda,\mu)$-smooth. Let agent 1 run the RestartQ-UCB algorithm, and let $\pi^{(1)}$ denote its induced policy. For $T$ large enough, the return of the algorithm is lower bounded by:
	\[
	\sum_{m=1}^M V_1^{(\pi^{(1)},\pi^{(2)})}(s_1^m) \geq \frac{\lambda}{1+\mu} \left[ \sum_{m=1}^M V_1^{(\pi^{(1)\star},\pi^{(2)\star})}(s_1^m) - \widetilde{O}(T^{\frac{\beta+2}{3}})\right]. 
	\] 
\end{theorem}

\begin{remark}
	(Significance of model-freeness.) Decentralized multi-agent RL is generally only possible with model-free approaches (see, e.g.,~\citet{arslan2016decentralized,tian2020provably,daskalakis2020independent}); model-based methods proceed by explicitly estimating the transition and reward functions, which crucially relies on observing the other agents' actions. This further demonstrates the flexibility and significance of model-free methods, when one addresses the non-stationarity issues in  multi-agent RL through the lens of non-stationary RL. 
\end{remark}

\section{Application to Inventory Control Across Related Products}\label{sec:application_inventory}

In this section, we discuss the application of our non-stationary RL algorithm to the problem of inventory control across related products. Different from conventional inventory control problems (e.g., \cite{HuhR09}) that only consider one product, we investigate the case where a sequence of related products are being sold, 
and the products share similar but different demand distributions. This is motivated by the sequential launch of  related products (e.g., the line of iPhone) that allows us to leverage experience from past products to inform inventory management for future ones. Following \cite{yuan2021marrying,CheungSLZ20} (who only consider a single product being sold), we focus on the setting of zero lead time, fixed cost, and lost sales.

\subsection{Problem Setup}\label{subsec:inventory}
The inventory control problem has $M$ episodes, representing $M$ different but related products. Each episode/product lasts for $H$ time steps.\footnote{We assume for simplicity that the life cycle of each product is of the same length. } For each time step $h\in[H]$ of an episode $m\in[M]$, the following sequence of events happens in order:
\begin{enumerate}
	\item The seller observes her stock level $s_h^m\geq 0$ for product $m$ at the beginning of time step $h$, and decides on the quantity $a_h^m\geq 0$ to order. 
	\item If $a_h^m>0$, the order arrives immediately, and the seller's stock level becomes $s_h^m+a_h^m$. The seller pays a fixed cost $f$ and a $c$ per-unit ordering cost. 
	\item The random demand $X_h^m$ is realized. The seller only observes the actual sales quantity, or \emph{censored demand} $Y_h^m = \min\{X_h^m,s_h^m+a_h^m\}$. She will not know the actual demand if $X_h^m \geq s_h^m+a_h^m$. \wmedit{Following prior works~\citet{cheung2020reinforcement,agarwal2020pc} and \citet{yuan2021marrying}, we assume that the demands $X_h^m$ are independent random variables over $m$, but they do not necessarily follow identical distributions since we consider different products across the episodes.  }
	\item All unfulfilled demands are permanently lost and incur a per-unit lost sales cost $p$. Excess inventory incurs a per-unit holding cost $q$. The total cost at step $h$ can be expressed as
	\[
	C_h^m(s_h^m,a_h^m) = f \cdot \mb{I}[a_h^m>0] + c\cdot a_h^m + p\cdot  \left[X_h^m -s_h^m-a_h^m\right]^+ + q\cdot \left[s_h^m+a_h^m-X_h^m\right]^+. 
	\]
	\item The inventory carried over to the next step $h+1$ is $s_{h+1}^m = [s_h^m + a_h^m - X_h^m]^+$. 
\end{enumerate}

Following \citet{CheungSLZ20,yuan2021marrying}, we assume that the seller has a finite storage capacity $S$, in the sense that she can hold at most $S-1$ units of inventory at any time. The seller's objective is to minimize her cumulative cost  $\sum_{m=1}^M\sum_{h=1}^H C_h^m(s_h^m,a_h^m)$. At the end of each episode, as a product is reaching the end of its life cycle, we assume for simplicity that the storage is emptied at no cost. Such an inventory control problem can be easily formulated as an instance of the non-stationary RL model that we defined in Section~\ref{sec:preliminary}. Concretely, we treat the stock level $s_h^m$ at the beginning of each time step as the state of the environment, and regard the order quantity $a_h^m$ as the action at the corresponding time step. Consequently, we define the state space of the problem as $\mc{S} = \{0, 1, \dots, S-1\}$, and the state-dependent action space as $\mc{A}_s = \{0,1,\dots, S-1-s\}$. One can verify that Algorithm~\ref{alg:RQUCB} and its analysis easily generalize to state-dependent action spaces. 

The reward function of the non-stationary MDP is defined as $R_h^m(s_h^m,a_h^m) = - C_h^m(s_h^m,a_h^m)$, and we let $r_h^m(s_h^m,a_h^m) = \mathbb{E}[R_h^m(s_h^m,a_h^m)]$ be the expected value of the reward. For any $s_h^m,s_{h+1}^m\in \mc{S}$ and $a_h^m\in \mc{A}_s$, we define the state transition function as 
$$
P_h^m(s_{h+1}^m \mid s_h^m,a_h^m)=\mbp\left(s_h^m+a_h^m-\min \{s_h^m+a_h^m,X_h^m\}=s_{h+1}^m\right).
$$ 
Our definitions of the policy $\pi$, the value function $V_h^{m,\pi}$, the state-action value function $Q_h^{m,\pi}$, as well as the optimal policy $\pi^\star$ and its corresponding value functions $V_h^{m,\star}, Q_h^{m,\star}$ directly carry over from Section~\ref{sec:preliminary} to this problem instance, and we do not repeat such definitions here for simplicity. The variation budget $\Delta$ is also defined in the same way as in Section~\ref{sec:preliminary}, which captures the differences in the products' demand distributions for this problem. The dynamic regret of the agent's policy is defined analogously as
\[
\mc{R}(\pi, M) = \sum_{m=1}^{M}\left(V_{1}^{m,\star}\left(s_{1}^{m}\right)-V_{1}^{m,\pi}\left(s_{1}^{m}\right)\right).
\]

\subsection{Implementation of RestartQ-UCB}
Notably, one major difference between the inventory control problem we considered in Section~\ref{subsec:inventory} and our non-stationary MDP formulation in Section~\ref{sec:preliminary} is that due to demand censoring, the seller cannot calculate the actual cost $C^m_h(s^m_h,a^m_h)$, and hence the immediate reward $R_h^m(s_h^m,a_h^m)$ is also not observable. Nevertheless, we will show that one can bypass such an issue by using a pseudo-reward technique, \wmedit{which was originally introduced for a stationary problem~\citep{agrawal2019learning}}. Specifically, for every time step $h\in[H]$ in episode $m\in[M]$, and for every state $s\in\mc{S}$ and action $a\in\mc{A}_s$, we define the pseudo-reward as
\[
R_h^{m,\text{pseudo}}(s,a) \defeq R_h^m(s,a) + p \cdot X_h^m = -f \cdot \mb{I}[a>0] - c\cdot a -q\cdot [s+a-Y_h^m]^+ +p\cdot Y_h^m,
\]
where we recall that the censored demand $Y_h^m = \min\{X_h^m,s+a\}$ is perfectly observable. Similarly, we can also define the mean pseudo-reward as
\[
r_h^{m,\text{pseudo}}(s,a) \defeq \ee[R_h^{m,\text{pseudo}}(s,a)]=\ee[R_h^m(s,a)+p\cdot X_h^m] = r_h^m(s,a) + p\cdot \ee[X_h^m]. 
\]
Therefore, the mean pseudo-reward can be considered as shifting the mean reward function uniformly by an amount of $p\cdot \ee[X_h^m]$. Without loss of generality, we normalize the pseudo-reward to the range $[0,1]$. We use the tuple $\mc{M} = \left\{\mc{S},\mc{A}, H, \{P_h^m\}_{m\in[M],h\in[H]}, \{r_h^{m}\}_{m\in[M],h\in[H]}\right\}$ to denote the non-stationary MDP with respect to the original reward function, and let $\mc{M}^{\text{pseudo}} = \left\{\mc{S},\mc{A}, H, \{P_h^m\}_{m\in[M],h\in[H]}, \{r_h^{m,\text{pseudo}}\}_{m\in[M],h\in[H]}\right\}$ be the one corresponding to the pseudo-reward. We further define $\pi^{\star, \text{pseudo}}$ to be the (episode-wise) optimal policy for $\mc{M}^{\text{pseudo}}$, and let $V_h^{m,\star,\text{pseudo}}$ and $Q_h^{m,\star,\text{pseudo}}$, respectively, be the corresponding value function and state-action value function. 

Since only the pseudo-reward is observable, we can only apply our RestartQ-UCB algorithm to $\mc{M}^{\text{pseudo}}$ rather than $\mc{M}$. A natural question then is whether we can generalize the performance guarantee from $\mc{M}^{\text{pseudo}}$ to $\mc{M}$. Interestingly, the following result (adapted from \citet{agrawal2019learning}) shows that, for any (possibly non-Markovian) policy $\pi$ induced by Algorithm~\ref{alg:RQUCB}, the dynamic regret on $\mc{M}^{\text{pseudo}}$ and $\mc{M}$ are equal. 
\begin{lemma}\label{lemma:inventory}
	(Adapted from Lemma 3.1 in \citet{agrawal2019learning}). Let $\mc{F}_h^m$ be the set of all historical information collected up to the beginning of time step $h$ of episode $m$. Let $\pi$ be the (possibly non-Markovian) policy induced by Algorithm~\ref{alg:RQUCB}, such that $\pi_h^m(s_h^m,\mc{F}_h^m)$ maps the state and history to a distribution over the action space. Then, $\pi$ incurs the same dynamic regret on $\mc{M}$ and $\mc{M}^{\text{pseudo}}$:
	\[
	\sum_{m=1}^{M}\left(V_{1}^{m,\star}\left(s_{1}^{m}\right)-V_{1}^{m,\pi}\left(s_{1}^{m}\right)\right) = \sum_{m=1}^{M}\left(V_{1}^{m,\star,\text{pseudo}}\left(s_{1}^{m}\right)-V_{1}^{m,\pi,\text{pseudo}}\left(s_{1}^{m}\right)\right). 
	\]
\end{lemma}

Together with Theorem~\ref{thm:regret}, we obtain the following dynamic regret bound for running Algorithm~\ref{alg:RQUCB} on the inventory control problem across related products. 

\begin{theorem}\label{thm:inventory}
	For $T= \Omega(SA\Delta H^2)$, and for any $\delta \in (0,1)$, with probability at least $1-\delta$, the dynamic regret of running Algorithm~\ref{alg:RQUCB} on the inventory control problem formulated in Section~\ref{subsec:inventory} with pseudo-rewards and Freedman bonuses is bounded by $\widetilde{O}( S^{\frac{1}{3}}A^{\frac{1}{3}} \Delta^{\frac{1}{3}} H T^{\frac{2}{3}} )$, where $\widetilde{O}(\cdot)$ hides poly-logarithmic factors of $S,A,T$ and $1/\delta$.
\end{theorem}
\begin{remark}[Comparison with \cite{cheung2020reinforcement}]
	Although both our work and \cite{cheung2020reinforcement} consider applications in inventory control and utilizes techniques from \citet{agrawal2019learning}, the foci are quite different. In \citet{cheung2020reinforcement}, the authors only study single product inventory, whereas in contrast, our work studies the setting where there is a sequence of related, but \emph{different} products. 
	
	Specifically, in \cite{cheung2020reinforcement}, variation budget has been defined with respect to demand changes within a single product selling horizon, and the corresponding regret upper bound scales with this budget, whereas in ours no constraint has been put to limit the demand changes within a single product's selling horizon (an episode), and the variation budget captures the difference across products. This is similar to a meta/transfer learning setting where the goal is to leverage data obtained from inventory learning for similar products to accelerate inventory learning for the new product.
	
	Moreover, as discussed in Section \ref{sec:intro}, a direct application of the results in \cite{cheung2020reinforcement} to this setting may lead to a worse regret upper bound.
\end{remark}
\wmedit{
	\begin{remark}Our results can be extended to a multi-product inventory control problem with a warehouse-capacity constraint, similar to the setting studied in~\citet{shi2016nonparametric}. Specifically, we have an episodic setting with $n$ products and $M$ episodes, where each episode lasts for $H$ time steps. For each time step $h\in[H]$ of an episode $m\in[M]$, a demand is specified for every product $i\in[n]$. In our non-stationary formulation, the demands need not follow identical distributions over time. An overall warehouse capacity constraint is also imposed on the total number of products simultaneously in the inventory. At each time step, the seller observes the stock level and decides on the quantity to order for each product at a certain per-unit ordering cost. Unfulfilled demands are permanently lost and incur a per-unit lost sales cost. Excess inventory also incurs a per-unit holding cost. The seller's objective is to minimize the cumulative cost. 
Such a multi-product problem can also be cast as an MDP, where we define the stock levels of all products to be the state of the environment, and define the joint ordering quantity across all products as the action of each step. We also let the action space be state-dependent to handle the joint capacity constraint; in particular, an action is considered invalid at a certain state if the corresponding ordering quantity causes the stock levels to exceed the warehouse capacity. Applying Algorithm~\ref{alg:RQUCB} to such a non-stationary multi-product problem leads to the same dynamic regret bound as in Theorem~\ref{thm:inventory}, though the state space should now be interpreted as the possible combinations of stock levels across all products that do not exceed the warehouse capacity, which is significantly larger than the single-product case. The same is true for the action space. A final remark is that the multi-product formulation above does not consider upgrading~\citep{yu2015dynamic}, the situation where a high-quality product is used to serve the demand of a lower-quality one that has been sold out. Upgrading adds an additional element of difficulty to the decision-making process, as the seller now needs to consider the ordering and upgrading decisions jointly. We leave the treatment of such a more intricate scenario to our future work. 
\end{remark}
}

\section{Concluding Remarks}\label{sec:conclusions}
In this paper, we have considered model-free reinforcement learning in non-stationary episodic MDPs. We have proposed an algorithm named RestartQ-UCB that adopts a simple restarting strategy. RestartQ-UCB with Freedman-type bonus terms achieves a dynamic regret of $\widetilde{O}(S^{\frac{1}{3}} A^{\frac{1}{3}} \Delta^{\frac{1}{3}} H T^{\frac{2}{3}})$, which nearly matches the information-theoretical lower bound $\Omega(S^{\frac{1}{3}} A^{\frac{1}{3}} \Delta^{\frac{1}{3}} H^{\frac{2}{3}} T^{\frac{2}{3}})$. We have further presented a parameter-free algorithm named Double-Restart Q-UCB that removes the assumption on knowing the variation budget. Numerical experiments have validated the advantages of RestartQ-UCB in terms of both cumulative rewards and computational efficiency. Examples in multi-agent RL and inventory control have been discussed as applications to illustrate the power of our method. An interesting future direction would be to close the $\widetilde{O}(H^{\frac{1}{3}})$ factor gap between the upper and lower bounds that we have established for the non-stationary RL problem. It would also be interesting to explore if non-stationary RL can be helpful in other multi-agent RL or inventory control scenarios. 


\bibliography{ref}
\bibliographystyle{ormsv080}


\newpage
\begin{APPENDIX}{Supplementary}

\section{Applications to Sequential Transfer and Multi-Task RL}\label{app:application}
Other areas that could benefit from non-stationary RL include sequential transfer in bandit \citep{BastaniSLZ21} and RL~\citep{tirinzoni2020sequential} and multi-task RL~\citep{brunskill2013sample}, which in turn are conceptually related to continual RL~\citep{kaplanis2018continual} and life-long RL~\citep{abel2018policy}. In the setting of sequential transfer/multi-task RL, the agent encounters a sequence of tasks over time with different system dynamics, and seeks to bootstrap learning by transferring  knowledge from previously-solved tasks. Typical solutions in this area~\citep{brunskill2013sample,tirinzoni2020sequential,sun2020temple} need to assume that there are \emph{finitely many} candidate tasks, and every task should be \emph{sufficiently different} from the others\footnote{Needless to say, this assumption itself also to some extent contradicts the primary motivation of transfer learning. After all, we only want to transfer knowledge among tasks that are essentially similar to each other. }. Only under this assumption can the agent quickly identify the current task it is operating on, by essentially comparing the system dynamics it observes with the dynamics it has memorized for each candidate task. After identifying the current task with high confidence, the agent then invokes the policy that it learned through previous interactions with this specific task. This transfer learning paradigm in turn causes another problem---it ``cold switches'' between policies that are most likely very different, which might lead to unstable and inconsistent behaviors of the agent over time. Fortunately, non-stationary RL can help alleviate both the finite-task assumption and the cold-switching problem. First, non-stationary RL algorithms do not need the candidate tasks to be sufficiently different in order to correctly identify each of them, because the algorithm itself can tolerate some variations in the task environment. There will also be no need to assume the finiteness of the candidate task set anymore, and the candidate tasks can be drawn from a continuous space. Second, since we are running the same non-stationary RL algorithm for a series of tasks, it improves its policy gradually over time, instead of cold-switching to a completely independent policy for each task. This could largely help with the unstable behavior issues.

\section{Proofs of the Technical Lemmas}\label{app:lemmas}
\subsection{Proof of Lemma~\ref{lemma:optimalQ}}
\begin{proof}
	For each $d\in[D]$, define $\Delta^{(d)}_r$ to be the \emph{local variation} of the mean reward function within epoch $d$. By definition, we have $\sum_{d=1}^D \Delta_r^{(d)} \leq \Delta_r$. Further, for each $d\in [D]$ and $h\in [H]$, define $\Delta_{r,h}^{(d)}$ to be the variation of the mean reward at step $h$ in epoch $d$, i.e., 
	$$\Delta_{r,h}^{(d)} \defeq \sum_{m = (d-1)K+1}^{\min\{dK,M\}-1} \sup_{s,a} \abs{r_h^m(s,a) - r_h^{m+1}(s,a)}.$$
	It also holds that $\sum_{h=1}^H \Delta_{r,h}^{(d)} = \Delta_r^{(d)}$ by definition. Define $\Delta_p^{(d)}$ and $\Delta_{p,h}^{(d)}$ analogously. 
	
	In the following, we will prove a stronger statement: $\abs{Q_h^{k_1,\star}(s,a) - Q_h^{k_2,\star}(s,a)}\leq \sum_{h'=h}^H \Delta_{r,h'}^{(1)} +H \sum_{h'=h}^H \Delta_{p,h'}^{(1)}$, which implies the statement of the lemma because $\sum_{h'=h}^H \Delta_{r,h'}^{(1)} \leq \Delta_{r}^{(1)}$ and $\sum_{h'=h}^H \Delta_{p,h'}^{(1)} \leq \Delta_{p}^{(1)}$ by definition.  Our proof relies on backward induction on $h$. First, the statement holds for $h=H$ because for any $(s,a)$, by definition
	\begin{align}
	\abs{Q_H^{k_1,\star}(s,a) - Q_H^{k_2,\star}(s,a) } 
	&= \abs{r_H^{k_1}(s,a) - r_H^{k_2}(s,a)} 
	\leq \sum_{k=k_1}^{k_2-1} \abs{r_H^{k+1}(s,a) - r_H^{k}(s,a)}\nonumber\\
	&\leq \sum_{k=1}^{K-1} \abs{r_H^{k+1}(s,a) - r_H^{k}(s,a)} 
	\leq \Delta_{r,H}^{(1)},\label{eqn:1}
	\end{align}
	where we have used the triangle inequality. Now suppose the statement holds for $h+1$; by the Bellman optimality equation,
	\begin{align}
	&Q_h^{k_1,\star}(s,a) - Q_h^{k_2,\star}(s,a) \nonumber\\
	=&  P^{k_1}_hV_{h+1}^{k_1,\star}(s,a) - P^{k_2}_hV_{h+1}^{k_2,\star}(s,a) + r_h^{k_1}(s,a) - r_h^{k_2}(s,a)\nonumber\\
	\leq & P^{k_1}_hV_{h+1}^{k_1,\star}(s,a) - P^{k_2}_hV_{h+1}^{k_2,\star}(s,a) + \Delta_{r,h}^{(1)}\label{eqn:2}\\
	=& \sum_{s'\in \mc{S}} P_h^{k_1}(s'\mid s,a)V_{h+1}^{k_1,\star}(s') - \sum_{s'\in \mc{S}} P_h^{k_2}(s'\mid s,a)V_{h+1}^{k_2,\star}(s') + \Delta_{r,h}^{(1)}\nonumber\\
	=&  \sum_{s'\in \mc{S}} \left(P_h^{k_1}(s'\mid s,a)Q_{h+1}^{k_1,\star}(s', \pi_{h+1}^{k_1,\star}(s')) 
	-  P_h^{k_2}(s'\mid s,a)Q_{h+1}^{k_2,\star}(s', \pi_{h+1}^{k_2,\star}(s'))\right) 
	+ \Delta_{r,h}^{(1)},\label{eqn:3}
	\end{align}
	where inequality~\eqref{eqn:2} holds due to a similar reasoning as in~\eqref{eqn:1}, and in~\eqref{eqn:3} $\pi^{k_1,\star}$ and $\pi^{k_2,\star}$ denote the optimal policy in episodes $k_1$ and $k_2$, respectively. Then by our induction hypothesis on $h+1$, for any $s'\in\mc{S}$,
	\begin{align}
	Q_{h+1}^{k_1,\star}(s',\pi_{h+1}^{k_1,\star}(s')) 
	\leq & Q_{h+1}^{k_2,\star}(s', \pi_{h+1}^{k_1,\star}(s')) +\sum_{h'=h+1}^H \Delta_{r,h'}^{(1)} + H\sum_{h'=h+1}^H \Delta_{p,h'}^{(1)}\nonumber\\
	\leq & Q_{h+1}^{k_2,\star}(s', \pi_{h+1}^{k_2,\star}(s')) +\sum_{h'=h+1}^H \Delta_{r,h'}^{(1)} + H\sum_{h'=h+1}^H \Delta_{p,h'}^{(1)},\label{eqn:4}
	\end{align}
	where inequality~\eqref{eqn:4} is due to the optimality of the policy $\pi^{k_2,\star}$ in episode $k_2$ over $\pi^{k_1,\star}$. Then,
	\begin{align}
	&Q_h^{k_1,\star}(s,a) - Q_h^{k_2,\star}(s,a) \nonumber\\
	\leq& \sum_{s'\in \mc{S}} (P_h^{k_1}(s'\mid s,a)
	-  P_h^{k_2}(s'\mid s,a))  Q_{h+1}^{k_2,\star}(s', \pi_{h+1}^{k_2,\star}(s'))
	+\sum_{h'=h}^H \Delta_{r,h'}^{(1)} +H\sum_{h'=h+1}^H \Delta_{p,h'}^{(1)}\nonumber\\
	\leq & \norm{P_h^{k_1}(\cdot | s,a) -  P_h^{k_2}(\cdot | s,a)}_1  \norm{ Q_{h+1}^{k_2,\star}(\cdot, \pi_{h+1}^{k_2,\star}(\cdot )) }_\infty  +\sum_{h'=h}^H \Delta_{r,h'}^{(1)} + H\sum_{h'=h+1}^H \Delta_{p,h'}^{(1)}\label{eqn:5}\\
	\leq & \Delta_{p,h}^{(1)} (H-h) +\sum_{h'=h}^H \Delta_{r,h'}^{(1)} +H\sum_{h'=h+1}^H \Delta_{p,h'}^{(1)}\label{eqn:6}\\
	\leq & \sum_{h'=h}^H \Delta_{r,h'}^{(1)} +H\sum_{h'=h}^H \Delta_{p,h'}^{(1)}\nonumber,
	\end{align}
	where~\eqref{eqn:5} is by H{\"o}lder's inequality, and~\eqref{eqn:6} is by the definition of $\Delta_{p,h}^{(1)}$ and by the definition of optimal $Q$-values that $Q_{h+1}^{k_2,\star}(s,a) \leq H-h,\forall (s,a)\in \mc{S}\times \mc{A}$. Repeating a similar process gives us $Q_h^{k_2,\star}(s,a) - Q_h^{k_1,\star}(s,a)\leq\sum_{h'=h}^H \Delta_{r,h'}^{(1)} +H\sum_{h'=h}^H \Delta_{p,h'}^{(1)}$. This completes our proof. 
\end{proof}

\subsection{Proof of Lemma~\ref{lemma:optimistic_bound}}
\begin{proof}
	It should be clear from the way we update $Q_h(s,a)$ that $Q_h^k(s,a)$ is monotonically decreasing in $k$. We now prove $Q_h^{k,\star}(s,a) \leq Q_h^{k+1}(s,a)$ for all $s,a,h,k$ by induction on $k$. First, it holds for $k=1$ by our initialization of $Q_h(s,a)$. For $k\geq 2$, now suppose $Q_h^{j,\star}(s,a)\leq Q_h^{j+1}(s,a)\leq  Q_h^{j}(s,a)$ for all $s,a,h$ and $1 \leq j \leq k$. For a fixed triple $(s,a,h)$, we consider the following two cases.  
	
	\textbf{Case 1:} $Q_h(s,a)$ is updated in episode $k$. Then with probability at least $1-2\delta$,
	\begin{align}
	Q_h^{k+1}(s,a) = & \frac{\check{r}_h(s,a)}{\check{N}_h^k(s,a)} + \frac{\check{v}_h(s,a)}{\check{N}_h^k(s,a)} + b_h^k + 2b_\Delta\nonumber\\
	\geq & \frac{\check{r}_h(s,a)}{\check{n}} + \frac{1}{\check{n}}\sum_{i=1}^{\check{n}} V_{h+1}^{\check{l}_i, \star}(s_{h+1}^{\check{l}_i}) + \sqrt{\frac{H^2}{\check{n}}\iota}+\sqrt{\frac{\iota}{\check{n}}}+ 2b_\Delta\label{eqn:b1}\\ 
	\geq &  \frac{\check{r}_h(s,a)}{\check{n}} + \frac{1}{\check{n}}\sum_{i=1}^{\check{n}}P_h^{\check{l}_i}V_{h+1}^{\check{l}_i,\star}(s,a) +\sqrt{\frac{\iota}{\check{n}}}+ 2b_\Delta\label{eqn:b2}\\
	= &   \frac{\check{r}_h(s,a)}{\check{n}} + \frac{1}{\check{n}}\sum_{i=1}^{\check{n}}\left( Q_h^{\check{l}_i,\star}(s,a) - r_h^{\check{l}_i}(s,a) \right)+\sqrt{\frac{\iota}{\check{n}}} + 2b_\Delta\label{eqn:b3}\\
	\geq & Q_h^{k,\star}(s,a)+b_\Delta.\label{eqn:b4}
	\end{align} 
	Inequality~\eqref{eqn:b1} is by the induction hypothesis that $Q_{h+1}^{\check{l}_i}(s_{h+1}^{\check{l}_i},a)\geq Q_{h+1}^{\check{l}_i,\star}(s_{h+1}^{\check{l}_i},a),\forall a\in\mc{A}$, and hence $V_{h+1}^{\check{l}_i}(s_{h+1}^{\check{l}_i})\geq V_{h+1}^{\check{l}_i,\star}(s_{h+1}^{\check{l}_i})$. Inequality~\eqref{eqn:b2} follows from the Azuma-Hoeffding inequality. \eqref{eqn:b3} uses the Bellman optimality equation. Inequality~\eqref{eqn:b4} is by the Hoeffding's inequality that $\frac{1}{\check{n}}\left(\sum_{i=1}^{\check{n}}r_h^{\check{l}_i}(s,a) - \check{r}_h(s,a) \right)  \leq \sqrt{\frac{\iota}{\check{n}}}$ with high probability, and by Lemma~\ref{lemma:optimalQ} that $Q_h^{\check{l}_i,\star}(s,a) + b_\Delta \geq Q_h^{k,\star}(s,a)$. According to the monotonicity of $Q_h^{k}(s,a)$, we know that $Q_{h}^{k,\star}(s, a) \leq Q_{h}^{k+1}(s, a) \leq Q_{h}^{k}(s, a)$. In fact, we have proved the stronger statement $Q_h^{k+1}(s,a)\geq Q_h^{k,\star}(s,a)+b_\Delta$ that will be useful in Case 2 below. 
	
	\textbf{Case 2:} $Q_h(s,a)$ is not updated in episode $k$. Then there are two possibilities:
	
	\begin{enumerate} 
		\item If $Q_h(s,a)$ has never been updated from episode $1$ to episode $k$: It is easy to see that $Q_h^{k+1}(s,a) = Q_h^{k}(s,a) = \dots = Q_h^{1}(s,a)=H-h+1 \geq Q_h^{k,\star}(s,a)$ holds. 
		\item If $Q_h(s,a)$ has been updated at least once from episode $1$ to episode $k$: Let $j$ be the index of the latest episode that $Q_h(s,a)$ was updated. Then, from our induction hypothesis and Case~1, we know that $Q_h^{j+1}(s,a)\geq Q_h^{j,\star}(s,a)+b_\Delta$.  
		Since $Q_h(s,a)$ has not been updated from episode $j+1$ to episode $k$, we know that $Q_h^{k+1}(s,a) = Q_h^{k}(s,a)=\dots = Q_h^{j+1}(s,a) \geq Q_h^{j,\star}(s,a) + b_\Delta\geq Q_h^{k,\star}(s,a)$, where the last inequality holds because of Lemma~\ref{lemma:optimalQ}. 
	\end{enumerate}
	
	A union bound over all time steps completes our proof. 
\end{proof}

\subsection{Proof of Lemma~\ref{lemma:new_optimistic_bound}}
\begin{proof}
	This proof follows a similar structure as the proof of Lemma~\ref{lemma:optimistic_bound}. It should be clear from the way we update $Q_h(s,a)$ that $Q_h^k(s,a)$ is monotonically decreasing in $k$. We now prove $Q_h^{k,\star}(s,a)-2(H-h+1)b_\Delta \leq Q_h^{k+1}(s,a)$ for all $s,a,h,k$ by induction on $k$. First, it holds for $k=1$ by our initialization of $Q_h(s,a)$. For $k\geq 2$, now suppose that $Q_h^{j,\star}(s,a)-2(H-h+1)b_\Delta\leq Q_h^{j+1}(s,a)\leq Q_h^{j}(s,a)$ for all $s,a,h$ and $1 \leq j \leq k$. For a fixed triple $(s,a,h)$, we consider the following two cases. 
	
	\textbf{Case 1:} $Q_h(s,a)$ is updated in episode $k$. Then, with probability at least $1-2\delta$,
	\begin{align}
	Q_h^{k+1}(s,a) = & \frac{\check{r}_h(s,a)}{\check{N}_h^k(s,a)} + \frac{\check{v}_h(s,a)}{\check{N}_h^k(s,a)} + b_h^k\nonumber\\
	\geq & \frac{\check{r}_h(s,a)}{\check{n}} + \frac{1}{\check{n}}\sum_{i=1}^{\check{n}} V_{h+1}^{\check{l}_i, \star}(s_{h+1}^{\check{l}_i})-2(H-h)b_\Delta + \sqrt{\frac{H^2}{\check{n}}\iota}+\sqrt{\frac{\iota}{\check{n}}}\label{eqn:bb1}\\ 
	\geq &  \frac{\check{r}_h(s,a)}{\check{n}} + \frac{1}{\check{n}}\sum_{i=1}^{\check{n}}P_h^{\check{l}_i}V_{h+1}^{\check{l}_i,\star}(s,a) +\sqrt{\frac{\iota}{\check{n}}}-2(H-h)b_\Delta\label{eqn:bb2}\\
	= &   \frac{\check{r}_h(s,a)}{\check{n}} + \frac{1}{\check{n}}\sum_{i=1}^{\check{n}}\left( Q_h^{\check{l}_i,\star}(s,a) - r_h^{\check{l}_i}(s,a) \right)+\sqrt{\frac{\iota}{\check{n}}}-2(H-h)b_\Delta\label{eqn:bb3}\\
	\geq & Q_h^{k,\star}(s,a)-b_\Delta-2(H-h)b_\Delta.\label{eqn:bb4}
	\end{align} 
	Inequality~\eqref{eqn:bb1} is by the induction hypothesis that $Q_{h+1}^{\check{l}_i}(s_{h+1}^{\check{l}_i},a)\geq Q_{h+1}^{\check{l}_i,\star}(s_{h+1}^{\check{l}_i},a)-2(H-h)b_\Delta,\forall a\in\mc{A}$, and hence $V_{h+1}^{\check{l}_i}(s_{h+1}^{\check{l}_i})\geq V_{h+1}^{\check{l}_i,\star}(s_{h+1}^{\check{l}_i})-2(H-h)b_\Delta$. Inequality~\eqref{eqn:bb2} follows from the Azuma-Hoeffding inequality. \eqref{eqn:bb3} uses the Bellman optimality equation. Inequality~\eqref{eqn:bb4} is by the Hoeffding's inequality that $\frac{1}{\check{n}}\left( \sum_{i=1}^{\check{n}}r_h^{\check{l}_i}(s,a) - \check{r}_h(s,a) \right) \leq \sqrt{\frac{\iota}{\check{n}}}$ with high probability, and by Lemma~\ref{lemma:optimalQ} that $Q_h^{\check{l}_i,\star}(s,a) \geq Q_h^{k,\star}(s,a)- b_\Delta$. According to the monotonicity of $Q_h^{k}(s,a)$, we know that $Q_{h}^{k,\star}(s, a) -2(H-h+1)b_\Delta \leq Q_{h}^{k+1}(s, a) \leq Q_{h}^{k}(s, a)$. In fact, we have proved the stronger statement $Q_h^{k+1}(s,a)\geq Q_h^{k,\star}(s,a)-b_\Delta-2(H-h)b_\Delta$ that will be useful in Case 2 below. 
	
	\textbf{Case 2:} $Q_h(s,a)$ is not updated in episode $k$. Then there are two possibilities:
	
	\begin{enumerate} 
		\item If $Q_h(s,a)$ has never been updated from episode $1$ to episode $k$: It is easy to see that $Q_h^{k+1}(s,a) = Q_h^{k}(s,a) = \dots = Q_h^{1}(s,a)=H-h+1 \geq Q_h^{k,\star}(s,a)-2(H-h+1)b_\Delta$ holds. 
		\item If $Q_h(s,a)$ has been updated at least once from episode $1$ to episode $k$: Let $j$ be the index of the latest episode that $Q_h(s,a)$ was updated. Then, from our induction hypothesis and Case~1, we know that $Q_h^{j+1}(s,a)\geq Q_h^{j,\star}(s,a)-b_\Delta-2(H-h)b_\Delta$.  
		Since $Q_h(s,a)$ has not been updated from episode $j+1$ to episode $k$, we know that $Q_h^{k+1}(s,a) = Q_h^{k}(s,a)=\dots = Q_h^{j+1}(s,a) \geq Q_h^{j,\star}(s,a) -b_\Delta-2(H-h)b_\Delta\geq Q_h^{k,\star}(s,a)-2(H-h+1)b_\Delta$, where the last inequality holds because of Lemma~\ref{lemma:optimalQ}. 
	\end{enumerate}
	
	A union bound over all time steps completes our proof. 
\end{proof}

\subsection{Proof of Proposition~\ref{prop:1}}
In the following, we will bound each term in $\Lambda_{h+1}^k$ separately in a series of lemmas. 

\begin{lemma}\label{lemma:b_bound}
	With probability $1$, we have that
	$$\sum_{h=1}^H\sum_{k=1}^K (1+\frac{1}{H})^{h-1}(3b_h^k+5b_\Delta) \leq O(\sqrt{SAKH^5\iota}+ KH\Delta_r^{(1)} + KH^2\Delta_p^{(1)}).$$ 
\end{lemma}
\begin{proof}
	First, by the definition of $b_\Delta$, it is easy to see that $\sum_{h=1}^H\sum_{k=1}^K (1+\frac{1}{H})^{h-1}5b_\Delta \leq \sum_{h=1}^H\sum_{k=1}^KO(\Delta_r^{(1)} + H\Delta_p^{(1)})\leq O(KH\Delta_r^{(1)} + KH^2\Delta_p^{(1)})$. Recall our definition that $e_1 =H$ and $e_{i+1} = \floor{(1+\frac{1}{H})e_i},i\geq 1$. For a fixed $h\in[H]$, since $H^2 \geq 1$,
	\begin{align}
	&\sum_{k=1}^K (1+\frac{1}{H})^{h-1} 3b_h^k \leq \sum_{k=1}^K (1+\frac{1}{H})^{h-1} 12\sqrt{\frac{H^2}{\check{N}_h^k(s_h^k,a_h^k)}\iota}\nonumber\\
	=&12H\sqrt{\iota}\sum_{s,a}\sum_{j\geq 1} (1+\frac{1}{H})^{h-1}\sqrt{\frac{1}{e_j}}\sum_{k=1}^K\indicator\left[(s_h^k,a_h^k) = (s,a), \check{N}_h^k(s_h^k,a_h^k) = e_j\right]\nonumber\\
	=& 12H\sqrt{\iota}\sum_{s,a}\sum_{j\geq 1} (1+\frac{1}{H})^{h-1}w(s,a,j)\sqrt{\frac{1}{e_j}},\nonumber
	\end{align}
where $w(s,a,j)\defeq \sum_{k=1}^K \indicator\left[(s_h^k,a_h^k)=(s,a), \check{N}_h^k(s_h^k,a_h^k) = e_j\right]$, and $w(s,a) \defeq \sum_{j\geq 1} w(s,a,j)$. We then know that $\sum_{s,a} w(s,a) = K$. For a fixed $(s,a)$, let us now find an upper bound of $j$, denoted as $J$. Since each stage is $(1+\frac{1}{H})$ times longer than the previous stage, we know for $1\leq j\leq J$, $w(s,a,j) = \sum_{k=1}^K \indicator\left[(s_h^k,a_h^k)=(s,a), \check{N}_h^k(s_h^k,a_h^k) = e_j\right] = \floor{(1+\frac{1}{H})e_j}$. From $\sum_{j=1}^J w(s,a,j) = w(s,a)$, we get $e_J \leq (1+\frac{1}{H})^{J-1} \leq \frac{10}{1+\frac{1}{H}}\frac{w(s,a)}{H}$. Therefore, 
\[
\begin{aligned} 
\sum_{j\geq 1} (1+\frac{1}{H})^{h-1}w(s,a,j)\sqrt{\frac{1}{e_j}}\leq O\left(\sum_{j=1}^J \sqrt{e_j}\right)\leq O\left(\sqrt{w(s,a)H}\right). 
\end{aligned}
\]
Finally, by the Cauchy-Schwartz inequality, we have
\[
\sum_{h=1}^H\sum_{k=1}^K (1+\frac{1}{H})^{h-1} 3b_h^k =O\left(H^2\sqrt{\iota}\sum_{s,a}\sum_{j\geq 1} w(s,a,j)\sqrt{\frac{1}{e_j}}\right)\leq \sqrt{SAKH^5\iota}.
\]
Combining the bounds for $b_h^k$ and $b_\Delta$ completes the proof. 
\end{proof}

\begin{lemma}\label{lemma:phi_bound}
	With probability at least $1-\delta$, it holds that $$\sum_{h=1}^H\sum_{k=1}^K(1+\frac{1}{H})^{h-1}\phi_{h+1}^k \leq O(\sqrt{KH^3\iota}+KH\Delta_r^{(1)} + KH^2\Delta_p^{(1)}).$$ 
\end{lemma}
\begin{proof}
	We have that
	\begin{align}
	&\sum_{h=1}^H\sum_{k=1}^K(1+\frac{1}{H})^{h-1}\phi_{h+1}^k\nonumber\\
	= &\sum_{h=1}^H\sum_{k=1}^K(1+\frac{1}{H})^{h-1} \frac{1}{\check{n}}  \sum_{i=1}^{\check{n}} \left(P_h^{k} - \mbf{e}_{s_{h+1}^k} \right) \left(V_{h+1}^{\check{l}_i,\star} - V_{h+1}^{k,\pi}\right)(s_h^k,a_h^k)\nonumber\\
	= &\sum_{h=1}^H\sum_{k=1}^K(1+\frac{1}{H})^{h-1} \frac{1}{\check{n}}  \sum_{i=1}^{\check{n}} \left(P_h^{k} - \mbf{e}_{s_{h+1}^k} \right) \left(V_{h+1}^{\check{l}_i,\star} - V_{h+1}^{k,\star}+V_{h+1}^{k,\star} - V_{h+1}^{k,\pi}\right)(s_h^k,a_h^k)\nonumber\\
	\leq &\sum_{h=1}^H\sum_{k=1}^K(1+\frac{1}{H})^{h-1}2b_\Delta + \sum_{h=1}^H\sum_{k=1}^K(1+\frac{1}{H})^{h-1} \left(P_h^{k} - \mbf{e}_{s_{h+1}^k} \right) \left(V_{h+1}^{k,\star} - V_{h+1}^{k,\pi}\right)(s_h^k,a_h^k)\nonumber,
	\end{align}
	where the last inequality follows from Lemma~\ref{lemma:optimalQ} and the definition of $b_\Delta$. From the proof of Lemma~\ref{lemma:b_bound}, we know that the first term can be bounded as $$\sum_{h=1}^H\sum_{k=1}^K (1+\frac{1}{H})^{h-1}2b_\Delta \leq O(KH\Delta_r^{(1)} + KH^2\Delta_p^{(1)}).$$ Further, the second term is bounded by the Azuma-Hoeffding inequality as
	$$
	\sum_{h=1}^H\sum_{k=1}^K(1+\frac{1}{H})^{h-1} \left(P_h^{k} - \mbf{e}_{s_{h+1}^k} \right) \left(V_{h+1}^{k,\star} - V_{h+1}^{k,\pi}\right)(s_h^k,a_h^k)\leq O(\sqrt{KH^3\iota}).
	$$ 
	Combining the two terms completes the proof. 
\end{proof}

\begin{lemma}\label{lemma:xi_bound}
	With probability at least $1-(KH+1)\delta$, it holds that 
	\[
	\sum_{h=1}^H\sum_{k=1}^K(1+\frac{1}{H})^{h-1}\xi_{h+1}^k \leq O( \sqrt{SAKH^3\iota}+KH^2\Delta_p^{(1)}).
	\]
\end{lemma}
\begin{proof}
	We have that
	\begin{align}
	&\sum_{h=1}^H\sum_{k=1}^K(1+\frac{1}{H})^{h-1}\xi_{h+1}^k\nonumber\\
	=& \sum_{h=1}^H\sum_{k=1}^K(1+\frac{1}{H})^{h-1} \frac{1}{\check{n}}  \sum_{i=1}^{\check{n}} \left(P_h^{k} - \mbf{e}_{s_{h+1}^{\check{l}_i}} \right) \left(V_{h+1}^{\check{l}_i} - V_{h+1}^{\check{l}_i,\star}\right) (s_h^k,a_h^k)\nonumber\\
	=& \sum_{h=1}^H\sum_{k=1}^K(1+\frac{1}{H})^{h-1} \frac{1}{\check{n}}  \sum_{i=1}^{\check{n}} \left(P_h^{k} - P_h^{\check{l}_i} + P_h^{\check{l}_i} - \mbf{e}_{s_{h+1}^{\check{l}_i}} \right) \left(V_{h+1}^{\check{l}_i} - V_{h+1}^{\check{l}_i,\star}\right) (s_h^k,a_h^k)\nonumber\\
	\leq & O(KH^2\Delta_p^{(1)})+ \sum_{h=1}^H\sum_{k=1}^K(1+\frac{1}{H})^{h-1} \frac{1}{\check{n}}\sum_{i=1}^{\check{n}} \left( P_h^{\check{l}_i} - \mbf{e}_{s_{h+1}^{\check{l}_i}} \right) \left(V_{h+1}^{\check{l}_i} - V_{h+1}^{\check{l}_i,\star}\right) (s_h^k,a_h^k)\label{eqn:e1},
	\end{align}
	where the last step is by the fact that $V_{h+1}^{\check{l}_i}(s_h^k,a_h^k) \geq V_{h+1}^{\check{l}_i,\star}(s_h^k,a_h^k)$ from Lemma~\ref{lemma:optimistic_bound}, and then by H\"{o}lder's inequality and the triangle inequality. The following proof is analogous to the proof of Lemma 15 in~\citet{zhang2020almost}. For completeness we reproduce it here. We have
	\begin{align}
	&\sum_{h=1}^H\sum_{k=1}^K(1+\frac{1}{H})^{h-1} \frac{1}{\check{n}}\sum_{i=1}^{\check{n}} \left( P_h^{\check{l}_i} - \mbf{e}_{s_{h+1}^{\check{l}_i}} \right) \left(V_{h+1}^{\check{l}_i} - V_{h+1}^{\check{l}_i,\star}\right) (s_h^k,a_h^k)\nonumber\\
	=&  \sum_{h=1}^H\sum_{k=1}^K\sum_{j=1}^K(1+\frac{1}{H})^{h-1} \frac{1}{\check{n}_h^k}\sum_{i=1}^{\check{n}_h^k} \indicator\left[\check{l}_{h,i}^k = j\right]\left( P_h^{j} - \mbf{e}_{s_{h+1}^{j}} \right) \left(V_{h+1}^{j} - V_{h+1}^{j,\star}\right) (s_h^k,a_h^k) \nonumber\\
	=&\sum_{h=1}^H\sum_{k=1}^K\sum_{j=1}^K(1+\frac{1}{H})^{h-1} \frac{1}{\check{n}_h^k}\sum_{i=1}^{\check{n}_h^k} \indicator\left[\check{l}_{h,i}^k = j\right]\left( P_h^{j} - \mbf{e}_{s_{h+1}^{j}} \right) \left(V_{h+1}^{j} - V_{h+1}^{j,\star}\right) (s_h^j,a_h^j) \label{eqn:e2},
	\end{align}
	where~\eqref{eqn:e2} holds because $\check{l}_{h,i}^k(s_h^k,a_h^k) = j$ if and only if $j$ is in the previous stage of $k$ and $(s_h^k,a_h^k) = (s_h^j,a_h^j)$. For simplicity of notations, we define $\theta_{h+1}^k \defeq (1+\frac{1}{H})^{h-1}\sum_{j=1}^K \frac{1}{\check{n}_h^j}\sum_{i=1}^{\check{n}_h^j} \indicator\left[\check{l}_{h,i}^j = k\right]$. Then, we further have (note that we have swapped the notation of $j$ and $k$) $$\eqref{eqn:e2} = \sum_{h=1}^H\sum_{k=1}^K\theta_{h+1}^k\left( P_h^{k} - \mbf{e}_{s_{h+1}^{k}} \right) \left(V_{h+1}^{k} - V_{h+1}^{k,\star}\right) (s_h^k,a_h^k).$$
	
	For $(k,h)\in [K]\times [H]$, let $x_h^k$ denote the number of occurrences of the triple $(s_h^k,a_h^k,h)$ in the current stage. Define also $\tilde{\theta}_{h+1}^k \defeq (1+\frac{1}{H})^{h-1} \frac{\floor{(1+\frac{1}{H})x_h^k}}{x_h^k}\leq 3$. Define $\mc{K}\defeq \{(k,h):\theta_{h+1}^k = \tilde{\theta}_{h+1}^k\}$, and $\bar{\mc{K}}\defeq \{(k,h)\in [K]\times[H]:\theta_{h+1}^k \neq \tilde{\theta}_{h+1}^k\}$. Then, we have that
	\[
	\begin{aligned}
	\eqref{eqn:e2} = &\sum_{h=1}^H\sum_{k=1}^K\tilde{\theta}_{h+1}^k\left( P_h^{k} - \mbf{e}_{s_{h+1}^{k}} \right) \left(V_{h+1}^{k} - V_{h+1}^{k,\star}\right) (s_h^k,a_h^k) \\
	& \qquad+ \sum_{h=1}^H\sum_{k=1}^K (\theta_{h+1}^k-\tilde{\theta}_{h+1}^k) \left( P_h^{k} - \mbf{e}_{s_{h+1}^{k}} \right) \left(V_{h+1}^{k} - V_{h+1}^{k,\star}\right) (s_h^k,a_h^k).
	\end{aligned}
	\]
	Since $\tilde{\theta}_{h+1}^k$ is independent of $s_{h+1}^k$, by the Azuma-Hoeffding inequality, it holds with probability at least $1-\delta$ that  
	\begin{equation}\label{eqn:e3}
	\sum_{h=1}^H\sum_{k=1}^K\tilde{\theta}_{h+1}^k\left( P_h^{k} - \mbf{e}_{s_{h+1}^{k}} \right) \left(V_{h+1}^{k} - V_{h+1}^{k,\star}\right) (s_h^k,a_h^k) \leq O(\sqrt{KH^3\iota}). 
	\end{equation}
	It is easy to see that if $k$ is in a stage that is before the second last stage of the triple $(s_h^k,a_h^k,h)$, then $(k,h)\in\mc{K}$. For a triple $(s,a,h)$, define $\mc{K}_h^\perp(s,a) \defeq \{k\in[K]: k \text{ is in the second last stage of the triple } (s,a,h),\ (s_h^k,a_h^k) = (s,a) \}$. We have that 
	\begin{align}
	&\sum_{h=1}^H\sum_{k=1}^K (\theta_{h+1}^k-\tilde{\theta}_{h+1}^k) \left( P_h^{k} - \mbf{e}_{s_{h+1}^{k}} \right) \left(V_{h+1}^{k} - V_{h+1}^{k,\star}\right) (s_h^k,a_h^k)\nonumber\\
	=& \sum_{s,a,h}\sum_{k:(k,h)\in\bar{\mc{K}}} \indicator\left[(s_h^k,a_h^k) = (s,a)\right] (\theta_{h+1}^k-\tilde{\theta}_{h+1}^k) \left( P_h^{k} - \mbf{e}_{s_{h+1}^{k}} \right) \left(V_{h+1}^{k} - V_{h+1}^{k,\star}\right) (s,a)\nonumber\\
	=& \sum_{s,a,h}(\theta_{h+1}(s,a)-\tilde{\theta}_{h+1}(s,a))\sum_{k\in \mc{K}^\perp_h(s,a)} \left( P_h^{k} - \mbf{e}_{s_{h+1}^{k}} \right) \left(V_{h+1}^{k} - V_{h+1}^{k,\star}\right) (s,a),\label{eqn:e4}
	\end{align}
	where for a fixed triple $(s,a,h)$, we have defined $\theta_{h+1}(s,a)\defeq \theta_{h+1}^k$, for any $k\in \mc{K}^\perp_h(s,a)$. Note that $\theta_{h+1}(s,a)$ is well-defined, because $\theta_{h+1}^{k_1} = \theta_{h+1}^{k_2}, \forall k_1,k_2\in \mc{K}_h^\perp(s,a)$. Similarly, let $\tilde{\theta}_{h+1}(s,a) \defeq \tilde{\theta}_{h+1}^k$ for any $k \in \mc{K}^\perp_h(s,a)$, and $\tilde{\theta}_{h+1}(s,a)$ is also well-defined. By the Azuma-Hoeffding inequality and a union bound, it holds with probability at least $1-KH\delta$ that 
	\begin{align}
	\eqref{eqn:e4} \leq& \sum_{s,a,h}O\left(\sqrt{H^2\abs{\mc{K}_h^\perp (s,a)}\iota}\right)\nonumber\\
	= &\sum_{s,a,h}O\left(\sqrt{H^2\check{N}_h^{K+1}(s,a)\iota}\right)\nonumber\\
	\leq &O\bigg(\sqrt{SAH^3 \iota \sum_{s,a,h}\check{N}_h^{K+1}(s,a)}\bigg)\label{eqn:e5}\\
	\leq & O\left(\sqrt{SAKH^3\iota}\right)\label{eqn:e6}
	\end{align}
	where $\check{N}_h^{K+1}(s,a)$ is defined to be the total number of visitations to the triple $(s,a,h)$ over the entire $K$ episodes. \eqref{eqn:e5} is by the Cauchy-Schwartz inequality. \eqref{eqn:e6} holds because, by the way stages are defined, for each triple $(s,a,h)$, the length of its last two stages is at most an $O(1/H)$ fraction of the total number of visitations. 
	
	Combining~\eqref{eqn:e1}, \eqref{eqn:e3} and \eqref{eqn:e6} completes the proof. 
\end{proof}

\section{Proof of Theorem~\ref{thm:regret}}\label{app:proof_regret}
We introduce a few terms to facilitate the analysis. Denote by $s_h^k$ and $a_h^k$ respectively the state and action taken at step $h$ of episode $k$. Let $N_{h}^{k}(s, a), \check{N}_{h}^{k}(s, a), Q_{h}^{k}(s, a)$ and $V_{h}^{k}(s)$ denote, respectively, the values of $N_{h}(s, a), \check{N}_{h}(s, a), Q_{h}(s, a)$ and $V_{h}(s)$ at the \emph{beginning} of the $k$-th episode in Algorithm~\ref{alg:RQUCB}. Further, for the triple $(s_h^k,a_h^k,h)$, let $n_h^k$ be the total number of episodes that this triple has been visited prior to the current stage, and let $l_{h,i}^k$ denote the index of the episode that this triple was visited the $i$-th time among the total $n_h^k$ times. Similarly, let $\check{n}_h^k$ denote the number of visits to the triple $(s_h^k,a_h^k,h)$ in the stage right before the current stage, and let $\check{l}_{h,i}^k$ be the $i$-th episode among the $\check{n}_h^k$ episodes right before the current stage. For simplicity, we use $l_i$ and $\check{l}_i$ to denote $l_{h,i}^k$ and $\check{l}_{h,i}^k$, and $\check{n}$ to denote $\check{n}_h^k$, when $h$ and $k$ are clear from the context. We also use $\check{r}_h(s,a)$ and $\check{v}_h(s,a)$ to denote the values of $\check{r}_h(s_h^k,a_h^k)$ and $\check{v}_h(s_h^k,a_h^k)$ when updating the $Q_h(s_h^k,a_h^k)$ value in Line~\ref{line:16} of Algorithm~\ref{alg:RQUCB}.

We now proceed to analyze the dynamic regret in one epoch, and at the very end of Appendix~\ref{app:proof_regret}, we will see how to combine the dynamic regret over all the epochs to prove Theorem~\ref{thm:regret}. The following analysis will be conditioned on the successful event of Lemma~\ref{lemma:optimistic_bound}. 

The dynamic regret of Algorithm~\ref{alg:RQUCB} in epoch $d=1$ can hence be expressed as
\begin{align}
\mc{R}^{(d)}(\pi, K) &= \sum_{k=1}^{K}\left(V_{1}^{k,*}\left(s_{1}^{k}\right)-V_{1}^{k,\pi}\left(s_{1}^{k}\right)\right)\nonumber\\ 
&\leq \sum_{k=1}^{K}\left(V_{1}^{k}\left(s_{1}^{k}\right)-V_{1}^{k,\pi}\left(s_{1}^{k}\right)\right).\label{eqn:regret}
\end{align}
From the update rules of the value functions in Algorithm~\ref{alg:RQUCB}, we have
$$
 \begin{aligned}
	&V_{h}^{k}(s_{h}^{k})\!
	\leq\! \indicator{\left[n_h^k=0\right]}H \!+\! \frac{\check{r}_h(s_h^k,a_h^k)}{\check{N}_h^k(s_h^k,a_h^k)}\! +\! \frac{\check{v}_h(s_h^k,a_h^k)}{\check{N}_h^k(s_h^k,a_h^k)} \!+\! b_h^k \!+ \!2b_\Delta \\
	&=\! \indicator{\left[n_h^k=0\right]}H \!+\! \frac{\check{r}_h(s_h^k,a_h^k)}{\check{N}_h^k(s_h^k,a_h^k)} \!+\! \frac{1}{\check{n}}\sum_{i=1}^{\check{n}} V_{h+1}^{\check{l}_i}(s_{h+1}^{\check{l}_i}) \!+\! b_h^k\!+\! 2b_\Delta.
	\end{aligned}
$$

For ease of exposition, we define the following terms:
\begin{align} 
\delta_h^k \defeq V_h^k(s_h^k)\! -\! V_h^{k,\star}(s_h^k),\ \zeta_h^k \defeq V_h^k(s_h^k)\! -\! V_h^{k,\pi}(s_h^k).\label{def:zeta}
\end{align} 
We further define $\tilde{r}_h^k(s_h^k,a_h^k) \defeq \frac{\check{r}_h(s_h^k,a_h^k)}{\check{N}_h^k(s_h^k,a_h^k)}- r_h^k(s_h^k,a_h^k)$. Then by the Hoeffding's inequality, it holds with high probability that 
\begin{align}
\tilde{r}_h^k(s_h^k,a_h^k) &\leq \frac{1}{\check{n}}\sum_{i=1}^{\check{n}}r_h^{\check{l}_i}(s_h^k,a_h^k)  + \sqrt{\frac{\iota}{\check{n}}}- r_h^k(s_h^k,a_h^k)\nonumber\\
&\leq b_h^k + b_\Delta.\label{eqn:tmp1}
\end{align}
By the Bellman equation $V_h^{k,\pi}(s_h^k)= Q_h^{k,\pi}(s_h^k,\pi(s_h^k))= r_h^k(s_h^k,a_h^k) + P_h^kV_{h+1}^{k,\pi}(s_h^k,a_h^k)$, we have 
\begin{align} 
\zeta_h^k 
\leq& \indicator{\left[n_h^k=0\right]}H + \frac{1}{\check{n}}\sum_{i=1}^{\check{n}} V_{h+1}^{\check{l}_i}(s_{h+1}^{\check{l}_i}) + b_h^k+ 2b_\Delta + \tilde{r}_h^k(s_h^k,a_h^k) - P_h^kV_{h+1}^{k,\pi}(s_h^k,a_h^k)\nonumber\\
\leq& \indicator{\left[n_h^k=0\right]}H + 
\frac{1}{\check{n}} \sum_{i=1}^{\check{n}} P_h^{\check{l}_i}V_{h+1}^{\check{l}_i}(s_h^k,a_h^k)- P_h^kV_{h+1}^{k,\pi}(s_h^k,a_h^k)
+ 3b_h^k + 3b_\Delta\label{eqn:c1}\\
=& \indicator{\left[n_h^k=0\right]}H + 
\underbrace{\frac{1}{\check{n}} \sum_{i=1}^{\check{n}} \left(P_h^{\check{l}_i} - P_h^k\right) V_{h+1}^{\check{l}_i} (s_h^k,a_h^k)}_{\tiny \circled{1}} + \underbrace{\frac{1}{\check{n}}  \sum_{i=1}^{\check{n}} P_h^{k} \left(V_{h+1}^{\check{l}_i} - V_{h+1}^{\check{l}_i,\star}\right) (s_h^k,a_h^k)}_{\tiny \circled{2}}+ 3b_h^k+ 3b_\Delta\nonumber\\
&+ \underbrace{\frac{1}{\check{n}}  \sum_{i=1}^{\check{n}} P_h^{k} \left(V_{h+1}^{\check{l}_i,\star} - V_{h+1}^{k,\pi}\right)(s_h^k,a_h^k)}_{\tiny \circled{3}} \label{eqn:c2},
\end{align}
where~\eqref{eqn:c1} is by the Azuma-Hoeffding inequality and by~\eqref{eqn:tmp1}. In the following, we bound each term in~\eqref{eqn:c2} separately.   First, by H\"{o}lder's inequality, we have 
\begin{equation}\label{eqn:c4}
{ \circled{1}} \leq \frac{1}{\check{n}}\sum_{i=1}^{\check{n}} \Delta_p^{(1)} (H-h) \leq  b_\Delta.
\end{equation}
Let $\mbf{e}_j$ denote a standard basis vector of proper dimensions that has a $1$ at the $j$-th entry and $0$s at the others, in the form of $(0,\dots,0,1,0, \dots,0)$. Recall the definition of $\delta_h^k$ in~\eqref{def:zeta}, and we have
\begin{align}
{ \circled{2}} =&  \underbrace{\frac{1}{\check{n}}  \sum_{i=1}^{\check{n}} \left(P_h^{k} - \mbf{e}_{s_{h+1}^{\check{l}_i}} \right) \left(V_{h+1}^{\check{l}_i} - V_{h+1}^{\check{l}_i,\star}\right) (s_h^k,a_h^k)}_{\xi_{h+1}^k}+ \frac{1}{\check{n}}\sum_{i=1}^{\check{n}} \delta_{h+1}^{\check{l}_i} =  \xi_{h+1}^k+ \frac{1}{\check{n}}\sum_{i=1}^{\check{n}} \delta_{h+1}^{\check{l}_i} .\label{eqn:c5}
\end{align}
Finally, recalling the definition of $\zeta_h^k$ in~\eqref{def:zeta}, we have that 
\begin{align}
{\circled{3}} =& \frac{1}{\check{n}}  \sum_{i=1}^{\check{n}}\left(V_{h+1}^{\check{l}_i,\star}(s_{h+1}^k) - V_{h+1}^{k,\pi}(s_{h+1}^k)\right) \nonumber\\
&+  \underbrace{\frac{1}{\check{n}}  \sum_{i=1}^{\check{n}} \left(P_h^{k} - \mbf{e}_{s_{h+1}^k} \right) \left(V_{h+1}^{\check{l}_i,\star} - V_{h+1}^{k,\pi}\right)(s_h^k,a_h^k)}_{\phi_{h+1}^k}\nonumber\\
=& \frac{1}{\check{n}}  \!\sum_{i=1}^{\check{n}}\!\left(V_{h+1}^{\check{l}_i,\star}(s_{h+1}^k) \!-\! V_{h+1}^{k,\star}(s_{h+1}^k)\right) \!+\! \zeta_{h+1}^k \!-\! \delta_{h+1}^k \!+\! \phi_{h+1}^k\nonumber\\
\leq& b_\Delta + \zeta_{h+1}^k - \delta_{h+1}^k + \phi_{h+1}^k \label{eqn:c3}
\end{align}
where inequality~\eqref{eqn:c3} is by Lemma~\ref{lemma:optimalQ}. Combining~\eqref{eqn:c2}, \eqref{eqn:c4}, \eqref{eqn:c5}, and~\eqref{eqn:c3} leads to
\begin{align}
\zeta_h^k \leq &\indicator\left[n_h^k=0\right]H + \frac{1}{\check{n}}\sum_{i=1}^{\check{n}} \delta_{h+1}^{\check{l}_i} + \xi_{h+1}^k + \zeta_{h+1}^k- \delta_{h+1}^k+\phi_{h+1}^k +3b_h^k+ 5b_\Delta.\label{eqn:zeta}
\end{align}
To find an upper bound of $\sum_{k=1}^K \zeta_h^k$, we proceed to upper bound each term on the RHS of~\eqref{eqn:zeta} separately. First, notice that $\sum_{k=1}^K \indicator\left[n_h^k=0\right]\leq SAH$, because each fixed triple $(s,a,h)$ contributes at most $1$ to $\sum_{k=1}^K \indicator\left[n_h^k=0\right]$. \wmedit{In the following, we upper bound the second term in~\eqref{eqn:zeta}. Notice that
\begin{equation}\label{eqn:d1}
\sum_{k=1}^K \frac{1}{\check{n}_h^k}\sum_{i=1}^{\check{n}_h^k} \delta_{h+1}^{\check{l}_{h,i}^k} = \sum_{k=1}^K \sum_{j=1}^K \frac{1}{\check{n}_h^k} \delta_{h+1}^{j} \sum_{i=1}^{\check{n}_h^k} \indicator\left[\check{l}_{h,i}^k = j\right] =\sum_{j=1}^K  \delta_{h+1}^{j} \sum_{k=1}^K  \frac{1}{\check{n}_h^k} \sum_{i=1}^{\check{n}_h^k} \indicator\left[\check{l}_{h,i}^k = j\right].
\end{equation}
For a fixed episode $j$, notice that $\sum_{i=1}^{\check{n}_h^k} \indicator[\check{l}_{h,i}^k = j]\leq 1$, and that $\sum_{i=1}^{\check{n}_h^k} \indicator[\check{l}_{h,i}^k = j]= 1$ happens if and only if $(s_h^k,a_h^k) = (s_h^j,a_h^j)$ and $(j,h)$ lies in the previous stage of $(k,h)$ with respect to the triple $(s_h^k,a_h^k,h)$.  Let $\mc{K}\defeq \{ k\in[K]: \sum_{i=1}^{\check{n}_h^k} \indicator[\check{l}_{h,i}^k = j]= 1 \}$; then, we know that every element $k\in \mc{K}$ has the same value of $\check{n}_h^k$, i.e., there  exists an integer $N_j>0$, such that $\check{n}_h^k = N_j,\forall k \in \mc{K}$. Further, by our definition of the stages, we know that $|\mc{K}|\leq (1+\frac{1}{H})N_j$, because the current stage is at most $(1+\frac{1}{H})$ times longer than the previous stage. Therefore, for every $j$, we know that 
\begin{equation}\label{eqn:d2}
\sum_{k=1}^K  \frac{1}{\check{n}_h^k} \sum_{i=1}^{\check{n}_h^k} \indicator\left[\check{l}_{h,i}^k = j\right] \leq 1+ \frac{1}{H}.
\end{equation} 
Substituting it back into \eqref{eqn:d1} leads to
\begin{equation}\label{eqn:d0}
\sum_{k=1}^K \frac{1}{\check{n}_h^k}\sum_{i=1}^{\check{n}_h^k} \delta_{h+1}^{\check{l}_{h,i}^k}\leq (1+\frac{1}{H}) \sum_{k=1}^{K} \delta_{h+1}^{k}.
\end{equation} }

Combining~\eqref{eqn:zeta} and \eqref{eqn:d0}, we now have that 
\begin{align} 
\sum_{k=1}^K\zeta_h^k \leq &SAH^2+ \frac{1}{H}\sum_{k=1}^{K} \delta_{h+1}^{k}+ \sum_{k=1}^{K}\left(\xi_{h+1}^k + \zeta_{h+1}^k + \phi_{h+1}^k +3b_h^k+ 5b_\Delta\right)\nonumber\\
\leq & SAH^2+ (1+\frac{1}{H})\sum_{k=1}^{K} \zeta_{h+1}^{k}+ \sum_{k=1}^{K}\underbrace{\left(\xi_{h+1}^k + \phi_{h+1}^k +3b_h^k+ 5b_\Delta\right)}_{\Lambda_{h+1}^k}\label{eqn:d3},
\end{align}
where in~\eqref{eqn:d3} we have used the fact that $\delta_{h+1}^k \leq \zeta_{h+1}^k$, which in turn is due to the optimality that $V_h^{k,\star}(s_h^k)\geq V_h^{k,\pi}(s_h^k)$. Notice that we have $\zeta_h^k$ on the LHS of~\eqref{eqn:d3} and $\zeta_{h+1}^k$ on the RHS. By iterating~\eqref{eqn:d3} over $h = H, H-1,\dots, 1$, we conclude that 
\begin{equation}\label{eqn:zeta_bound}
\sum_{k=1}^K \zeta_1^k \leq O\left( SAH^3 + \sum_{h=1}^H \sum_{k=1}^K (1+\frac{1}{H})^{h-1}\Lambda_{h+1}^k \right). 
\end{equation}
We bound $\sum_{h=1}^H \sum_{k=1}^K (1+\frac{1}{H})^{h-1}\Lambda_{h+1}^k$ in the proposition below. Its proof relies on a series of lemmas in Appendix~\ref{app:lemmas} that upper bound each term in $\Lambda_{h+1}^k$ separately. 
\begin{proposition}\label{prop:1}
	With probability at least $1-(KH+2)\delta$, it holds that 
	{\[  
		\sum_{h=1}^H \!\sum_{k=1}^K (1+\frac{1}{H})^{h-1}\!\Lambda_{h+1}^k \!\leq\! \widetilde{O}( \sqrt{SAKH^5} + KH\Delta_r^{(1)} \!+ KH^2\Delta_p^{(1)}). 
		\]}
\end{proposition}
Now we are ready to prove Theorem~\ref{thm:regret}.
\begin{proof}
	(of Theorem~\ref{thm:regret}) By~\eqref{eqn:regret} and~\eqref{eqn:zeta_bound}, and by replacing $\delta$ with $\frac{\delta}{KH+2}$ in Proposition~\ref{prop:1}, we know that the dynamic regret in epoch $d=1$ can be upper bounded with probability at least $1-\delta$ by:
	\begin{align}
	\mc{R}^{(d)}(\pi,K) \leq  \widetilde{O}( SAH^3 \!+\! \sqrt{SAKH^5}\! +\! KH\Delta_r^{(1)} \!+ \!KH^2\Delta_p^{(1)} ),
	\end{align}
	and this holds for every epoch $d\in [D]$. Suppose $T= \Omega(SA\Delta H^2)$; summing up the dynamic regret over all the $D$ epochs gives us an upper bound of $\widetilde{O}( D\sqrt{SAKH^5} + \sum_{d=1}^DKH\Delta_r^{(d)} + \sum_{d=1}^DKH^2\Delta_p^{(d)} )$. Recall the definition that $\sum_{d=1}^D\Delta_r^{(d)} \leq \Delta_r$, $\sum_{d=1}^D\Delta_p^{(d)} \leq \Delta_p$, $\Delta = \Delta_r+\Delta_p$, and that $K = \Theta(\frac{T}{DH})$. By setting $D=S^{-\frac{1}{3}}A^{-\frac{1}{3}}\Delta^{\frac{2}{3}}H^{-\frac{2}{3}}T^{\frac{1}{3}}$, the dynamic regret over the entire $T$ steps is bounded by $\mc{R}(\pi,M) \leq \widetilde{O}( S^{\frac{1}{3}}A^{\frac{1}{3}}\Delta^{\frac{1}{3}} H^{\frac{5}{3}} T^{\frac{2}{3}} ),$	which completes the proof.
\end{proof}

\section{Proof Sketch of Theorem~\ref{thm:regret_no_budget}}
\begin{proofsketch}
	We only outline the difference with respect to the proof of Theorem~\ref{thm:regret}. The reader should have no difficulty recovering the complete proof by following the same routine as  in the proof of Theorem~\ref{thm:regret}. Specifically, it suffices to investigate the steps that are involved with Lemma~\ref{lemma:optimistic_bound}. 
	
	The dynamic regret of the new algorithm in epoch $d=1$ can now be expressed as
	\begin{equation}\label{eqn:g2}
	\mc{R}^{(d)}(\pi, K) = \sum_{k=1}^{K}\left(V_{1}^{k,*}\left(s_{1}^{k}\right)-V_{1}^{k,\pi}\left(s_{1}^{k}\right)\right) \leq \sum_{k=1}^{K}\left(V_{1}^{k}\left(s_{1}^{k}\right)-V_{1}^{k,\pi}\left(s_{1}^{k}\right)\right) + 2KHb_\Delta,
	\end{equation}
	where we applied the results of Lemma~\ref{lemma:new_optimistic_bound} instead of Lemma~\ref{lemma:optimistic_bound}. The reader should bear in mind that from the new update rules of the value functions, we now have 
	\begin{equation}\label{eqn:g1}
	V_{h}^{k}(s_{h}^{k})\leq \indicator{\left[n_h^k=0\right]}H + \frac{\check{r}_h(s_h^k,a_h^k)}{\check{N}_h^k(s_h^k,a_h^k)} + \frac{\check{v}_h(s_h^k,a_h^k)}{\check{N}_h^k(s_h^k,a_h^k)} + b_h^k, 
	\end{equation}
	where the RHS no longer has the additional bonus term $b_\Delta$. If we define $\zeta_h^k,\ \xi_{h+1}^k$, and $\phi_{h+1}^k$ in the same way as before, the reader can easily verify that all the derivations until Equation~\eqref{eqn:zeta_bound} still holds, although the value of $\Lambda_{h+1}^k$ should be re-defined as $\Lambda_{h+1}^k \defeq \xi_{h+1}^k + \phi_{h+1}^k +3b_h^k+ 3b_\Delta$
	due to the new upper bound in~\eqref{eqn:g1} that is independent of $b_\Delta$. Proposition~\ref{prop:1} also follows analogously though some additional attention should be paid to the proof of Lemma~\ref{lemma:xi_bound} where the results of Lemma~\ref{lemma:optimistic_bound} have been utilized. Finally, we obtain the dynamic regret upper bound in epoch $d=1$ as follows:
	\begin{equation}\label{eqn:a5}
	\mc{R}^{(d)}(\pi,K) \leq  \widetilde{O}\left( SAH^3 + \sqrt{SAKH^5} + KH\Delta_r^{(1)} + KH^2\Delta_p^{(1)} \right) + 2KHb_\Delta,	
	\end{equation}
	where the additional term $2KHb_\Delta$ comes from~\eqref{eqn:g2}. From our definition of $b_\Delta$, we can easily see that $2KHb_\Delta \leq O(KH\Delta_r^{(1)} + KH^2\Delta_p^{(1)})$. Therefore, we can conclude that the dynamic regret upper bound in one epoch remains the same order, which leaves the dynamic regret over the entire horizon also unchanged. 
\end{proofsketch}

\section{Proof of Theorem~\ref{thm:freedman}}\label{app:proof_freedman}
Similar to the proofs of Theorems~\ref{thm:regret} and~\ref{thm:regret_no_budget}, we start with the dynamic regret in one epoch, and then extend to all epochs in the end. The proof follows the same routine as in the proofs of Theorems~\ref{thm:regret} and~\ref{thm:regret_no_budget}. Given that a rigorous analysis on the Freedman-based bonus with variance reduction is present in~\citet{zhang2020almost}, one should not find it difficult to extend our Hoeffding-based algorithm to a Freedman-based one. Therefore, rather than providing a complete proof of Theorem~\ref{thm:freedman}, in the following, we sketch the differences and highlight the additional analysis needed that is not covered by the proof of Theorem~\ref{thm:regret_no_budget} and~\citet{zhang2020almost}. 

To facilitate the analysis, first recall a few notations $N_h^k,\check{N}_h^k,Q_h^k(s,a), V_h^k(s), n_h^k,l_{h,i}^k, \check{n}_h^k,\check{l}_{h,i}^k,l_i$ and $\check{l}_i$ that we have defined in Appendix~\ref{app:proof_regret}. In addition, when $(h,k)$ is clear from the context, we drop the time indices and simply use $\check{\mu}, \check{\sigma}, \mu^{\re}, \sigma^{\re}$ to denote their corresponding values in the computation of the $Q_h(s_h^k,a_h^k)$ value in Line~\ref{line:16} of Algorithm~\ref{alg:RQUCB}. 

We start with the following lemma, which is an analogue of Lemma~\ref{lemma:new_optimistic_bound} but requires a more careful treatment of variations accumulated in $\mu^{\re}$ and $\check{\mu}_h$. It states that the optimistic $Q_h^k(s,a)$ is an ``upper bound'' of the optimal $Q_h^{k,\star}(s,a)$ subject to an error term of the order $2(H-h+1)b_\Delta$ with high probability. 
\begin{lemma}\label{lemma:optimistic_freedman}
	\wmedit{(Freedman, no local budgets) For $\delta \in (0,1)$, with probability at least $1-2KH\delta$, it holds that $Q_{h}^{k,\star}(s, a) - 4(H-h+1)b_\Delta  \leq Q_{h}^{k+1}(s, a) \leq Q_{h}^{k}(s, a),\forall (s,a,h,k)\in \mc{S}\times\mc{A}\times[H]\times[K]$. }
\end{lemma} 
\begin{proof}
	\wmedit{It should be clear from the way we update $Q_h(s,a)$ that $Q_h^k(s,a)$ is monotonically decreasing in $k$. We now prove $Q_h^{k,\star}(s,a) - 4(H-h+1)b_\Delta \leq Q_h^{k+1}(s,a)$ for all $s,a,h,k$ by induction on $k$. First, it holds for $k=1$ by our initialization of $Q_h(s,a)$. For $k\geq 2$, now suppose $Q_h^{j,\star}(s,a) - 4(H-h+1)b_\Delta\leq Q_h^{j}(s,a)$ for all $s,a,h$ and $1 \leq j \leq k$. For a fixed triple $(s,a,h)$, we consider the following two cases.  }
	
	\textbf{Case 1:} $Q_h(s,a)$ is updated in episode $k$. Notice that it suffices to analyze the case where $Q_h(s,a)$ is updated using $\underline{b}_h^k$, because the other case of $b_h^k$ would be exactly the same as in Lemma~\ref{lemma:new_optimistic_bound}. With probability at least $1-\delta$,
	\begin{align}
	Q_h^{k+1}(s,a) = & \frac{\check{r}_h(s,a)}{\check{N}_h^k(s,a)} + \frac{\mu^\re (s,a)}{N_h^k(s,a)} +  \frac{\check{\mu}_h(s,a)}{\check{N}_h^k(s,a)} + 2\underline{b}_h^k \nonumber\\
	= & \frac{\check{r}_h(s,a)}{\check{n}} + \underbrace{\frac{1}{n}\sum_{i=1}^n\left( V_{h+1}^{\re,l_i}(s_{h+1}^{l_i}) - P_h^{l_i}V_{h+1}^{\re,l_i}(s,a) \right)}_{\chi_1} \nonumber\\
	&+ \underbrace{\frac{1}{\check{n}}\sum_{i=1}^{\check{n}}\left[\left( V_{h+1}^{\check{l}_i}(s_{h+1}^{\check{l}_i}) - V_{h+1}^{\re,\check{l}_i}(s_{h+1}^{\check{l}_i}) \right) - \left( P_h^{\check{l}_i}V_{h+1}^{\check{l}_i} - P_{h}^{\check{l}_i}V_{h+1}^{\re,\check{l}_i} \right)(s,a) \right] }_{\chi_2} \nonumber\\
	&+\underbrace{\frac{1}{n}\sum_{i=1}^n P_h^{l_i} V_{h+1}^{\re, {l_i}} + \frac{1}{\check{n}}\sum_{i=1}^{\check{n}} \left( P_h^{\check{l}_i}V_{h+1}^{\check{l}_i} - P_h^{\check{l}_i}V_{h+1}^{\re,\check{l}_i} \right)(s,a)}_{\chi_3}+ 2\underline{b}_h^k. \label{eqn:f1}
	\end{align} 
	In the following, we will bound each term in~\eqref{eqn:f1} separately. First, we have that
	\begin{align}
	\chi_3  =& \frac{1}{n}\sum_{i=1}^n\left( P_h^{l_i} V_{h+1}^{\re, {l_i}} - P_h^kV_{h+1}^{\re,l_i}\right)(s,a)\label{eqn:f2}\\
	&- \frac{1}{\check{n}}\sum_{i=1}^{\check{n}} \left( P_h^{\check{l}_i}V_{h+1}^{\re,\check{l}_i} - P_h^kV_{h+1}^{\re,\check{l}_i} \right)(s,a)\label{eqn:f3}\\
	& + \frac{1}{n}\sum_{i=1}^n P_h^kV_{h+1}^{\re,l_i}(s,a) - \frac{1}{\check{n}}\sum_{i=1}^{\check{n}}P_h^kV_{h+1}^{\re,\check{l}_i}(s,a) + \frac{1}{\check{n}}\sum_{i=1}^{\check{n}}  P_h^{\check{l}_i}V_{h+1}^{\check{l}_i}(s,a)\label{eqn:f4}\\
	\geq &\frac{1}{\check{n}}\sum_{i=1}^{\check{n}}  P_h^{\check{l}_i}V_{h+1}^{\check{l}_i}(s,a)-2b_\Delta\label{eqn:f8},
	\end{align}
	where~\eqref{eqn:f2}$\geq -b_\Delta$ and~\eqref{eqn:f3}$\geq -b_\Delta$ by H\"{o}lder's inequality and the definition of $b_\Delta$. In~\eqref{eqn:f4}, we have that $\frac{1}{n}\sum_{i=1}^n P_h^kV_{h+1}^{\re,l_i}(s,a) - \frac{1}{\check{n}}\sum_{i=1}^{\check{n}}P_h^kV_{h+1}^{\re,\check{l}_i}(s,a) \geq 0$, because $V_{h+1}^{\re,k}(s)$ is non-increasing in $k$. 
	
	Following a similar procedure as in Lemma 10, Lemma 12, and Lemma 13 in~\citet{zhang2020almost}, we can further bound $\abs{\chi_1}$ and $\abs{\chi_2}$ as follows:
	\begin{align}
	\left|\chi_{1}\right| &\leq 2 \sqrt{\frac{\nu^{\re} \iota}{n}}+\frac{5 H \iota^{\frac{3}{4}}}{n^{\frac{3}{4}}}+\frac{2 \sqrt{\iota}}{T n}+\frac{2 H \iota}{n},\label{eqn:f6}\\
	\left|\chi_{2}\right| &\leq 2 \sqrt{\frac{\check{\nu} \iota}{\check{n}}}+\frac{5 H \iota^{\frac{3}{4}}}{\check{n}^{\frac{3}{4}}}+\frac{2 \sqrt{\iota}}{T \check{n}}+\frac{2 H \iota}{\check{n}},\label{eqn:f7}
	\end{align}
	where $\nu^{\text {ref }}\defeq \frac{\sigma^{\text {ref }}}{n}-\left(\frac{\mu^{\text {ref }}}{n}\right)^{2}$ and $\check{{\nu}}\defeq\frac{\check{\sigma}}{\check{n}}-\left(\frac{\check{\mu}}{\check{n}}\right)^{2}$. These are the steps where Freedman's inequality~\citep{freedman1975tail} come into use, and we omit these steps since they are essentially the same as the derivations in~\citet{zhang2020almost}. We can see from~\eqref{eqn:f6}, \eqref{eqn:f7}, and the definition of $\underline{b}_h^k$ that $\abs{\chi_1} + \abs{\chi_2}\leq \underline{b}_h^k$. 
	
	Substituting the results on $\chi_1,\chi_2$ and $\chi_3$ back to~\eqref{eqn:f1}, it holds that with probability at least $1-\delta$,
	\begin{align}
	Q_{h}^{k+1}(s,a) = & \frac{\check{r}_h(s,a)}{\check{n}}  + \chi_1 + \chi_2 + \chi_3 + 2\underline{b}_h^k\nonumber\\
	\geq & \frac{\check{r}_h(s,a)}{\check{n}}  + \frac{1}{\check{n}}\sum_{i=1}^{\check{n}}  P_h^{\check{l}_i}V_{h+1}^{\check{l}_i}(s,a) + \underline{b}_h^k-2b_\Delta\label{eqn:f5}\\
	\geq & \frac{\check{r}_h(s,a)}{\check{n}}  + \frac{1}{\check{n}}\sum_{i=1}^{\check{n}} P_h^{\check{l}_i}V_{h+1}^{\check{l}_i,\star}(s,a)-4(H-h)b_\Delta + \underline{b}_h^k-2b_\Delta\label{eqn:f9}\\
	= & \frac{\check{r}_h(s,a)}{\check{n}}  + \frac{1}{\check{n}}\sum_{i=1}^{\check{n}} \left( Q_h^{\check{l}_i,\star}(s,a) - r_h^{\check{l}_i}(s,a) \right) + \underline{b}_h^k-4(H-h)b_\Delta - 2b_\Delta\nonumber\\
	\geq& \frac{1}{\check{n}}\sum_{i=1}^{\check{n}}  Q_h^{\check{l}_i,\star}(s,a)  - 4(H-h)b_\Delta-2b_\Delta \geq Q_h^{k,\star}(s,a)- 4(H-h)b_\Delta -3 b_\Delta\label{eqn:f13},
	\end{align}
	where in~\eqref{eqn:f5} we used~\eqref{eqn:f8}, \eqref{eqn:f6}, \eqref{eqn:f7}, and the definition of $\underline{b}_h^k$ in Algorithm~\ref{alg:RQUCB}. \eqref{eqn:f9} is by the induction hypothesis that $Q_{h+1}^{\check{l}_i}(s_{h+1}^{\check{l}_i},a)\geq Q_{h+1}^{\check{l}_i,\star}(s_{h+1}^{\check{l}_i},a)-2(H-h)b_\Delta,\forall a\in\mc{A},1\leq \check{l}_i \leq k$. The second to last inequality holds due to the Hoeffding's inequality that $\frac{1}{\check{n}}\left( \sum_{i=1}^{\check{n}}r_h^{\check{l}_i}(s,a) - \check{r}_h(s,a) \right) \leq \sqrt{\frac{\iota}{\check{n}}} \leq \underline{b}_h^k$ with high probability. Finally, the last inequality follows from Lemma~\ref{lemma:optimalQ}. 
	
	According to the monotonicity of $Q_h^{k}(s,a)$, we can conclude from~\eqref{eqn:f13} that $Q_{h}^{k,\star}(s, a)- 4(H-h+1)b_\Delta \leq Q_{h}^{k+1}(s, a) \leq Q_{h}^{k}(s, a)$. In fact, we have proved the stronger statement $Q_h^{k+1}(s,a)\geq Q_h^{k,\star}(s,a)- 4(H-h+1)b_\Delta+b_\Delta$ that will be useful in Case 2 below. 
	
	\textbf{Case 2:} $Q_h(s,a)$ is not updated in episode $k$. Then, there are two possibilities:
	
	\begin{enumerate} 
		\item If $Q_h(s,a)$ has never been updated from episode $1$ to episode $k$: It is easy to see that $Q_h^{k+1}(s,a) = Q_h^{k}(s,a) = \dots = Q_h^{1}(s,a)=H-h+1 \geq Q_h^{k,\star}(s,a)$ holds. 
		\item If $Q_h(s,a)$ has been updated at least once from episode $1$ to episode $k$: Let $j$ be the index of the latest episode that $Q_h(s,a)$ was updated. Then, from our induction hypothesis and Case~1, we know that $Q_h^{j+1}(s,a)\geq Q_h^{j,\star}(s,a)- 4(H-h+1)b_\Delta+b_\Delta$.  
		Since $Q_h(s,a)$ has not been updated from episode $j+1$ to episode $k$, we know that $Q_h^{k+1}(s,a) = Q_h^{k}(s,a)=\dots = Q_h^{j+1}(s,a) \geq Q_h^{j,\star}(s,a) - 4(H-h+1)b_\Delta+b_\Delta\geq Q_h^{k,\star}(s,a)- 4(H-h+1)b_\Delta$, where the last inequality holds because of Lemma~\ref{lemma:optimalQ}. 
	\end{enumerate}
	
	A union bound over all time steps completes our proof. 
\end{proof}

Conditional on the successful event of Lemma~\ref{lemma:optimistic_freedman}, the dynamic regret of RestartQ-UCB Freedman in epoch $d=1$ can hence be expressed as
\begin{equation}\label{eqn:regret_freedman}
\mc{R}^{(d)}(\pi, K) = \sum_{k=1}^{K}\left(V_{1}^{k,*}\left(s_{1}^{k}\right)-V_{1}^{k,\pi}\left(s_{1}^{k}\right)\right) \leq \sum_{k=1}^{K}\left(V_{1}^{k}\left(s_{1}^{k}\right)-V_{1}^{k,\pi}\left(s_{1}^{k}\right)\right) + 4KHb_\Delta.
\end{equation}
From the update rules of the value functions in Algorithm~\ref{alg:RQUCB}, we have
$$
\begin{aligned}
&V_{h}^{k}(s_{h}^{k})\leq \indicator{\left[n_h^k=0\right]}H + \frac{\check{r}_h(s_h^k,a_h^k)}{\check{n}} + \frac{{\mu}_h^{\re,k}}{{n}}+ \frac{\check{\mu}_h^{k}}{\check{n}} + 2\underline{b}_h^k \\
=& \indicator{\left[n_h^k=0\right]}H \!+\! \frac{\check{r}_h(s_h^k,a_h^k)}{\check{n}} \!+\! \frac{1}{n}\sum_{i=1}^n V_{h+1}^{\re,l_i}(s_{h+1}^{l_i}) \!+\! \frac{1}{\check{n}}\sum_{i=1}^{\check{n}} ( V_{h+1}^{\check{l}_i}(s_{h+1}^{\check{l}_i}) \!-\! V_{h+1}^{\re,\check{l}_i}(s_{h+1}^{\check{l}_i}) ) \!+\! 2\underline{b}_h^k .
\end{aligned}
$$
If we again define $\zeta_h^k \defeq V_h^k(s_h^k) - V_h^{k,\pi}(s_h^k)$, we can follow a similar routine as in the proof of Theorem~\ref{thm:regret} (details can be found in~\citet{zhang2020almost}) and obtain 
\[
	\sum_{k=1}^K \zeta_1^k \leq O\left( SAH^3 + \sum_{h=1}^H \sum_{k=1}^K (1+\frac{1}{H})^{h-1}\Lambda_{h+1}^k \right), 
\]
where $\Lambda_{h+1}^k \defeq \psi_{h+1}^k + \xi_{h+1}^k + \phi_{h+1}^k + 4\underline{b}_h^k + 4b_\Delta$ with the following definitions: 
\[
\begin{aligned}
	\psi_{h+1}^k &\defeq \frac{1}{n_h^k}\sum_{i=1}^{n_h^k}\left( P_h^kV_{h+1}^{\re,l_i} - P_h^kV_{h+1}^{\re,K+1} \right)(s_h^k,a_h^k),\\
	\xi_{h+1}^k &\defeq \frac{1}{\check{n}_h^k}\sum_{i=1}^{\check{n}_h^k}\left(P_h^k - \mathbf{e}_{s_{h+1}^{\check{l}_i}} \right)  \left( V_{h+1}^{\check{l}_i} - V_{h+1}^{\check{l}_i,\star} \right)(s_h^k,a_h^k),\\
	\phi_{h+1}^k &\defeq \left( P_{h}^k - \mbf{e}_{s_{h+1}^k} \right) \left(V_{h+1}^{\check{l}_i,\star} - V_{h+1}^{k,\pi} \right) (s_h^k,a_h^k).
\end{aligned}
\]
An upper bound on the first four terms in $\Lambda_{h+1}^k$ is derived in the proof of Lemma 7 in~\citet{zhang2020almost} (There is an extra term of $\sqrt{\frac{1}{\check{n}}\iota}$ in our defnition of $\underline{b}_h^k$ compared to theirs, but it does not affect the leading term in the upper bound). By further recalling the definition of $b_\Delta$, we can obtain the following lemma. 
\begin{lemma}
	(Lemma 7 in~\citet{zhang2020almost}) With probability at least $(1-O(H^2T^4\delta))$, it holds that 
	\[
		\sum_{h=1}^{H} \sum_{k=1}^{K}(1+\frac{1}{H})^{h-1} \!\Lambda_{h+1}^{k}\!=\!O\left(\!\sqrt{S A H^{3} K \iota}\!+\! \sqrt{KH^3 \iota} \log (KH)\!+\!S^{2} A^{\frac{3}{2}} H^{\frac{33}{4}} K^{\frac{1}{4}}\iota \!+\! KH\Delta_r^{(1)} \!+\! KH^2\Delta_p^{(1)} \right). 
	\]
\end{lemma}
Combined with~\eqref{eqn:regret_freedman} and the definition of $\zeta_h^k$, we obtain the dynamic regret bound in a single epoch:
\[
	\mc{R}^{(d)}(\pi, K) = 	O\left(\sqrt{S A H^{3} K \iota}\!+\! \sqrt{KH^3 \iota} \log (KH)\!+\!S^{2} A^{\frac{3}{2}} H^{\frac{33}{4}} K^{\frac{1}{4}}\iota \!+\! KH\Delta_r^{(1)} \!+\! KH^2\Delta_p^{(1)} + KHb_{\Delta}\right), \forall d\in [D].
\]
From our definition of $b_\Delta$, we can easily see that $KHb_{\Delta} \leq O(KH\Delta_r^{(1)} + KH^2\Delta_p^{(1)})$. Finally, suppose $T$ is greater than a polynomial of $S,A,\Delta$ and $H$, $\sqrt{S A H^{3} K \iota}$ would be the leading term of the dynamic regret in a single epoch. In this case, summing up the dynamic regret over all the $D$ epochs gives us an upper bound of
\begin{equation}\label{eqn:tmp_freedman}
\widetilde{O}\left(D\sqrt{SAH^3 K} + \sum_{d=1}^DKH\Delta_r^{(d)}+\sum_{d=1}^D KH^2\Delta_p^{(d)}\right).
\end{equation}
Recall that $\sum_{d=1}^D\Delta_r^{(d)} \leq \Delta_r$, $\sum_{d=1}^D\Delta_p^{(d)} \leq \Delta_p$, $\Delta = \Delta_r+\Delta_p$, and that $K = \Theta(\frac{T}{DH})$. By setting $D=S^{-\frac{1}{3}}A^{-\frac{1}{3}}\Delta^{\frac{2}{3}}T^{\frac{1}{3}}$, the dynamic regret over the entire $T$ steps is bounded by 
\[
\mc{R}(\pi,M) \leq \widetilde{O}\left( S^{\frac{1}{3}}A^{\frac{1}{3}}\Delta^{\frac{1}{3}} H T^{\frac{2}{3}} \right).
\]
This completes the proof of Theorem~\ref{thm:freedman}.

\section{Proof of Theorem~\ref{thm:borl}}
First, we define $D^\dagger$ to be the optimal candidate value in $\mc{J}$ that leads to the lowest dynamic regret. Recall that since $\mc{J}$ is a discretized set and only covers values in the range of $\left[\floor{\frac{T}{SAH^2 W}}, \floor{\frac{T}{SAH^2 }} \right]$, it might not contain the actual optimal value $D^\star=S^{-\frac{1}{3}}A^{-\frac{1}{3}}\Delta^{\frac{2}{3}}T^{\frac{1}{3}}$ for the number of epochs $D$. Further, let $R_i(D)$ be the cumulative reward collected in phase $i$ due to choosing the value $D$ for the number of total epochs. Then, the dynamic regret of Algorithm~\ref{alg:borl} can be decomposed into two parts:
\begin{align}
\mc{R}(\pi, M) =& \sum_{m=1}^{M}\left(V_{1}^{m,\star}\left(s_{1}^{m}\right)-V_{1}^{m,\pi}\left(s_{1}^{m}\right)\right)\nonumber\\
=& \left[\sum_{m=1}^{M}V_{1}^{m,\star}\left(s_{1}^{m}\right) - \sum_{i=1}^{\ceil{M/W}} R_i(D^\dagger) \right] + \left[ \sum_{i=1}^{\ceil{M/W}} R_i(D^\dagger) - \sum_{i=1}^{\ceil{M/W}} R_i(D_i) \right],\label{eqn:decompose}
\end{align}
where the first term is the dynamic regret of using the optimal candidate value $D^\dagger$ of the number of epochs, and  the second term is caused by the regret of learning the optimal candidate value using the Exp3.P algorithm. 
Applying the regret bound of the Exp3.P algorithm (\citet{auer2002nonstochastic}), for any choice of $D^\dagger$, the second term in \eqref{eqn:decompose} is upper bounded by 
\begin{equation}\label{eqn:a6}
\sum_{i=1}^{\ceil{M/W}} R_i(D^\dagger) - \sum_{i=1}^{\ceil{M/W}} R_i(D_i) \leq \widetilde{O}(WH \sqrt{\ceil{M/W}(J+1)}) = \widetilde{O}(\sqrt{HTW}) = \widetilde{O}(H^{\frac{3}{4}}T^{\frac{3}{4}} ),	
\end{equation}
where in the last step we used  that $W = \sqrt{HT}$.

From the proof of Theorem~\ref{thm:freedman} (e.g., Equation \eqref{eqn:tmp_freedman} with the fact that $K = \Theta(\frac{T}{DH})$), and applying the Azuma-Hoeffding inequality and a union bound, we can upper bound the first term in \eqref{eqn:decompose} by
\begin{equation}\label{eqn:decompose2}
	\sum_{m=1}^{M}V_{1}^{m,\star}\left(s_{1}^{m}\right) - \sum_{i=1}^{\ceil{M/W}} R_i(D^\dagger) \leq \widetilde{O}\left(\sqrt{SATD^\dagger H^2} + \frac{TH \Delta}{D^\dagger} + \sqrt{TH} \right). 	
\end{equation}
To derive a further upper bound of \eqref{eqn:decompose2}, we need to distinguish between two cases: Whether $D^\star$ is covered in the range of $\mc{J}$ or not. Since we have assumed that the horizon is sufficiently long, i.e., $T$ is greater than some polynomial of $S, A,\Delta$ and $H$, it holds that $D^\star=S^{-\frac{1}{3}}A^{-\frac{1}{3}}\Delta^{\frac{2}{3}}T^{\frac{1}{3}} \leq \floor{\frac{T}{SAH^2}}$. Therefore, to determine whether $D^\star$ is covered in the range of $\mc{J}$, we only need to compare $D^\star$ with the lower bound $\floor{\frac{T}{SAH^2 W}}$ in $\mc{J}$. 
\begin{itemize}
	\item If $D^\star$ is covered in the range of $\mc{J}$, i.e., $D^\star \geq \floor{\frac{T}{SAH^2W}}$: Since $\mc{J}$ is discretized in a way that two consecutive values differ from each other by a factor of at most $W^{1/J}$, we know that there exists a value $D^\dagger \in \mc{J}$, such that $D^\star \leq D^\dagger \leq W^{1/J}D^\star$. In this case, we can upper bound the RHS of \eqref{eqn:decompose2} by
	\[
	\widetilde{O}\left(\sqrt{SATD^\dagger H^2} + \frac{TH \Delta}{D^\dagger} + \sqrt{TH} \right) \leq \widetilde{O}\left(\sqrt{SATW^{1/J}D^\star H^2} + \frac{TH \Delta}{D^\star} \right)\leq \widetilde{O}\left( S^{\frac{1}{3}}A^{\frac{1}{3}}\Delta^{\frac{1}{3}} H T^{\frac{2}{3}} \right),
	\]
	where in the last step we used the facts that $D^\star=S^{-\frac{1}{3}}A^{-\frac{1}{3}}\Delta^{\frac{2}{3}}T^{\frac{1}{3}}$ and that $W^{1/J} = W^{1/\ceil{\ln W}}\leq \exp(1)$. 
	\item If $D^\star$ is not covered in the range of $\mc{J}$, i.e., $D^\star < \floor{\frac{T}{SAH^2W}}$: Since $D^\star=S^{-\frac{1}{3}}A^{-\frac{1}{3}}\Delta^{\frac{2}{3}}T^{\frac{1}{3}} < \floor{\frac{T}{SAH^2W}}$, it implies that $\Delta < S^{-1}A^{-1}H^{-\frac{15}{4}}T^{\frac{1}{4}}$. The optimal candidate value in $\mc{J}$ would be the smallest one, and hence $D^\dagger = \floor{\frac{T}{SAH^2W}}$. In this case, we can upper bound the RHS of \eqref{eqn:decompose2} by
	\[
	\widetilde{O}\left(\sqrt{SATD^\dagger H^2} + \frac{TH \Delta}{D^\dagger}+ \sqrt{TH} \right) \leq \widetilde{O}\left(\sqrt{SAT H^2 \floor{\frac{T}{SAH^2W}}} + \frac{TH \Delta}{\floor{\frac{T}{SAH^2W}}} \right) \leq \widetilde{O}\left(H^{\frac{3}{4}}T^{\frac{3}{4}} \right),
	\]
	where in the last step we used that $\Delta < S^{-1}A^{-1}H^{-\frac{15}{4}}T^{\frac{1}{4}}$ and $W = \sqrt{HT}$. 
\end{itemize}
Combining the above two cases with \eqref{eqn:decompose}, \eqref{eqn:a6}, and \eqref{eqn:decompose2}, we can conclude that the dynamic regret of Algorithm~\ref{alg:borl} is upper bounded by
\[
\mc{R}(\pi, M) \leq \widetilde{O}\left( S^{\frac{1}{3}}A^{\frac{1}{3}} \Delta^{\frac{1}{3}} H T^{\frac{2}{3}} + H^{\frac{3}{4}}T^{\frac{3}{4}} \right).
\]
This completes the proof of Theorem~\ref{thm:borl}.

\section{Proof of Theorem~\ref{thm:lb}}
The proof of our lower bound relies on the construction of a ``hard instance'' of non-stationary MDPs. The instance we construct is essentially a switching-MDP: an MDP with piecewise constant dynamics on each \emph{segment} of the horizon, and its dynamics experience an abrupt change at the beginning of each new segment. More specifically, we divide the horizon $T$ into $L$ segments\footnote{The definition of segments is irrelevant to, and should not be confused with, the notion of epochs we previously defined. }, where each segment has $T_0 \defeq \floor{\frac{T}{L}}$ steps and contains $M_0\defeq \floor{\frac{M}{L}}$ episodes, each episode having a length of $H$. Within each such segment, the system dynamics of the MDP do not vary, and we construct the dynamics for each segment in a way such that the instance is a hard instance of stationary MDPs on its own. The MDP within each segment is essentially similar to the hard instances constructed in stationary RL problems~\citep{osband2016lower,jin2018q}. Between two consecutive segments, the dynamics of the MDP change abruptly, and we let the dynamics vary in a way such that no information learned from previous interactions with the MDP can be used in the new segment. In this sense, the agent needs to learn a new hard stationary MDP in each segment. Finally, optimizing the value of $L$ and the variation magnitude between consecutive segments (subject to the constraints of the total variation budget) leads to our lower bound. 

We start with a simplified episodic setting where the transition kernels and reward functions are held constant within each episode, i.e., $P_{1}^m = \dots = P_{h}^m = \dots P_{H}^m$ and $r_{1}^m = \dots = r_{h}^m = \dots r_{H}^m, \forall m\in[M]$. This is a popular but less challenging episodic setting, and its stationary counterpart has been studied in~\citet{azar2017minimax}. We further require that when the environment varies due to the non-stationarity, all steps in one episode should vary simultaneously in the same way. This simplified setting is easier to analyze, and its analysis conveniently leads to a lower bound for the un-discounted setting as a side result along the way. Later we will show how the analysis can be naturally extended to the more general setting we introduced in Section~\ref{sec:preliminary}, using techniques that have also been utilized in~\citet{jin2018q}. For simplicity of notations, we temporarily drop the $h$ indices and use $P^m$ and $r^m$ to denote the transition kernel and reward function whenever there is no ambiguity. 

\begin{figure}[t]
	\centering
	\includegraphics[width=.35\textwidth]{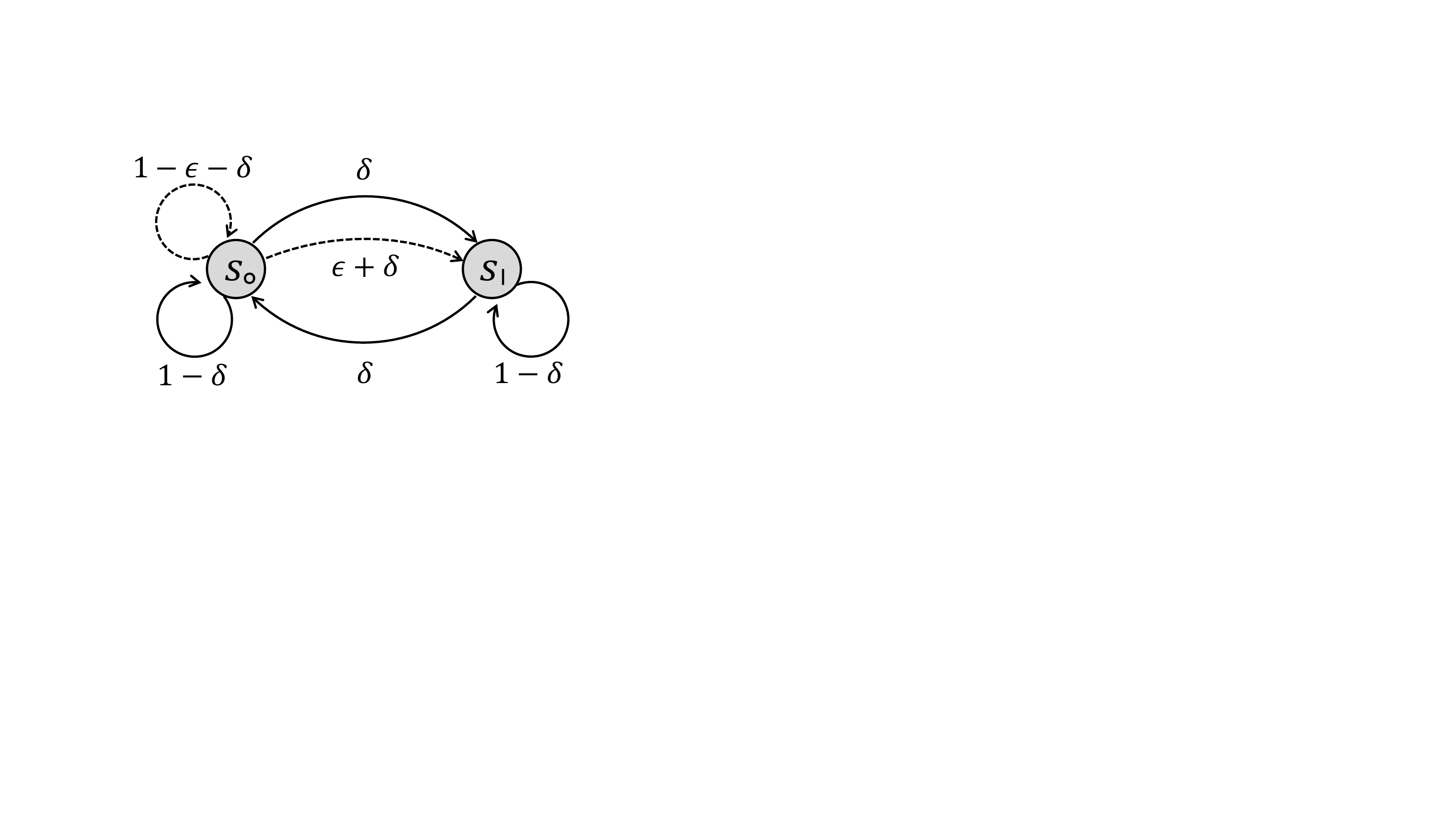}\caption{The ``JAO MDP'' constructed in~\citet{jaksch2010near}. Dashed lines denote transitions related to the good action $a^\star$.}\label{fig:1}
\end{figure}

Consider a two-state MDP as depicted in Figure~\ref{fig:1}. This MDP was initially proposed in~\citet{jaksch2010near} as a hard instance of stationary MDPs, and following~\citet{jin2018q} we will refer to this construction as the ``JAO MDP''. This MDP has $2$ states $\mc{S} = \{ \so, \si\}$ and $SA$ actions $\mc{A} = \{1,2,\dots, SA \}$. The reward does not depend on actions: state $\si$ always gives reward $1$ whatever action is taken, and state $\so$ always gives reward $0$. Any action taken at state $\si$ takes the agent to state $\so$ with probability $\delta$, and to state $\si$ with probability $1-\delta$. At state $\so$, for all but a single ``good'' action $a^\star$, the agent is taken to state $\si$ with probability $\delta$, and for the good action $a^\star$, the agent is taken to state $\si$ with probability $\delta+\epsilon$ for some $0 < \epsilon < \delta$. The exact values of $\delta$ and $\epsilon$ will be chosen later. Note that this is not an MDP with $S$ states and $A$ actions as we desire, but the extension to an MDP with $S$ states and $A$ actions is routine~\citep{jaksch2010near}, and is hence omitted here.

\begin{figure}[t]
	\centering\includegraphics[width=.75\textwidth]{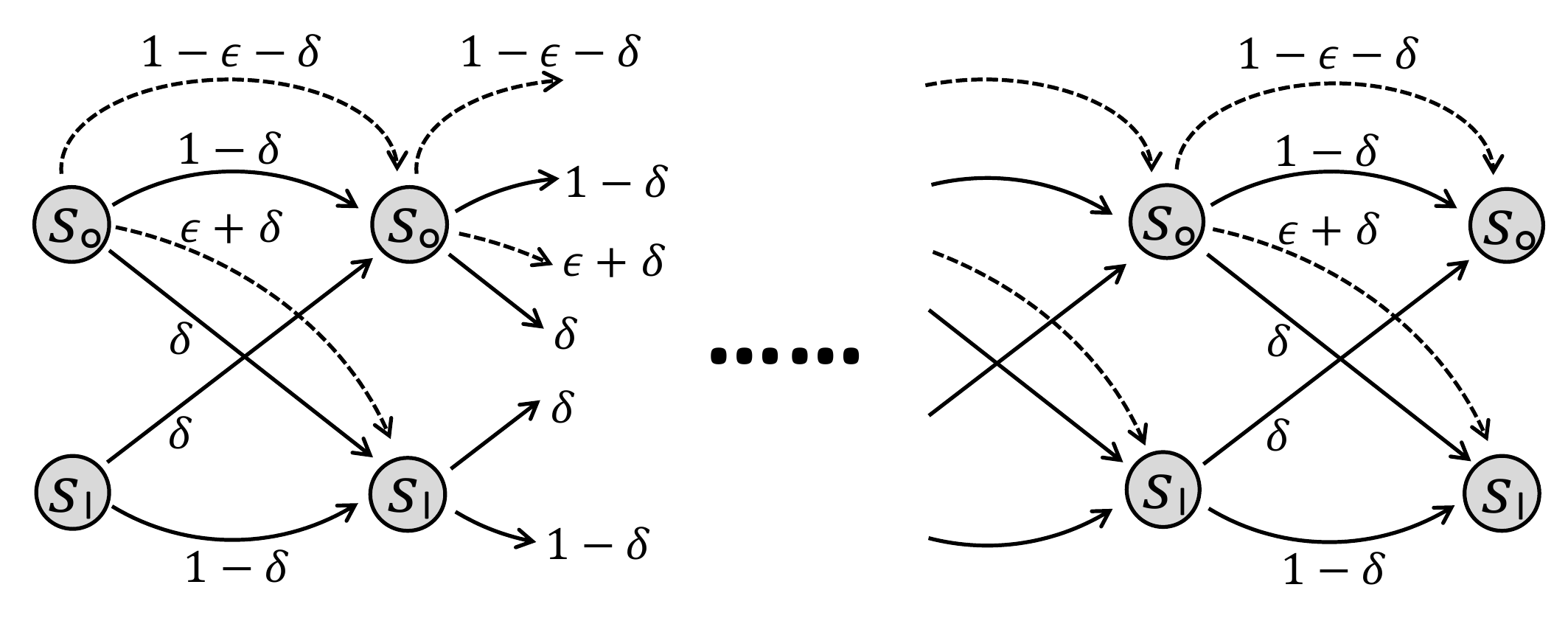}
	\caption{A chain with $H$ copies of JAO MDPs correlated in time. At the end of an episode, the state should deterministically transition from any state in the last copy to the $\so$ state in the first copy of the chain, the arrows of which are not shown in the figure. Also, the $\si$ state in the first copy is actually never reached and is redundant. }\label{fig:2}
\end{figure}

To apply the JAO MDP to the simplified episodic setting, we ``concatenate'' $H$ copies of exactly the same JAO MDP into a chain as depicted in Figure~\ref{fig:2}, denoting the $H$ steps in an episode. The initial state of this MDP is the $\so$ state in the first copy of the chain, and after each episode the state is ``reset'' to the initial state. In the following, we first show that the constructed MDP is a hard instance of stationary MDPs, without worrying about the evolution of the system dynamics. The techniques that we will be using are essentially the same as in the proofs of the lower bound in the multi-armed bandit problem~\citep{auer2002nonstochastic} or the reinforcement learning problem in the un-discounted setting~\citep{jaksch2010near}. 

The good action $a^\star$ is chosen uniformly at random from the action space $\mc{A}$, and we use $\ee_\star[\cdot ]$ to denote the expectation with respect to the random choice of $a^\star$. We write $\ee_a[\cdot]$ for the expectation conditioned on action $a$ being the good action $a^\star$. Finally, we use $\ee_{\text{unif}}[\cdot]$ to denote the expectation when there is no good action in the MDP, i.e., every action in $\mc{A}$ takes the agent from state $\so$ to $\si$ with probability $\delta$. Define the probability notations $\pp_\star(\cdot),\pp_a(\cdot)$, and $\pp_{\text{unif}}(\cdot)$ analogously. 

Consider running a reinforcement learning algorithm on the constructed MDP for $T_0$ steps, where $T_0 = M_0 H$. It has been shown in~\citet{auer2002nonstochastic} and~\citet{jaksch2010near} that it is sufficient to consider deterministic policies. Therefore, we assume that the algorithm maps deterministically from a sequence of observations to an action $a_t$ at time $t$. Define the random variables $\ni,\no$ and $\nos$ to be the total number of visits to state $\si$, the total number of visits to $\so$, and the total number of times that $a^\star$ is taken at state $\so$, respectively. Let $s_t$ denote the state observed at time $t$, and $a_t$ the action taken at time $t$. When there is no chance of ambiguity, we sometimes also use $s_h^m$ to denote the state at step $h$ of episode $m$, which should be interpreted as the state $s_t$ observed at time $t = (m-1)\times H + h$. The notation $a_h^m$ is used analogously. Since $\so$ is assumed to be the initial state, we have that
\[
\begin{aligned}
&\ee_a[\ni] = \sum_{t=1}^{T_0} \pp_a(s_t = \si) = \sum_{m=1}^{M_0}\sum_{h=2}^H\pp_a(s_h^m = \si)\\
=&\sum_{m=1}^{M_0}\sum_{h=2}^H \left(\pp_a(s_{h-1}^m = \so) \cdot \pp_a(s_h^m = \si\mid s_{h-1}^m = \so) + \pp_a(s_{h-1}^m = \si) \cdot \pp_a(s_h^m = \si\mid s_{h-1}^m = \si)  \right)\\
= & \sum_{m=1}^{M_0}\sum_{h=2}^H \left(\delta \pp_a(s_{h-1}^m = \so, a_h^m \neq a^\star) + (\delta +\epsilon) \pp_a(s_{h-1}^m = \so, a_h^m = a^\star) + (1-\delta)\pp_a(s_{h-1}^m=\si) \right)\\
\leq & \delta \ee_a[\no - \nos] + (\delta+\epsilon)\ee_a[\nos] + (1-\delta)\ee_a[\ni],
\end{aligned}
\]
and rearranging the last inequality gives us $\ee_a[\ni] \leq  \ee_a[\no - \nos] + (1+\frac{\epsilon}{\delta})\ee_a[\nos]$. 

For this proof only, define the random variable $W(T_0)$ to be the total reward of the algorithm over the horizon $T_0$, and define $G(T_0)$ to be the (static) regret with respect to the optimal policy. Since for any algorithm, the probability of staying in state $\so$ under $\pp_a(\cdot)$ is no larger than under $\pp_{\text{unif}}(\cdot)$, it follows that
\begin{align}
&\ee_a[W(T_0)] \leq \ee_a[\ni] \leq \ee_a[\no - \nos] + (1+\frac{\epsilon}{\delta})\ee_a[\nos]\nonumber\\
=& \ee_a[\no] + \frac{\epsilon}{\delta}\ee_a[\nos] \leq \ee_{\text{unif}}[\no] + \frac{\epsilon}{\delta}\ee_a[\nos]\nonumber\\
=& T_0 - \ee_{\text{unif}}[\ni] + \frac{\epsilon}{\delta}\ee_a[\nos].\label{eqn:a1}
\end{align}

Let $\tau^m_{\circ \vertrule{}}$ denote the first step that the state transits from state $\so$ to $\si$ in the $m$-th episode; then
\begin{align}
\ee_{\text{unif}}[\ni] =& \sum_{m=1}^{M_0} \sum_{h=1}^H \pp_{\text{unif}}(\tau^m_{\circ \vertrule{}} = h) \ee_{\text{unif}}[\ni \mid \tau^m_{\circ \vertrule{}} = h]\nonumber
= \sum_{m=1}^{M_0} \sum_{h=1}^H (1-\delta)^{h-1}\delta \ee_{\text{unif}}[\ni \mid \tau^m_{\circ \vertrule{}} = h]\nonumber\\
\geq & \sum_{m=1}^{M_0} \sum_{h=1}^H (1-\delta)^{h-1}\delta \frac{H-h}{2}\nonumber
=  \sum_{m=1}^{M_0}\left(\frac{H}{2} - \frac{1}{2\delta} + \frac{(1-\delta)^H}{2\delta}\right)\\
\geq & \frac{T_0}{2} - \frac{M_0}{2\delta}\label{eqn:a2}.
\end{align}

Since the algorithm is a deterministic mapping from the observation sequence to an action, the random variable $\nos$ is also a function of the observations up to time $T$. In addition, since the immediate reward only depends on the current state, $\nos$ can further be considered as a function of just the state sequence up to $T$. Therefore, the following lemma from~\citet{jaksch2010near}, which in turn was adapted from Lemma A.1 in~\citet{auer2002nonstochastic}, also applies in our setting. 

\begin{lemma}(Lemma 13 in~\citet{jaksch2010near})\label{lemma:KL}
	For any finite constant $B$, let $f:\{\so,\si\}^{T_0 +1} \ra [0,B]$ be any function defined on the state sequence $\mbf{s}\in \{\so,\si\}^{T_0+1}$. Then, for any $0 < \delta \leq \frac{1}{2}$, any $0 < \epsilon \leq 1-2\delta$, and any $a\in \mc{A}$, it holds that
	\[
	\mathbb{E}_{a}[f(\mbf{s})] \leq \mathbb{E}_{\text{unif}}[f(\mbf{s})]+\frac{B}{2} \cdot \frac{\epsilon}{\sqrt{\delta}} \sqrt{2 \mathbb{E}_{\text {unif}}\left[\nos\right]}.
	\]
\end{lemma}

Since $\nos$ itself is a function from the state sequence to $[0,T_0]$, we can apply Lemma~\ref{lemma:KL} and arrive at
\begin{equation}\label{eqn:a3}
\mathbb{E}_{a}[\nos] \leq \mathbb{E}_{\text{unif }}[\nos]+\frac{T_0}{2} \cdot \frac{\epsilon}{\sqrt{\delta}} \sqrt{2 \mathbb{E}_{\text {unif}}\left[\nos\right]}.
\end{equation}
From~\eqref{eqn:a2}, we have that $\sum_{a=1}^{SA} \mathbb{E}_{\text{unif }}[\nos] = T_0 - \mathbb{E}_{\text{unif }}[\ni] \leq \frac{T_0}{2} + \frac{M_0}{2\delta}$. By the Cauchy-Schwarz inequality, we further have that $\sum_{a=1}^{SA}\sqrt{2 \mathbb{E}_{\text {unif}}\left[\nos\right]} \leq \sqrt{SA (T_0 + \frac{M_0}{\delta})}$. Therefore, from~\eqref{eqn:a3}, we obtain
\[
\sum_{a=1}^{SA}\mathbb{E}_{a}[\nos]  \leq \frac{T_0}{2} + \frac{M_0}{2\delta} + \frac{T_0}{2}\cdot \frac{\epsilon}{\sqrt{\delta}} \sqrt{SA (T_0 + \frac{M_0}{\delta})}.
\]
Together with~\eqref{eqn:a1} and~\eqref{eqn:a2}, it holds that
\begin{align}
\ee_\star[W(T_0)] \leq &\frac{1}{SA} \sum_{a=1}^{SA}\ee_a[W(T_0)]\nonumber \\
\leq &\frac{T_0}{2} +\frac{M_0}{2\delta} +  \frac{\epsilon}{\delta}\frac{1}{SA}\left(\frac{T_0}{2} + \frac{M_0}{2\delta} + \frac{T_0}{2}\cdot \frac{\epsilon}{\sqrt{\delta}} \sqrt{SA (T_0 + \frac{M_0}{\delta})}\right). \label{eqn:a4}
\end{align}

\subsection{The Un-discounted Setting}
Let us now momentarily deviate from the episodic setting and consider the un-discounted setting (with $M_0 = 1$). This is the case of the JAO MDP in Figure~\ref{fig:1} where there is not reset. We could calculate the stationary distribution and find that the optimal average reward for the JAO MDP is $\frac{\delta + \epsilon}{2\delta + \epsilon}$. It is also easy to calculate that the diameter of the JAO MDP is $D = \frac{1}{\delta}$. Therefore, the expected (static) regret with respect to the randomness of $a^*$ can be lower bounded by
\[
\begin{aligned}
\ee_\star[G(T_0)] =& \frac{\delta + \epsilon}{2\delta + \epsilon}T_0 - \ee_\star[W(T_0)]\\
\geq & \frac{\epsilon T_0}{4\delta + 2\epsilon} - \frac{D}{2} - \frac{\epsilon D(T_0+D)}{2SA} - \frac{\epsilon^2 T_0D\sqrt{D}}{2\sqrt{SA}}(\sqrt{T_0} + \sqrt{D}).
\end{aligned}
\]
By assuming $T_0\geq DSA$ (which in turn suggests $D\leq \sqrt{\frac{T_0D}{SA}}$) and setting $\epsilon = c\sqrt{\frac{SA}{T_0D}}$ for $c = \frac{3}{40}$, we further have that
\[
\begin{aligned}
\ee_\star[G(T_0)] \geq & \left(\frac{c}{6} - \frac{c}{2SA} - \frac{cD}{2SAT_0} - \frac{c^2}{2} - \frac{c^2}{2}\sqrt{\frac{D}{T_0}} \right)\sqrt{SAT_0D} - \frac{D}{2}\\
\geq & \left(\frac{3}{20}c - c^2 - \frac{1}{200} \right)\sqrt{SAT_0D} = \frac{1}{1600}\sqrt{SAT_0D}.  
\end{aligned}
\]
It is easy to verify that our choice of $\delta$ and $\epsilon$ satisfies our assumption that $ 0 < \epsilon < \delta$. So far, we have recovered the (static) regret lower bound of $\Omega(\sqrt{SAT_0D})$ in the un-discounted setting, which was originally proved in~\citet{jaksch2010near}. 

Based on this result, let us now incorporate the non-stationarity of the MDP and derive a lower bound for the dynamic regret $\mc{R}(T)$. Recall that we are constructing the non-stationary environment as a switching-MDP. For each segment of length $T_0$, the environment is held constant, and the regret lower bound for each segment is $\Omega(\sqrt{SAT_0D})$. At the beginning of each new segment, we uniformly sample a new action $a^*$ at random from the action space $\mc{A}$ to be the good action for the new segment. In this case, the learning algorithm cannot use the information it learned during its previous interactions with the environment, even if it knows the switching structure of the environment. Therefore, the algorithm needs to learn a new (static) MDP in each segment, which leads to a dynamic regret lower bound of $\Omega(L\sqrt{SAT_0D}) = \Omega(\sqrt{SATLD})$, where let us recall that $L$ is the number of segments. Every time the good action $a^*$ varies, it will cause a variation of magnitude $2\epsilon$ in the transition kernel. The constraint of the overall variation budget requires that $2\epsilon L  = \frac{3}{20}\sqrt{\frac{SA}{T_0D}} L\leq \Delta$, which in turn requires $L \leq 4\Delta^{\frac{2}{3}} T^{\frac{1}{3}} D^{\frac{1}{3}} S^{-\frac{1}{3}} A^{-\frac{1}{3}}$. Finally, by assigning the largest possible value to $L$ subject to the variation budget, we obtain a dynamic regret lower bound of $\Omega\left( S^{\frac{1}{3}}A^{\frac{1}{3}} \Delta^{\frac{1}{3}} D^{\frac{2}{3}} T^{\frac{2}{3}} \right)$. This completes the proof of Proposition~\ref{corollary}.

\subsection{The Episodic Settings}
Now let us go back to our simplified episodic setting, as depicted in Figure~\ref{fig:2}. One major difference with the previous un-discounted setting is that we might not have time to mix between $\so$ and $\si$ in $H$ steps. (Note that we only need to reach the stationary distribution over the $(\so,\si)$ pair in each step $h$, rather than the stationary distribution over the entire MDP. In fact, the latter case is never possible because the entire MDP is not aperiodic.) It can be shown that the optimal policy on this MDP has a mixing time of $\Theta\left(\frac{1}{\delta}\right)$~\citep{jin2018q}, and hence we can choose $\delta$ to be slightly larger than $\Theta(\frac{1}{H})$ to guarantee sufficient time to mix. All the analysis up to inequality~\eqref{eqn:a4} carries over to the episodic setting, and essentially we can set $\delta$ to be $\Theta\left(\frac{1}{H}\right)$ to get a (static) regret lower bound of $\Omega(\sqrt{SAT_0H})$ in each segment. Another difference with the previous setting lies in the usage of the variation budget. Since we require that all the steps in the same episode should vary simultaneously, it now takes a variation budget of $2\epsilon H$ each time we switch to a new action $a^*$ at the beginning of a new segment. Therefore, the overall variation budget now puts a constraint of $2\epsilon H L \leq O(\Delta)$ on the magnitude of each switch. Again, by choosing $\epsilon = \Theta\left(\sqrt{\frac{SA}{T_0 H}}\right)$ and optimizing over possible values of $L$ subject to the budget constraint, we obtain a dynamic regret lower bound of $\Omega\left( S^{\frac{1}{3}}A^{\frac{1}{3}} \Delta^{\frac{1}{3}} H^{\frac{1}{3}} T^{\frac{2}{3}} \right)$ in the simplified episodic setting. 

Finally, we consider the standard episodic setting as introduced in Section~\ref{sec:preliminary}. In this setting, we essentially will be concatenating $H$ distinct JAO MDPs, each with an independent good action $a^*$, into a chain like Figure~\ref{fig:2}. The transition kernels in these JAO MDPs are also allowed to vary asynchronously in each step $h$, although our construction of the lower bound does not make use of this property. As argued similarly in~\citet{jin2018q}, the number of observations for each specific JAO MDP is only $T_0/H$, instead of $T_0$. Therefore, we can assign a slightly larger value to $\epsilon$ and the learning algorithm would still not be able to identify the good action given the fewer observations.  Setting $\delta = \Theta\left(\frac{1}{H}\right)$ and $\epsilon = \Theta\left(\sqrt{\frac{SA}{T_0}}\right)$ leads to a (static) regret lower bound of $\Omega(H\sqrt{SAT_0})$ in the stationary RL problem. Again, the transition kernels in all the $H$ JAO MDPs vary simultaneously at the beginning of each new segment. By optimizing $L$ subject to the overall budget constraint $2\epsilon HL \leq O(\Delta)$, we obtain a dynamic regret lower bound of $\Omega\left( S^{\frac{1}{3}}A^{\frac{1}{3}} \Delta^{\frac{1}{3}} H^{\frac{2}{3}} T^{\frac{2}{3}} \right)$ in the episodic setting. This completes our proof of Theorem~\ref{thm:lb}.

\section{Proof of Theorem~\ref{thm:team}}\label{app:proof_application}
\begin{proof}
	\wmedit{We first show that when the switching cost of agent 2 satisfies $N_{\text{switch}} = O(T^{\beta})$ for $0 < \beta < 1$, the dynamic regret of agent 1 is upper bounded by $\widetilde{O}(T^{\frac{\beta+2}{3}})$. To see this, notice that from the perspective of agent 1, the environment is non-stationary due to the fact that agent 2 is changing its policy over time. Since the switching cost of agent 2 is upper bounded by $O(T^\beta)$, by the definitions of $\Delta_r$ and $\Delta_p$ in Section~\ref{sec:preliminary}, we know that the variation of the environment from the perspective of agent 1 is upper bounded by $O(T^\beta)$. Substituting the value of $\Delta$ with $O(T^\beta)$ in Theorem~\ref{thm:regret_no_budget} or Theorem~\ref{thm:freedman} leads to the desired result.  }
	
	From the $(\lambda,\mu)$-smoothness of the MDP, it follows that
	\[
	\lambda \cdot V_h^{(\pi^{(1)\star},\pi^{(2)\star})}(s) - \mu\cdot V_h^{(\pi^{(1)},\pi^{(2)})}(s) \leq V_h^{(\pi^{(1)\star},\pi^{(2)})}(s), \forall s \in\mc{S},h\in[H].
	\]
	Therefore, it holds that
	\[
	\begin{aligned}
	&\sum_{m=1}^M \left( \lambda \cdot  V_1^{(\pi^{(1)\star},\pi^{(2)\star})}(s_1^m) - (1+\mu) \cdot V_1^{(\pi^{(1)},\pi^{(2)})}(s_1^m) \right)\\
	\leq& \sum_{m=1}^M \left(V_1^{(\pi^{(1)\star},\pi^{(2)})}(s_1^m)  - V_1^{(\pi^{(1)},\pi^{(2)})}(s_1^m) \right)\\
	\leq & \sum_{m=1}^{M}\left( \sup_{\pi^{(1)\star}} V_1^{(\pi^{(1)\star},\pi^{(2)})} \left(s_{1}^{m}\right)-V_{1}^{(\pi^{(1)},\pi^{(2)})}\left(s_{1}^{m}\right)\right)\\
	=& \mc{R}^{\pi^{(2)}}(\pi^{(1)}, M) = \widetilde{O}(T^{\frac{\beta+2}{3}}),
	\end{aligned}
	\] 
	\wmedit{where the last step follows from the $\widetilde{O}(T^{\frac{\beta+2}{3}})$ dynamic regret bound of agent 1 as we discussed above. }	Rearranging the terms leads to the desired result. 
\end{proof}

\section{Proofs for Section~\ref{sec:application_inventory}}\label{app:proof_inventory}
\subsection{Proof of Lemma~\ref{lemma:inventory}}
First, we show that 
\begin{equation}\label{eqn:a7}
	\sum_{m=1}^{M}\left(V_{1}^{m,\star}\left(s_{1}^{m}\right)-V_{1}^{m,\star,\text{pseudo}}\left(s_{1}^{m}\right)\right) = -\sum_{m=1}^{M}\sum_{h=1}^H p\cdot \ee[X_h^m].
\end{equation}
Recall that the pseudo-reward is constructed by uniformly shifting the reward function by an amount of $p\cdot \ee[X_h^m]$. Since the difference between the reward and the pseudo-reward does not depend on the action taken, for any realization of the demands $\{X_h^m\}_{m\in[M],h\in[H]}$, the optimal policies $\pi^\star$ and $\pi^{\star,\text{pseudo}}$ induce the same distribution over the action space, which in turn leads to the same distribution over state-action trajectories. We can hence conclude that \eqref{eqn:a7} holds. 

Similarly, one can show by induction that for any realization of the demands $\{X_h^m\}_{m\in[M],h\in[H]}$, Algorithm~\ref{alg:RQUCB} also induces the same distribution of action sequences on $\mc{M}$ and $\mc{M}^\text{pseudo}$. This leads us to
\begin{equation}\label{eqn:a8}
	\sum_{m=1}^{M}\left(V_{1}^{m,\pi}\left(s_{1}^{m}\right)-V_{1}^{m,\pi,\text{pseudo}}\left(s_{1}^{m}\right)\right) = -\sum_{m=1}^{M}\sum_{h=1}^H p\cdot \ee[X_h^m].
\end{equation}
Combining \eqref{eqn:a7} and \eqref{eqn:a8} yields the desired result.

\section{Simulations Setup}\label{app:simulations}
We compare RestartQ-UCB with three baseline algorithms: LSVI-UCB-Restart~\citep{zhou2020nonstationary}, Q-Learning UCB, and Epsilon-Greedy~\citep{watkins1989learning}. LSVI-UCB-Restart is a state-of-the-art non-stationary RL algorithm that combines optimistic least-squares value iteration with periodic restarts. It is originally designed for non-stationary RL in linear MDPs, but in our simulations we reduce it to the tabular case by setting the feature map to be essentially an identity mapping, i.e., the feature dimension is set to be $d = S \times A$. Q-Learning UCB is simply our RestartQ-UCB algorithm with no restart. It is a Q-learning based algorithm that uses upper confidence bounds to guide the exploration. Epsilon-Greedy is also a Q-learning based algorithm with restarts. Compared with RestartQ-UCB, Epsilon-Greedy does not employ a UCB-based bonus term to explicitly force exploration. Instead, it takes the greedy action according to the estimated $Q$ function with a high probability $1-\epsilon$, and explores an action from the action set uniformly at random with probability $\epsilon$. 

We evaluate the cumulative rewards of the four algorithms on a variant of a reinforcement learning task named Bidirectional Diabolical Combination Lock~\citep{agarwal2020pc,misra2020kinematic}. This task is designed to be particularly difficult for \emph{exploration}. At the beginning of each episode, the agent starts at a fixed state. According to its first action, the agent transitions to one of the two paths, or ``combination locks'', each of length $H$. Each path is a chain of $H$ states, where the state at the endpoint of each path gives a high reward. At each step on the path, there is only one ``correct'' action that leads the agent to the next state on the path, while the other $A-1$ actions lead it to a sinking state that yields a small per-step reward of $\frac{1}{8H}$ ever since. Since we are considering a non-deterministic MDP, each intended transition ``succeeds'' with probability $0.98$; that is, even if the agent takes the correct action at a certain step, there is still a $0.02$ probability that it will end in the sinking state. The agent obtains a $0$ reward when taking a correct action, and gets a $\frac{1}{8H}$ reward at the step when it transitions to the sinking state. Finally, the endpoint state of one path gives a reward of $1$, while the other endpoint only gives a reward of $0.25$. As argued in~\citet{agarwal2020pc}, the following properties make this task especially challenging: First, it has sparse high rewards, and uniform exploration only has a $A^{-H}$ probability of reaching a high reward endpoint. Second, it has dense low rewards, and a locally optimal policy will lead to the sinking state quickly. Third, there is no indication which path has the globally optimal reward, and the agent must remember to still visit the other one. Interested readers can refer to Section 5.1 of~\citet{agarwal2020pc} for detailed descriptions of the task. 

We introduce two types of non-stationarity to the Bidirectional Diabolical Combination Lock task, namely \emph{abrupt} variations and \emph{gradual} variations. For abrupt variations, we periodically switch the two high-reward endpoints: One high-reward endpoint gives a reward of $1$ at the beginning, and abruptly changes to a reward of $0.25$ after a certain number of episodes, and then switches back to the reward of $1$ after the same number of episodes. The other high-reward endpoint goes the other way around. For gradual changes, we gradually vary the transition probability at the starting state: At the first episode, one action leads to the first path with $0.98$ probability, and to the second path with $0.02$ probability. We linearly decrease its probability of leading to the first path, and increase its probability to the second path. As a result, at the last episode, this action would lead to the first path with $0.02$ probability, and to the second path with $0.98$ probability instead. The same is true for the other actions. 

For simplicity, we use Hoeffding-based bonus terms in the simulations for RestartQ-UCB. We set $M = 5000, H=5, S =  10$, and $A = 2$. For abrupt variations, we switch the two high-reward endpoints after every $1000$ episodes. The hyper-parameters for each algorithm are optimized individually. For RestartQ-UCB, LSVI-UCB-Restart, and Epsilon-Greedy, we restart the algorithms after every $1000$ episodes both for abrupt variations and gradual variations. This is the same frequency as the abrupt variation of the environment (because the restart frequency is optimized as a hyper-parameter), although it turns out that other restart frequencies lead to very similar results. For Epsilon-Greedy, we set the exploration probability to be $\epsilon = 0.05$. All results are averaged over $30$ runs on a laptop with an Intel Core i5-9300H CPU and 16 GB memory.

\end{APPENDIX}
\end{document}